\documentclass[letterpaper]{article} 
\usepackage{preprint}  
\usepackage{times}  
\usepackage{helvet}  
\usepackage{courier}  
\usepackage[hyphens]{url} 
\usepackage{graphicx}
\usepackage{natbib}  
\usepackage{caption} 
\usepackage{algorithm}
\usepackage{algorithmic}

\usepackage{newfloat}
\usepackage{listings}
\DeclareCaptionStyle{ruled}{labelfont=normalfont,labelsep=colon,strut=off} 
\floatstyle{ruled}
\newfloat{listing}{tb}{lst}{}
\floatname{listing}{Listing}

\makeatletter
\def\ps@myheadings{%
  \let\@oddfoot\@empty
  \let\@evenfoot\@empty
  \def\@oddhead{\hfil{\fontsize{9}{10.5}\selectfont\normalfont\leftmark}\hfil}%
  \def\@evenhead{\hfil{\fontsize{9}{10.5}\selectfont\normalfont\leftmark}\hfil}%
  \let\@mkboth\@gobbletwo
  \let\sectionmark\@gobble
  \let\subsectionmark\@gobble
}
\makeatother

\pubnote{The Fortieth AAAI Conference on Artificial Intelligence (AAAI-26)}
\title{Adaptive Riemannian Graph Neural Networks}
\author{
    Xudong Wang\textsuperscript{\rm 1},
    Chris Ding\textsuperscript{\rm 1},
    Tongxin Li\textsuperscript{\rm 1},
    Jicong Fan\textsuperscript{\rm 1}\footnote{Corresponding author.\\This is the full version of the paper published in  
\emph{The 40th Annual AAAI Conference on Artificial Intelligence~(AAAI-2026)}, 
Main Technical Track.}
}
\affiliations{
    \textsuperscript{\rm 1}School of Data Science, The Chinese University of Hong Kong, Shenzhen (CUHK-Shenzhen), China\\
    xudongwang@link.cuhk.edu.cn, \{chrisding,litongxin,fanjicong\}@cuhk.edu.cn}

\usepackage{booktabs}
\usepackage{amsmath,amssymb,amsfonts,amsthm}
\usepackage{xspace}
\usepackage{subcaption}
\usepackage{xcolor}
\usepackage{longtable}
\usepackage{enumerate}
\usepackage{mathtools}
\usepackage{hyperref}
\usepackage{caption} 
\usepackage{multirow}
\newcommand{\modelname}{\textsf{ARGNN}\xspace}

\newtheorem{theorem}{Theorem}
\newtheorem{lemma}{Lemma}                
\newtheorem{proposition}{Proposition}    
\newtheorem{corollary}{Corollary}        
\newtheorem{definition}{Definition}
\newtheorem{assumption}{Assumption}
\newtheorem{remark}{Remark}

\newcommand{\ci}[1]{\scriptsize{$\pm$#1}}

\begin{document}

\maketitle
\thispagestyle{myheadings}
\begin{abstract}
Graph data often exhibits complex geometric heterogeneity, where structures with varying local curvature, such as tree-like hierarchies and dense communities, coexist within a single network. Existing geometric GNNs, which embed graphs into single fixed-curvature manifolds or discrete product spaces, struggle to capture this diversity. We introduce Adaptive Riemannian Graph Neural Networks (\modelname), a novel framework that learns a \emph{continuous and anisotropic Riemannian metric tensor field over the graph}. It allows each node to determine its optimal local geometry, enabling the model to fluidly adapt to the graph's structural landscape. Our core innovation is an efficient parameterization of the node-wise metric tensor, specializing to a learnable diagonal form that captures directional geometric information while maintaining computational tractability. To ensure geometric regularity and stable training, we integrate a Ricci flow-inspired regularization that smooths the learned manifold. Theoretically, we establish the rigorous geometric evolution convergence guarantee for \modelname and provide a continuous generalization that unifies prior fixed or mixed-curvature GNNs. Empirically, our method demonstrates superior performance on both homophilic and heterophilic benchmark datasets with the ability to capture diverse structures adaptively. Moreover, the learned geometries both offer interpretable insights into the underlying graph structure and empirically corroborate our theoretical analysis.
\end{abstract}

%=================================================================
% 1. INTRODUCTION
%=================================================================
\section{Introduction}
\label{sec:introduction}

Real-world graphs, from social networks to protein interaction maps, exhibit a rich geometric diversity that challenges various graph learning paradigms \citep{kipf2016semi,hamilton2017inductive,NEURIPS2023_b61da4f0,kang2023node,wang2024graph,sun2024learning,sun2024mmd,wang2025explainable,qian2025dhakr,wang2025learnable,guo2025graphmore,fan2025graph,fan2025interdisciplinary}. Consider a social network where some communities form deep, tree-like hierarchies best captured by hyperbolic geometry~\citep{chami2019hyperbolic}, while others create dense, tightly-knit cliques that resemble spherical manifolds~\citep{gu2018learning}. Forcing such a geometrically heterogeneous graph into a single geometric space (Euclidean, hyperbolic, or spherical) inevitably introduces significant distortion and information loss.

Pioneering work in geometric Graph Neural Networks (GNNs) has demonstrated the power of non-Euclidean spaces. Hyperbolic GNNs~\citep{chami2019hyperbolic,zhang2021hyperbolic} excel on tree-like graphs but struggle with dense cycles. Conversely, spherical methods handle cyclic structures well but are less effective on hierarchical data. This fundamental limitation stems from their shared assumption of \textit{global geometric homogeneity}. 

Recent advances have moved towards mixed-curvature approaches to address this issue. Models like $\kappa$-GCN~\citep{bachmann2020constant} learn an isotropic (scalar) curvature for each node, while state-of-the-art methods like CUSP~\citep{grover2025cusp} embed graphs into product manifolds of constant-curvature spaces (e.g., $\mathbb{H}^k \times \mathbb{S}^l \times \mathbb{E}^m$). While a significant step forward, these methods remain constrained. Scalar curvature approaches are still isotropic at the node level, unable to capture directional geometric information. Product manifold methods are limited to block-constant geometries chosen from a discrete set of curvatures. Neither can fully capture the fine-grained, continuous geometric variations inherent in complex data. 

\begin{figure}[t]
	\centering
	\begin{subfigure}[c]{0.46\columnwidth}
		\centering
		\fbox{\includegraphics[width=0.95\columnwidth]{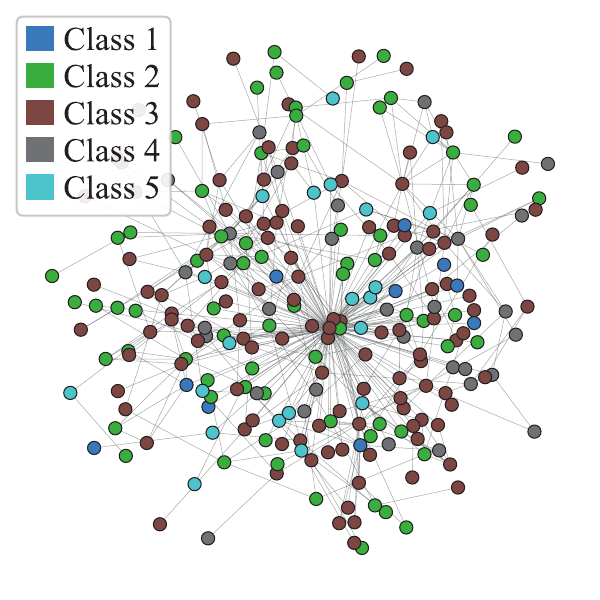}}
	\end{subfigure}
	\begin{subfigure}[c]{0.48\columnwidth}
		\centering
		\fbox{\includegraphics[trim= 60 60 60 60, clip, width=0.95\linewidth]{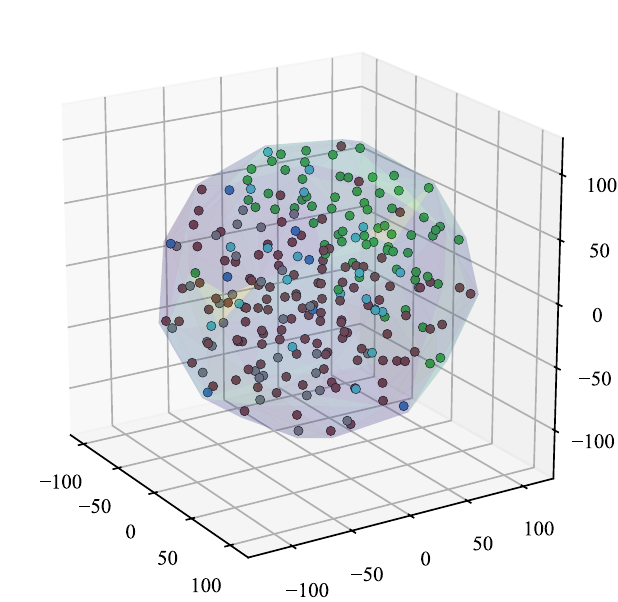}}
	\end{subfigure}
	\caption{
		\textbf{Geometric heterogeneity in \textsc{Wisconsin} network.} \textbf{Left:} raw graph topology coloured by class.
		\textbf{Right:} 3-D t-SNE of node features.
		The translucent hull is coloured by the magnitude of discrete mean
		curvature 
		(\textcolor{violet}{violet} $\!\to\!$ flat, \textcolor{yellow}{yellow} $\!\to\!$ strongly curved),
		showing that curvature varies across the \emph{Riemannian manifold}.}
	\label{fig:intro_wisc_vis}
\end{figure}

This geometric heterogeneity is not merely a theoretical construct; it is evident in real-world data, as illustrated in Figure~\ref{fig:intro_wisc_vis}. The real-world \textsc{Wisconsin} network~\cite{craven2000learning} embeds into a Riemannian feature manifold whose neighbourhoods can differ sharply in both curvature and orientation. Any single or fixed-curvature space would therefore distort part of the data, motivating our node-wise, adaptive metric tensor.

The importance of geometric understanding has been further highlighted by recent advances in curvature-based graph analysis. Graph Neural Ricci Flow (GNRF)~\citep{chen2025graph} demonstrates that evolving features according to discrete Ricci curvature improves representation quality, though it keeps the underlying geometry fixed. Discrete curvature methods like CurvDrop~\citep{liu2023curvdrop} and curvature-based graph rewiring~\citep{topping2021understanding} show promise but operate as preprocessing steps rather than end-to-end learning frameworks.

Our key insight is to move beyond pre-defined geometries by learning a \textbf{continuous, node-adaptive Riemannian metric field} that directly models the local structure. To realize this vision within a tractable framework, we introduce \textbf{Adaptive Riemannian Graph Neural Networks} (\modelname). Our framework proposes a principled and efficient parameterization of the metric tensor. Specifically, we model each node's local geometry with an \textit{anisotropic diagonal metric}, $\mathbf{G}_i = \text{diag}(\mathbf{g}_i)$. This design is not merely a computational shortcut; it corresponds to a flexible, per-dimension scaling of the feature space geometry. To ensure stable learning, we further integrate Ricci flow-inspired dynamics that regularize the evolution of the learned geometric manifold. Implementation code is available at our public repository\footnote {\url{https://github.com/MathAdventurer/ARGNN}}.
The contributions of our work are threefold:
{%
	\setlength{\leftmargini}{0pt}
	\setlength{\labelsep}{2pt}
	\setlength{\itemsep}{2pt}
	\setlength{\topsep}{2pt}
	\begin{enumerate}[i)]
		\item \textbf{Methodologically}, we propose a principled parameterization of the adaptive metric field using diagonal matrices. We provide strong geometric and algorithmic justifications for our design, framing it as an anisotropic conformal transformation that is both computationally efficient and highly expressive. To the best of our knowledge, \modelname is the first framework that learns continuous and anisotropic metric tensor fields for graphs. 
		\item \textbf{Theoretically}, we establish the convergence guarantee for \modelname's geometric evolution and further prove its role as a universal framework that generalizes and unifies prior isotropic, scalar-curvature, and product-manifold GNNs.
		\item \textbf{Empirically}, we conduct extensive experiments on a wide range of benchmark datasets, including comprehensive ablation studies. Our results demonstrate that \modelname achieves superior performance compared to state-of-the-art benchmark methods on both homophilic and heterophilic graphs.
\end{enumerate}}

%=================================================================
% 2. RELATED WORK
%=================================================================
\section{Related Work}\label{sec:related}

\subsection{Fixed-Curvature Geometric GNNs}

Geometric deep learning on graphs has been revolutionized by embedding approaches that leverage non-Euclidean geometries. \textbf{Hyperbolic GNNs} exploit the exponential volume growth of hyperbolic space to naturally represent hierarchical structures. HGCN~\citep{chami2019hyperbolic} extends graph convolutions to hyperbolic space using the Poincaré ball model, achieving remarkable success on tree-like graphs. HGAT~\citep{zhang2021hyperbolic} incorporates attention mechanisms in hyperbolic space, while recent work explores hyperbolic transformers~\citep{gulcehre2018hyperbolic} and variational autoencoders~\citep{sun2021hyperbolic}.

\textbf{Spherical GNNs} address the complementary challenge of cyclic and community structures. Spherical CNNs~\citep{cohen2018spherical} and their graph extensions~\citep{gu2018learning} leverage the positive curvature of spherical space to model dense, interconnected communities. However, both hyperbolic and spherical approaches assume global geometric homogeneity, limiting their applicability to mixed-topology graphs.

Fixed-curvature methods have the advantage compared with traditional GNNs in different target domains but fail catastrophically when graph geometry mismatches the embedding space, as shown in Figure~\ref{fig:intro_wisc_vis}, which prevents them from effectively handling the geometric heterogeneity prevalent in real-world networks.

\subsection{Discrete Mixed-Curvature Approaches}

Recognizing the limitations of fixed-curvature methods, recent work explores mixed-curvature embeddings. 

$\kappa$-GCN~\citep{bachmann2020constant} learns scalar curvature parameters for each node, allowing adaptation between hyperbolic, Euclidean, and spherical geometries. However, scalar curvature cannot capture directional geometric information.

CUSP~\citep{grover2025cusp} represents the current state-of-the-art in mixed-curvature approaches. It combines spectral graph analysis with embeddings in product manifolds of constant-curvature spaces (e.g., $\mathbb{H}^k\!\times\!\mathbb{S}^l\!\times \mathbb{E}^m$). CUSP also introduces spectral filtering techniques and curvature-aware graph Laplacians, achieving strong performance across diverse benchmarks.

$\mathcal{Q}$-GCN~\citep{xiong2022pseudo} extends the framework to pseudo-Riemannian manifolds with indefinite metrics, enabling more flexible curvature combinations. Self-supervised mixed-curvature methods~\citep{jin2020self} explore representation learning without labels.

However, despite their flexibility, these approaches primarily treat geometry as \emph{discrete choices} from finite sets of constant-curvature manifolds. Consequently, they offer limited support for continuous geometric variation and anisotropic (directional) structure.

\subsection{Curvature-Based Graph Analysis}

Recent advances in discrete differential geometry have enabled sophisticated graph analysis through curvature. 

\textbf{Discrete Ricci curvature}, including Ollivier-Ricci~\citep{ollivier2009ricci} and Forman-Ricci~\citep{forman2003bochner} curvature, provide local geometric characterizations of graph structure.

Graph Neural Ricci Flow (GNRF)~\citep{chen2025graph} represents a significant recent advance. GNRF evolves node features according to discrete Ricci curvature, showing that curvature-aware feature evolution improves representation quality. However, GNRF keeps the underlying graph geometry fixed while evolving features.

\textbf{Curvature-based rewiring} methods~\citep{topping2021understanding,fesser2024mitigating} use curvature analysis to identify and modify graph bottlenecks. CurvDrop~\citep{liu2023curvdrop} employs Ricci curvature for topology-aware dropout sampling. These methods show curvature's utility but operate as preprocessing steps rather than end-to-end learning.

Curvature-based methods have found success in protein analysis~\citep{wu2023curvagn, shen2024curvature}, community detection~\citep{ni2019community}, and graph generation~\citep{li2022curvature}. 
However, to our knowledge, \emph{joint learning} of geometry and features remains rare; methods typically fix geometry and evolve features (GNRF) or use curvature for preprocessing.

\section{Preliminaries}
\label{sec:preliminaries}

\subsection{Riemannian Geometry Essentials}

A \textbf{Riemannian manifold} $(\mathcal{M}, \mathbf{g})$ consists of a smooth manifold $\mathcal{M}$ equipped with a metric tensor field $\mathbf{g}$ that varies smoothly across the manifold~\citep{lee2018introduction}. At each point $p \in \mathcal{M}$, the metric tensor $\mathbf{g}_p$ defines an inner product on the tangent space $\mathcal{T}_p\mathcal{M}$, enabling the measurement of distances, angles, and curvatures.

For graph-structured data, we introduce a \textbf{metric tensor field over the graph} by associating each node $i \in \mathcal{V}$ with a metric tensor $\mathbf{G}_i \in \mathcal{S}_{++}^d$, where $\mathcal{S}_{++}^d$ denotes the cone of symmetric positive definite matrices. This creates a continuous geometric structure over the discrete graph domain and makes Christoffel symbols vanish, where each node's feature neighborhood in $\mathbb{R}^d$ possesses its own Riemannian geometry. The \textbf{geodesic distance} between sufficiently close points $\mathbf{x}, \mathbf{y} \in \mathbb{R}^d$ under the metric $\mathbf{G}$ of $\mathcal{T}_{\mathbf{x}}\mathcal{M}$ is:
{\setlength{\abovedisplayskip}{0pt}%
	\setlength{\belowdisplayskip}{0pt}%
	\setlength{\abovedisplayshortskip}{0pt}%
	\setlength{\belowdisplayshortskip}{0pt}%
	\[
	d_{\mathbf{G}}(\mathbf{x}, \mathbf{y}) = \sqrt{(\mathbf{x} - \mathbf{y})^T \mathbf{G} (\mathbf{x} - \mathbf{y})}\]}
The \textbf{Ricci curvature tensor} $\text{Ric}(\mathbf{G})$ characterizes how the metric changes across the manifold. In differential geometry, \textbf{Ricci flow}~\cite{chow2004ricci} evolves a metric tensor field $\{\mathbf{G}_i(t)\}$ according to $\frac{\partial \mathbf{G}_i}{\partial t} = -2\text{Ric}(\mathbf{G}_i)$, which smooths geometric irregularities while preserving topological structure, providing a principled approach for geometric regularization on graphs~\citep{hamilton1982three}. Due to space limitations, more Riemannian Geometry and Ricci curvature details are shown in Appendix~\ref{app:mathematical_foundations}.

\subsection{GNNs and Geometric Message Passing}

Consider an attributed graph $\mathcal{G} = (\mathcal{V}, \mathcal{E}, \mathbf{X})$ with nodes $\mathcal{V}$, edges $\mathcal{E}$, and node features $\mathbf{X} \in \mathbb{R}^{|\mathcal{V}| \times d}$. Standard message passing~\citep{gilmer2017neural} of GNNs updates node representations through:
{\setlength{\abovedisplayskip}{0pt}%
	\setlength{\belowdisplayskip}{0pt}%
	\setlength{\abovedisplayshortskip}{0pt}%
	\setlength{\belowdisplayshortskip}{0pt}%
	\[
	\mathbf{h}_i^{(\ell+1)}
	= \sigma\bigl(\mathbf{W}_s^{(\ell)}\mathbf{h}_i^{(\ell)}
	+ \sum_{j\in\mathcal{N}(i)} \mathbf{W}_m^{(\ell)}\mathbf{h}_j^{(\ell)}
	\bigr)
	\]}
where $\mathcal{N}(i)$ denotes the neighborhood of node $i$, $\mathbf{W}_s^{(\ell)}$ and $\mathbf{W}_m^{(\ell)}$ are learnable transformation matrices (parameters) on $\ell$-th layer, and $\sigma$ is a nonlinear activation function.

\textbf{Geometric message passing} enhances this by incorporating Riemannian structure by embedding graphs in non-Euclidean spaces. Existing approaches embed graphs in fixed geometries (hyperbolic, spherical) or discrete mixed-curvature spaces. However, as illustrated in Figure~\ref{fig:intro_wisc_vis}, real graphs exhibit \textbf{continuous geometric heterogeneity} and the local geometry varies across the network.

\subsection{Problem Formulation and Challenges}

The core challenge of this work is to move beyond the fixed-geometry paradigm. We aim to learn a node-adaptive geometric space jointly with the node representations $\mathbf{H}=\{\mathbf{h}_i\}_{i \in \mathcal{V}}$. This learned space is mathematically formulated as a Riemannian metric tensor field $\{\mathbf{G}_i\!\in\!\mathcal{S}_{++}^d\}_{i \in \mathcal{V}}$, where each node i is endowed with its own local metric $\mathbf{G}_i$. However, realizing this vision presents several key hurdles.

\textbf{First}, from a mathematical standpoint, each learned tensor $\mathbf{G}_i$ must remain symmetric positive definite (SPD) throughout training to constitute a valid Riemannian metric. \textbf{Second}, from a computational perspective, parameterizing a full $d\!\times\!d$ metric tensor for every node is prohibitive, scaling as $O(|\mathcal{V}|d^2)$ and hindering application to large graphs. \textbf{Third}, for learning stability, the geometric manifold must evolve smoothly; pathological curvatures could destabilize the training process and lead to poor generalization. \textbf{Finally}, to justify its increased complexity, the proposed framework must be provably more expressive than its fixed-geometry predecessors. Our \modelname framework is designed to address these challenges systematically.

%=================================================================
% 4. ADAPTIVE RIEMANNIAN GRAPH NEURAL NETWORKS
%=================================================================
\begin{figure*}[ht]
	\centering
    \includegraphics[trim= 3 0 2.5 0, clip, width=1.0\linewidth]{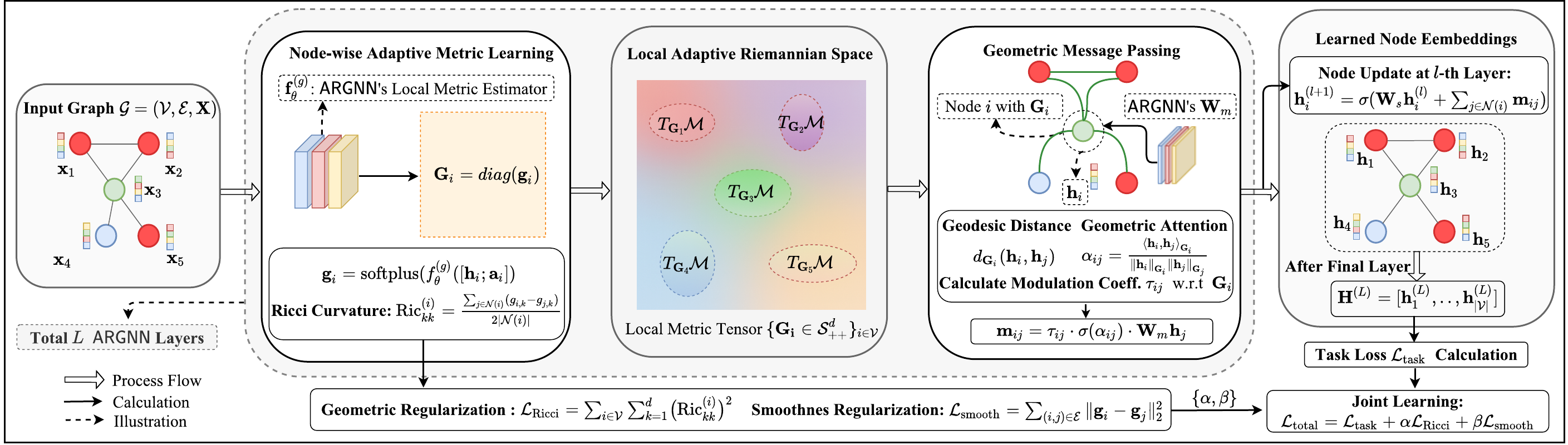}
	\caption{Diagram of our proposed \modelname, which jointly learns continuous, anisotropic metric tensor fields $\{\mathbf{G}_i\!\in\!\mathcal{S}_{++}^d\}_{i \in \mathcal{V}}$ and node embeddings $\mathbf{H}=\{\mathbf{h}_i\}_{i \in \mathcal{V}}$. The learned $\mathbf{G}_i$ gives beyond curvature information to depict the geometric diversity.}
	\label{fig:framework}
\end{figure*}
\section{Adaptive Riemannian Graph Neural Networks}
\label{sec:method}

This section introduces Adaptive Riemannian Graph Neural Networks (\modelname), a framework that learns a continuous, node-adaptive geometry tailored to the underlying graph structure. Our approach forgoes the manual selection of a geometric space and instead learns the geometry as part of the end-to-end training process. This section details the three core components of our framework: i) an efficient and interpretable parameterization of the node-wise metric tensor via a local metric estimator, ii) a geometric message passing scheme that leverages the learned metric, and iii) a Ricci flow-inspired regularization for stable geometric evolution. Figure~\ref{fig:framework} illustrates our approach's diagram.

\subsection{Learning an Anisotropic Metric Field}
\label{sec:metric_learning}

The foundational concept of our work is to equip the graph with a \textbf{Riemannian metric tensor field}, denoted as $\{\mathbf{G}_i \in \mathcal{S}_{++}^d\}_{i \in \mathcal{V}}$. Each metric tensor $\mathbf{G}_i$ is a symmetric positive-definite (SPD) matrix that defines a unique local geometry on the feature space neighborhood of node $i$. As stated before, a direct parameterization of a full $d\!\times\!d$ matrix for each node will be computationally prohibitive and lead to expensive downstream computations.

To address this, we propose a \textbf{principled simplification} by parameterizing each metric tensor $\mathbf{G}_i$ as a learnable \textbf{diagonal matrix}:
{\setlength{\abovedisplayskip}{0pt}%
	\setlength{\belowdisplayskip}{0pt}%
	\setlength{\abovedisplayshortskip}{0pt}%
	\setlength{\belowdisplayshortskip}{0pt}%
	\begin{equation}
		\mathbf{G}_i = \text{diag}(\mathbf{g}_i) = \text{diag}(g_{i,1}, g_{i,2}, \ldots, g_{i,d})
		\label{eq:diagonal_metric}
\end{equation}}
where $\mathbf{g}_i \in \mathbb{R}_{++}^d$ is a vector of positive diagonal elements. This diagonal parameterization is not an arbitrary choice but a motivated design with several key advantages:
{%
	\setlength{\leftmargini}{0pt}
	\setlength{\labelsep}{2pt}
	\setlength{\itemsep}{2pt}
	\setlength{\topsep}{1pt}
	\begin{enumerate}[i)]
		\item \textbf{Geometric Interpretation:} At its core, a metric tensor defines how to measure distances, angles, and curvatures. The standard Euclidean space, governed by the identity matrix, provides a single, global way of measurement. Our diagonal metric, $\mathbf{G}_i = \text{diag}(\mathbf{g}_i)$, acts as a local, learned modification to this standard. Specifically, it \textit{rescales the geometry independently along each feature axis}. The value of each diagonal element $g_{i,k}$ determines the stretch of the $k$-th dimension in the local vicinity of node $i$. This axis-aligned, non-uniform scaling type is formally known as an \textbf{anisotropic conformal transformation}~\citep{obata1970conformal,spivak1999comprehensive}. It endows the model with far greater flexibility than the isotropic scaling of constant-curvature spaces, while remaining more structured and tractable than a full metric tensor. The formal details of this geometric perspective are elaborated in Appendix~\ref{app:mathematical_foundations}.
		\item \textbf{Interpretability and Feature Decomposition:} The diagonal form assumes that the standard coordinate axes of the feature space align with the principal directions of the local geometry. This makes the model highly interpretable: each diagonal element $g_{i,k}$ directly quantifies the geometric importance or local scaling factor of the $k$-th feature dimension for node $i$. This is particularly effective when features are learned to be relatively disentangled.
		\item \textbf{Computational and Algorithmic Efficiency:} This parameterization dramatically reduces the complexity of learning and using the metric. The number of geometric parameters per node drops from $O(d^2)$ to $O(d)$. Furthermore, crucial geometric computations such as geodesic distance, matrix inversion, and eigendecomposition become highly efficient, which we will discuss later.
\end{enumerate}}
To learn these adaptive metrics, we employ a small neural network $f_\theta^{(g)}$ as the local metric estimator that maps a node's local structural information to its metric vector $\mathbf{g}_i$. Specifically, for each node $i$, we first aggregate its neighborhood features $\mathbf{a}_i = \frac{1}{|\mathcal{N}(i)|} \sum_{j \in \mathcal{N}(i)} \mathbf{h}_j$, and then compute its metric vector as:
{\setlength{\abovedisplayskip}{0pt}%
	\setlength{\belowdisplayskip}{0pt}%
	\setlength{\abovedisplayshortskip}{0pt}%
	\setlength{\belowdisplayshortskip}{0pt}%
	\begin{equation}
		\mathbf{g}_i = \text{softplus}\left(f_\theta^{(g)}\left([\mathbf{h}_i; \mathbf{a}_i]\right)\right)
		\label{eq:diagonal_network}
\end{equation}}
where $[\cdot; \cdot]$ denotes feature concatenation. The $\text{softplus}$ activation function ensures all elements of $\mathbf{g}_i$ are strictly positive, thereby naturally satisfying the SPD constraint for $\mathbf{G}_i$.

\subsection{Geometric Message Passing}
\label{sec:geometric_message_passing}
With the learned diagonal metric field, we now formulate a message passing scheme that is explicitly geometry-aware. The propagation of information between nodes is modulated by the local geometries, allowing the model to adapt its aggregation strategy based on the learned structure.
\subsubsection{Geodesic Distance and Principal Curvatures}
Under the learned metric $\mathbf{G}_i$, the geodesic distance (in this locally constant metric space) between the feature vectors of two nodes $i$ and $j$ simplifies to a weighted Euclidean distance:
{\setlength{\abovedisplayskip}{0pt}%
	\setlength{\belowdisplayskip}{0pt}%
	\setlength{\abovedisplayshortskip}{0pt}%
	\setlength{\belowdisplayshortskip}{0pt}%
	\begin{equation}
		\begin{aligned}
			d_{\mathbf{G}_i}(\mathbf{h}_i, \mathbf{h}_j) &= \sqrt{(\mathbf{h}_i - \mathbf{h}_j)^T \mathbf{G}_i (\mathbf{h}_i - \mathbf{h}_j)} \\& = \sqrt{\sum_{k=1}^d g_{i,k} (h_{i,k} - h_{j,k})^2}
		\end{aligned}
		\label{eq:geodesic_distance}
\end{equation}}

A major benefit of the diagonal form is that the principal directions of the geometry are aligned with the standard basis vectors $\{\mathbf{e}_k\}_{k=1}^d$, and the principal curvatures are directly related to the inverses of the metric elements $\{1/g_{i,k}\}_{k=1}^d$. This allows us to define a geometric modulation factor that captures how the geometry curves along the direction of the message.

\subsubsection{Geometric Modulation and Attention}
We introduce a geometric modulation coefficient $\tau_{ij}$ that weights the message from node $j$ to $i$ based on the projection of their directional vector onto the principal axes, modulated by the local curvature in those directions. Let $\mathbf{d}_{ij} = (\mathbf{h}_j - \mathbf{h}_i) / \|\mathbf{h}_j - \mathbf{h}_i\|_2$ be the normalized direction vector. We define $\tau_{ij}$ as:
{\setlength{\abovedisplayskip}{0pt}%
	\setlength{\belowdisplayskip}{0pt}%
	\setlength{\abovedisplayshortskip}{0pt}%
	\setlength{\belowdisplayshortskip}{0pt}%
	\begin{equation}
		\tau_{ij} = \sum_{k=1}^d d_{ij,k}^2 \cdot \tanh(-\log g_{i,k})
		\label{eq:geometric_modulation}
\end{equation}}
Here, the term $\tanh(-\log g_{i,k})$ acts as a curvature-dependent switch. If $g_{i,k}$ is large (high curvature, space contracts), the term approaches $-1$; if $g_{i,k}$ is small (low curvature, space expands), it approaches $+1$.

Complementing this, we introduce a geometric attention mechanism $\alpha_{ij}$ that computes the similarity between nodes within their respective local geometries:
{\setlength{\abovedisplayskip}{0pt}%
	\setlength{\belowdisplayskip}{0pt}%
	\setlength{\abovedisplayshortskip}{0pt}%
	\setlength{\belowdisplayshortskip}{0pt}%
	\begin{equation}
		\alpha_{ij} = \frac{\langle \mathbf{h}_i, \mathbf{h}_j \rangle_{\mathbf{G}_i}}{\|\mathbf{h}_i\|_{\mathbf{G}_i} \|\mathbf{h}_j\|_{\mathbf{G}_j}} = \frac{\sum_k g_{i,k} h_{i,k} h_{j,k}}{\sqrt{\sum_k g_{i,k} h_{i,k}^2} \sqrt{\sum_k g_{j,k} h_{j,k}^2}}
		\label{eq:geometric_attention}
\end{equation}}
This formulation measures the cosine similarity where the inner product is defined by node $i$'s metric, and each node's norm is measured in its own local metric.

\subsubsection{Node Representation Update} 
The final message $\mathbf{m}_{ij}$ from node $j$ to $i$ integrates the geometric modulation and attention with a standard learnable transformation $\mathbf{W}_m$:
{\setlength{\abovedisplayskip}{0pt}%
	\setlength{\belowdisplayskip}{0pt}%
	\setlength{\abovedisplayshortskip}{0pt}%
	\setlength{\belowdisplayshortskip}{0pt}%
	\begin{equation}
		\mathbf{m}_{ij} = \tau_{ij} \cdot \sigma(\alpha_{ij}) \cdot \mathbf{W}_m \mathbf{h}_j
		\label{eq:message_computation}
\end{equation}}
The node representations are then updated by aggregating messages from neighbors and combining them with the transformed self-representation, followed by a non-linear activation $\sigma$:
{\setlength{\abovedisplayskip}{0pt}%
	\setlength{\belowdisplayskip}{0pt}%
	\setlength{\abovedisplayshortskip}{0pt}%
	\setlength{\belowdisplayshortskip}{0pt}%
	\begin{equation}
		\mathbf{h}_i^{(l+1)} = \sigma\bigl(\mathbf{W}_s \mathbf{h}_i^{(l)} + \sum_{j \in \mathcal{N}(i)} \mathbf{m}_{ij}\bigr)
		\label{eq:node_update}
\end{equation}}
\subsection{Ricci Flow-Inspired Geometric Regularization}
\label{sec:ricci_flow}
To ensure that the learned metric field $\{\mathbf{G}_i\}$ is well-behaved and varies smoothly across the graph, we introduce two regularization terms inspired by the principles of discrete Ricci flow. This geometric evolution stabilizes training and prevents pathological curvatures.

First, we define a discrete approximation of the Ricci curvature for our diagonal metric field. Along the $k$-th principal direction at node $i$, let $\text{Ric}_{kk}$ be the abbreviation of $\text{Ric}(\mathbf{G}_i)_{kk}$, it can be approximated as:
{\setlength{\abovedisplayskip}{0pt}%
	\setlength{\belowdisplayskip}{0pt}%
	\setlength{\abovedisplayshortskip}{0pt}%
	\setlength{\belowdisplayshortskip}{0pt}%
	\begin{equation}
		\text{Ric}_{kk}^{(i)}
		= \frac{1}{2|\mathcal{N}(i)|}
		\sum_{j \in \mathcal{N}(i)} \frac{g_{i,k} - g_{j,k}}{d_{\text{graph}}(i,j)} .
		\label{eq:discrete_ricci}
	\end{equation}
}
where $d_{\text{graph}}(i,j)$ is the shortest path distance on the graph, and node $j\in \mathcal{N}(i)$ makes $d_{\text{graph}}=1$. See the detailed discussion in Appendix~\ref{app:discrete_ricci}. 
The \textbf{Ricci regularization} $\mathcal{L}_{\text{Ricci}}$ encourages the learned geometry to be Ricci-flat by penalizing the squared sum of Ricci curvatures, promoting a uniform curvature distribution:
{\setlength{\abovedisplayskip}{0pt}%
	\setlength{\belowdisplayskip}{0pt}%
	\setlength{\abovedisplayshortskip}{0pt}%
	\setlength{\belowdisplayshortskip}{0pt}%
	\begin{equation}
		\mathcal{L}_{\text{Ricci}} = \sum_{i \in \mathcal{V}} \sum_{k=1}^d (\text{Ric}_{kk}^{(i)})^2
		\label{eq:ricci_regularization}
\end{equation}}
The \textbf{smoothness regularization} $\mathcal{L}_{\text{smooth}}$ penalizes large differences in the metric vectors of adjacent nodes, ensuring the geometry varies smoothly over the graph structure:
{\setlength{\abovedisplayskip}{0pt}%
	\setlength{\belowdisplayskip}{0pt}%
	\setlength{\abovedisplayshortskip}{0pt}%
	\setlength{\belowdisplayshortskip}{0pt}%
	\begin{equation}
		\mathcal{L}_{\text{smooth}} = \sum_{(i,j) \in \mathcal{E}} \|\mathbf{g}_i - \mathbf{g}_j\|_2^2
		\label{eq:smoothness_regularization}
\end{equation}}
These regularizers guide the model to learn a coherent and stable geometric manifold that supports the downstream task. The total loss function then combines the primary task loss $\mathcal{L}_{\text{task}}$ with hyperparameters $\alpha,\beta$:
\begin{equation}
	\mathcal{L}_{\text{total}} = \mathcal{L}_{\text{task}} + \alpha \mathcal{L}_{\text{Ricci}} + \beta \mathcal{L}_{\text{smooth}}
	\label{eq:total_loss}
\end{equation}

\begin{table*}[ht]
	\centering
	\setlength{\tabcolsep}{2.2pt}
	\small
	\begin{tabular}[width = \linewidth]{l|ccccccccccc}
		\toprule
		\textbf{Dataset} & Cora & CiteSeer & PubMed & Actor & Chameleon & Squirrel & Texas & Cornell & Wisconsin \\
		\midrule
		Nodes & 2,708 & 3,327 & 19,717 & 7,600 & 2,277 & 5,201 & 183 & 183 & 251   \\
		Edges & 5,278 & 4,552 & 44,324 & 26,659 & 31,371 & 198,353 & 279 & 277 & 466   \\
		Features & 1,433 & 3,703 & 500 & 932 & 2,325 & 2,089 & 1,703 & 1,703 & 1,703   \\
		Classes & 7 & 6 & 3 & 5 & 5 & 5 & 5 & 5 & 5 \\
		\hline
		$\mathcal{H}$ & 0.825 & 0.718 & 0.792 & 0.215 & 0.247 & 0.217 & 0.057 & 0.301 & 0.196 \\
		\bottomrule
	\end{tabular}
	\caption{Dataset Statistics and Homophily Ratios $\mathcal{H}$}
	\label{tab:datasets}
\end{table*}

%=================================================================
% 5. THEORETICAL ANALYSIS
%=================================================================
\section{Theoretical Analysis}
\label{sec:theory}

We establish theoretical foundations for \modelname, proving its convergence, stability analysis, and universality. All detailed proofs are provided in Appendix~\ref{app:theory_proofs}.

\subsection{Convergence and Stability Analysis}
\label{sec:convergence_main}

\begin{theorem}[Convergence of Adaptive Geometry Learning]
	\label{thm:convergence}
	Consider \modelname with $L$ layers, hidden dimension $d$, and regularization weights $\alpha, \beta$. Under mild regularity conditions, the learned metric tensors $\{\mathbf{G}_i\}_{i \in \mathcal{V}}$ converge to a stationary point with:
	\begin{equation}
		\mathbb E\bigl[\|\nabla\mathcal L_{\mathrm{total}}^{(t)}\|^{2}\bigr]
		= O \Bigl(\tfrac1{\sqrt t}\,e^{-\mu_{\mathrm{eff}}t/L}\Bigr),
	\end{equation}
	where $\mu_{\text{eff}}$ is the effective curvature of the loss landscape. Optimal regularization hyperparameters satisfy:
	\begin{equation}
		\alpha^{\star}= \Theta \bigl(\tfrac{\mathcal{H}}{L}\min\!\bigl(1,\tfrac{d}{|\mathcal E|}\bigr)\bigr),\quad
		\beta^{\star}= \Theta \bigl(\tfrac{\mathcal{H}\sqrt d}{|\mathcal V|}\bigr).
	\end{equation}
	where $\mathcal{H} \in (0,1]$ is the dataset-dependent homophily ratio.
\end{theorem}
Assume the theoretical optimal hyperparameters as
\begin{equation}
	\alpha = \tfrac{c_1}{L} \cdot \min\bigl(1, \tfrac{d}{|\mathcal{E}|}\bigr),\quad 
	\beta = \tfrac{c_2\sqrt{d}}{|\mathcal{V}|}
\end{equation}
For constants $c_1$, $c_2$, we get an effective practice as follows,
\begin{proposition}[Homophily-Aware Constants]
	The dataset-dependent constants $c_1, c_2$ can be estimated as:
	\begin{equation}
		c_1 \approx  (1 - \mathcal{H}) + 0.1, \quad c_2 \approx 0.1 \cdot (1 + \mathcal{H})
	\end{equation}
	where $\mathcal{H} \in (0,1]$ is the graph homophily ratio.
\end{proposition}\label{prop:hyperparams_select}

The proof (Appendix~\ref{app:proof_convergence}) reveals key insights: i) Deeper networks require smaller $\alpha$ to maintain stability; ii) The smoothness weight $\beta$ should scale with feature dimension and inversely with graph size; iii) The interplay between geometric regularization and task loss creates a favorable optimization landscape. These theoretical guidelines directly inform our hyperparameter choices in experiments.

\subsection{Universal Approximation}
\begin{theorem}[Universal Geometric Framework]
	\label{thm:universality}
	\modelname provides a universal geometric framework that can generalize existing curvature-based GNNs. Specifically, GNNs operating on a fixed-curvature space (Euclidean, hyperbolic, spherical, or product manifolds) can be sufficiently approximated by a constrained parameterization of our learnable diagonal metric tensors $\mathbf{G}_i = \text{diag}(\mathbf{g}_i)$.
\end{theorem}
See Appendix~\ref{app:universality_proof} for the complete proof. We establish this by showing that fixed-curvature geometries correspond to specific constraints on $\mathbf{g}_i$: Euclidean ($\mathbf{g}_i = \mathbf{1}$), constant-curvature ($\mathbf{g}_i = c\mathbf{1}, c > 0$), and product manifolds (block-constant $\mathbf{g}_i$). 

\subsection{Computational Efficiency}

\begin{proposition}[Complexity Analysis]
	\label{thm:complexity}
	\modelname has time complexity $O((n+m)d^2)$ per layer,  where $n=|\mathcal{V}|$, $m=|\mathcal{E}|$, and $d$ is the hidden feature dimension, matching standard GNNs while providing continuous geometric adaptation.
\end{proposition}
See detailed proof in Appendix~\ref{app:proof_complexity}. We also give the empirical comparison and analysis in Appendix~\ref{app:extended_experiments}.

%=================================================================
% 6. EXPERIMENTS
%=================================================================
\section{Experiments}
\label{sec:experiments}
\subsection{Experimental Setup}
\label{sec:setup}

\textbf{Datasets}: We evaluate on nine widely used benchmark datasets spanning homophilic and heterophilic graphs: \textbf{Homophilic}: Cora, CiteSeer, PubMed~\citep{sen2008collective}; \textbf{Heterophilic}:  Actor~\citep{tang2009social}, Chameleon, Squirrel, Texas, Cornell, Wisconsin~\citep{pei2020geom}. Table~\ref{tab:datasets} provides detailed statistics.

\textbf{Baselines}: We compare against four categories of methods: 
i) \textbf{Traditional GNNs}: GCN~\citep{kipf2016semi}, GAT~\citep{velivckovic2018graph}, GraphSAGE~\citep{hamilton2017inductive}; 
ii) \textbf{Geometric GNNs}: HGCN~\citep{chami2019hyperbolic}, HGAT~\citep{zhang2021hyperbolic}, $\kappa$-GCN~\citep{bachmann2020constant}, $\mathcal{Q}$-GCN~\citep{xiong2022pseudo}; 
iii) \textbf{Heterophilic GNNs}: H2GCN~\citep{zhu2020beyond}, GPRGNN~\citep{chien2020adaptive}, FAGCN~\citep{bo2021beyond}; 
iv) \textbf{Recent Methods}: CUSP~\citep{grover2025cusp}, GNRF~\citep{chen2025graph}, CurvDrop~(+JKNet)~\citep{liu2023curvdrop}.

\textbf{Implementation Details}: We implemented on PyG~\citep{fey2019fast} and Geoopt~\citep{kochurov2020geoopt} frameworks. All experiments were run on NVIDIA GPUs with 10 random seeds, with the given $60\%/20\%/20\%$ splits from \citep{pei2020geom} for the node classification task. And using the same $80\%/5\%/15\%$ splits and settings from \citep{he2024pytorch} for the link existence prediction task.

\begin{table*}[ht]
	\centering
	\setlength{\tabcolsep}{2.2pt}
	\small
	\begin{tabular}{l|ccc|cccccc}
		\toprule
		& \multicolumn{3}{c|}{\textbf{Homophilic}} & \multicolumn{6}{c}{\textbf{Heterophilic}}\\
		\textbf{Method} & Cora & CiteSeer & PubMed & Actor & Chameleon & Squirrel & Texas & Cornell & Wisconsin \\
		\midrule
		GCN           & 75.21\ci{0.28} & 67.30\ci{1.05} & 83.75\ci{0.07} & 31.12\ci{0.96} & 61.16\ci{0.23} & 43.06\ci{0.33} & 75.61\ci{0.07} & 67.72\ci{1.19} & 59.46\ci{3.25} \\
		GAT           & 76.70\ci{0.13} & 66.23\ci{0.85} & 82.83\ci{0.22} & 32.65\ci{0.23} & 63.10\ci{0.77} & 43.90\ci{0.01} & 76.09\ci{0.77} & 74.01\ci{0.01} & 55.29\ci{5.40} \\
		GraphSAGE     & 71.88\ci{0.91} & 70.01\ci{0.64} & 81.09\ci{0.13} & 36.73\ci{0.01} & 59.99\ci{0.89} & 41.11\ci{1.16} & 77.11\ci{0.45} & 69.91\ci{0.24} & 81.18\ci{3.45} \\
		\midrule
		HGCN          & 78.50\ci{0.14} & 69.55\ci{0.39} & 83.72\ci{0.21} & 35.89\ci{0.29} & 60.18\ci{0.57} & 39.93\ci{0.35} & 88.11\ci{1.12} & 72.88\ci{1.15} & 86.70\ci{3.70}  \\
		HGAT          & 77.12\ci{0.01} & 70.12\ci{0.92} & 84.02\ci{0.19} & 35.12\ci{0.27} & 62.43\ci{0.59} & 41.78\ci{0.37} & 85.56\ci{1.10} & 73.12\ci{0.18} & 87.20\ci{3.50}  \\
		$\kappa$-GCN  & 78.71\ci{1.37} & 68.14\ci{0.34} & 85.18\ci{0.52} & 34.57\ci{0.26} & 62.12\ci{0.49} & 43.04\ci{0.31} & 85.03\ci{0.63} & 86.36\ci{0.64} & 86.90\ci{3.80} \\
		$\mathcal{Q}$-GCN & 79.64\ci{0.38} & 71.15\ci{1.11} & 84.76\ci{0.13} & 32.24\ci{0.65} & 61.83\ci{1.01} & 46.65\ci{0.90} & 82.76\ci{0.07} & 83.90\ci{0.71} & 86.50\ci{4.10} \\
		\midrule
		H2GCN         & {\color{orange}{82.70\ci{0.90}}} & 71.10\ci{0.80} & 84.60\ci{0.50} & 35.90\ci{1.20} & 60.10\ci{2.20} & 48.20\ci{2.00} & 84.90\ci{6.00} & 82.20\ci{4.20} & 86.67\ci{2.91}  \\
		GPRGNN        & 79.49\ci{0.31} & 67.61\ci{0.38} & 84.07\ci{0.09} & 37.43\ci{1.09} & 65.09\ci{0.43} & 47.51\ci{0.23} & 88.34\ci{0.09} & 87.21\ci{0.70} & {88.07\ci{1.00}}  \\
		FAGCN         & 82.50\ci{0.50} & 71.30\ci{0.60} & 84.30\ci{0.40} & 35.80\ci{1.10} & 66.90\ci{1.80} & {\color{orange}{52.30\ci{1.70}}} & 87.80\ci{3.20} & 85.50\ci{4.10} & 87.30\ci{3.60} \\
		\midrule
		CUSP          & {\color{purple}{83.45\ci{0.15}}} & {\color{purple}{74.21\ci{0.02}}} & {\color{purple}{87.99\ci{0.45}}} & {\color{purple}{41.91\ci{0.11}}} & {\color{purple}{70.23\ci{0.61}}} & {\color{purple}{52.98\ci{0.25}}} & \color{orange}{89.43\ci{2.72}} & {88.31\ci{1.09}} & {\color{purple}{88.30\ci{0.80}}} \\
		GNRF          & 82.10\ci{0.80} & {\color{orange}{73.50\ci{0.50}}} & {\color{orange}{86.80\ci{0.40}}} & {\color{orange}{40.80\ci{0.90}}} & {\color{orange}{68.90\ci{1.20}}} & 49.82\ci{1.50} & {\color{purple}{90.80\ci{1.30}}} & {\color{purple}{90.50\ci{1.10}}} & {\color{orange}{88.10\ci{1.40}}} \\
		CurvDrop      & 82.50\ci{0.70} & 72.80\ci{0.60} & 85.20\ci{0.50} & 39.50\ci{1.00} & 67.30\ci{1.40} & 50.10\ci{1.30} & 88.20\ci{2.10} & {\color{orange}{89.80\ci{1.80}}} & 87.50\ci{1.90} \\
		\midrule
		\textbf{\modelname} & \textbf{86.83\ci{0.84}} & \textbf{74.80\ci{1.26}} & \textbf{88.59\ci{0.25}} & \textbf{{42.18\ci{0.33}}} & \textbf{70.44\ci{1.27}}  & \textbf{53.12\ci{1.45}} & \textbf{{92.28\ci{1.59}}} & \textbf{{90.85\ci{0.33}}} & \textbf{90.65\ci{2.34}}\\ 
		\bottomrule
	\end{tabular}
	\caption{Node Classification Performance (Avg.\ F1-score \%($\uparrow$) $\pm$ $95\%$ Confidence Interval) on Benchmark Datasets. The {\bf{bold}}, {{\color{purple}{purple}} and {\color{orange}{orange}}} numbers denote the best, second best, and third best performances, respectively.}
	\label{tab:node_classification}
\end{table*}

\subsection{Main Results}
\label{sec:main_results}

Table~\ref{tab:node_classification} presents mean F1-scores with $95\%$ Confidence Intervals~(CI) from 10 runs, chosen as the most comprehensive metric for multi-class imbalanced datasets. Additional metrics (Accuracy, AUROC, and AUPRC) are provided in Appendix~\ref{app:extended_experiments}. \modelname achieves the best performance on all 9 datasets, with significant improvements across both homophilic and heterophilic graphs. For instance, on Cora (homophilic), we improve $3.38\%$ over CUSP, while on Wisconsin (heterophilic), we gain $2.35\%$ over the best baseline.

Our results reveal three key patterns: i) Fixed-curvature methods like HGCN show high variance on heterophilic graphs, but generally overperform Euclidean-space GNNs; ii) Mixed-curvature approaches (CUSP, GNRF) achieve more consistent results but remain limited by discrete geometry choices; iii) \modelname's continuous adaptation provides robust performance across the homophily spectrum.

For link prediction (Table~\ref{tab:link_prediction} in Appendix), we report average AUROC as the comprehensive metric for link prediction. \modelname achieves the highest scores on all 9 datasets, with particularly strong performance on geometrically complex graphs (Actor: $76.40\%$ vs. GNRF's $73.50\%$ and CUSP's $74.20\%$), validating our adaptive metric learning approach.

\subsection{Ablation Studies}
\label{sec:ablation}

We conduct systematic ablations to validate our theoretical framework and understand the contribution of each component. Figure~\ref{fig:ablation} presents results on three representative datasets: Cora (homophilic, $\mathcal{H}=0.825$), Actor (heterophilic, $\mathcal{H}=0.215$), and Wisconsin (mixed, $\mathcal{H}=0.196$).

\begin{figure}[h]
	\centering
	\begin{subfigure}[b]{0.49\columnwidth}
		\centering
		\fbox{\includegraphics[width=\linewidth]{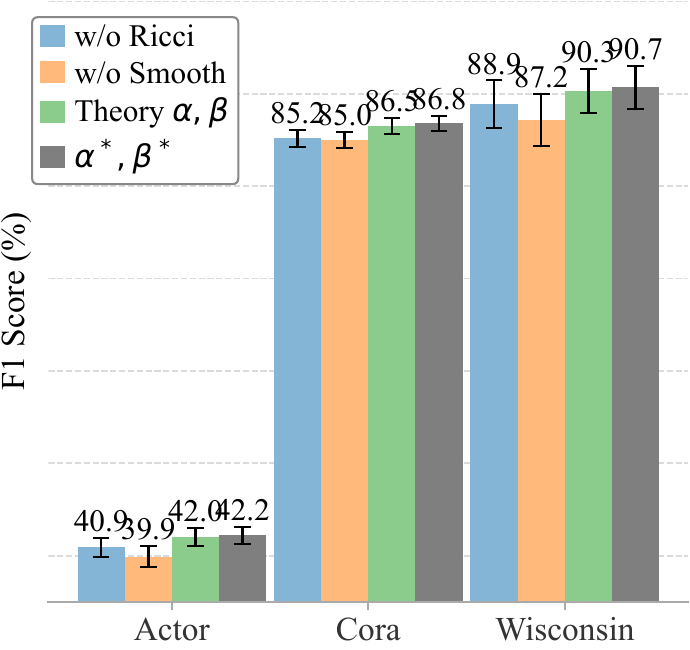}}
	\end{subfigure}
	\hfill
	\begin{subfigure}[b]{0.49\columnwidth}
		\centering
		\fbox{\includegraphics[width=\linewidth]{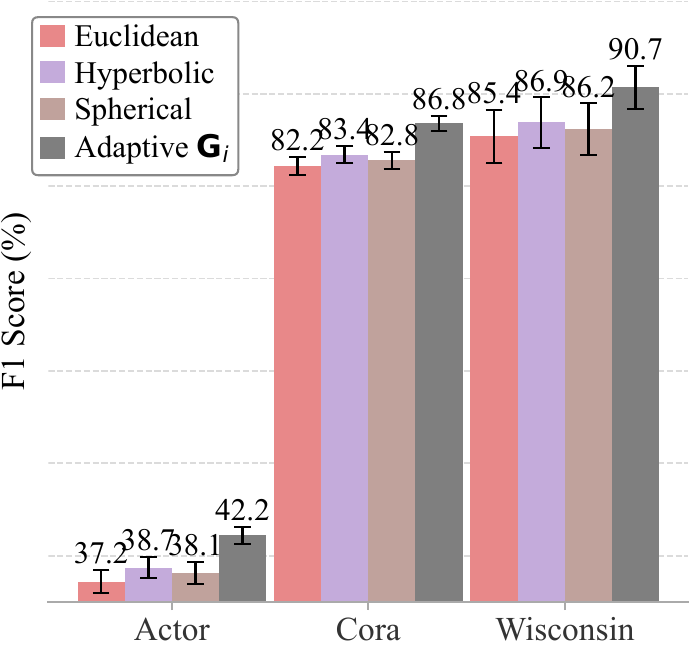}}
	\end{subfigure}
	\caption{Ablation studies on three datasets. \textbf{(a)} Impact of regularization components, including theoretically optimal $\alpha$, $\beta$ settings v.s. grid search optimal $\alpha^*,\beta^*$ \textbf{(b)} Comparison with fixed $\mathbf{G}_i \in \{\mathbf{I},0.5\mathbf{I},2\mathbf{I}\}$ for Euclidean, Hyperbolic and Spherical. Error bars show 95\% CI from 10 runs.}
	\label{fig:ablation}
\end{figure}

\textbf{Regularization Components (Fig.~\ref{fig:ablation}a):} Both Ricci and smoothness regularizations are essential but serve different roles. Removing Ricci regularization causes larger performance drops on heterophilic graphs (Actor: $-1.3\%$, Wisconsin: $-1.8\%$), where diverse geometries require careful control. Smoothness regularization is more critical for maintaining coherent metric fields, with Wisconsin showing the largest drop ($-3.5\%$) due to its mixed homophily structure requiring smooth transitions between geometric regimes.

Remarkably, our theory-guided hyperparameters achieve near-optimal performance. Using theoretical optimal $\alpha = \tfrac{c_1}{L} \cdot \min(1, \tfrac{d}{|\mathcal{E}|})$ and $\beta = \tfrac{c_2\sqrt{d}}{|\mathcal{V}|}$ with $c_1, c_2$ from Theorem~\ref{thm:convergence}, we obtain F1 scores within 0.5\% of exhaustive grid search. This validates our convergence analysis and provides practitioners with principled hyperparameter selection. 
% See all details of hyperparameter settings and analysis in Appendix~E. 

\textbf{Geometry Parameterizations (Fig.~\ref{fig:ablation}b):} Fixed geometries severely limit expressiveness across all graph types. Euclidean geometry~($\mathbf{G}_i=\mathbf{I}$), optimal for neither hierarchies nor cycles, shows the worst performance. Fixed hyperbolic geometry ($\mathbf{G}_i=0.5\mathbf{I}$) improves slightly on tree-like substructures but fails on dense regions. Around $5\%$ improvement of \modelname over fixed geometries on heterophilic graphs demonstrates the necessity of adaptive metrics.

The performance gap is most pronounced on Actor (5.0\% over Euclidean), where the co-occurrence network contains both star-like patterns (requiring hyperbolic geometry) and cliques (requiring spherical geometry), and our adaptive approach can simultaneously capture these diverse structures. 
Due to the space limit, detailed ablation results like additional parameterization comparisons and detailed computational overhead analysis are provided in Appendix~\ref{app:extended_experiments}.

\textbf{Computational Efficiency:} \modelname's diagonal parameterization achieves superior efficiency (nearly to HGCN): around $35\%$ faster than CUSP by avoiding costly product manifold projections, with around $40\%$ lower memory usage than full tensor methods. The $O(d)$ complexity per metric operation (v.s. $O(d^2)$ for full tensors) enables scaling to large graphs. Detailed benchmarks and scalability analysis are provided in Appendix~\ref{app:extended_experiments}.

\begin{figure}[h]
	\centering
	\begin{subfigure}[b]{0.31\columnwidth}
		\centering
		\fbox{\includegraphics[width=0.98\linewidth]{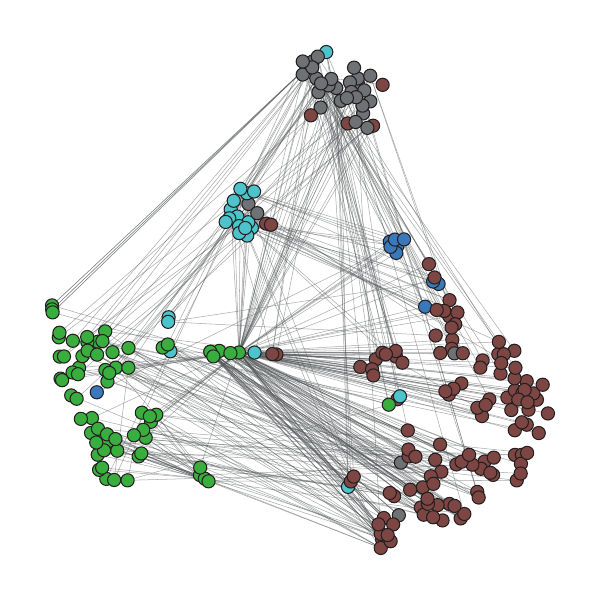}}
	\end{subfigure}
	\begin{subfigure}[b]{0.31\columnwidth}
		\centering
		\fbox{\includegraphics[width=0.98\linewidth]{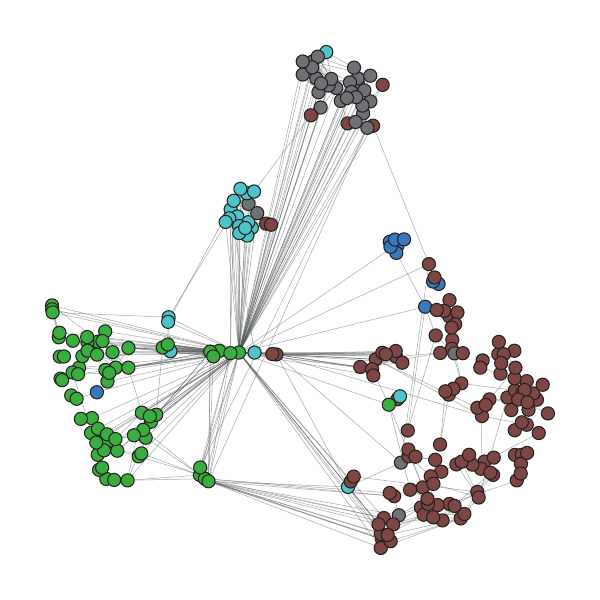}}
	\end{subfigure}
	\begin{subfigure}[b]{0.31\columnwidth}
		\centering
		\fbox{\includegraphics[trim= 60 50 60 70, clip, width=0.98\linewidth]{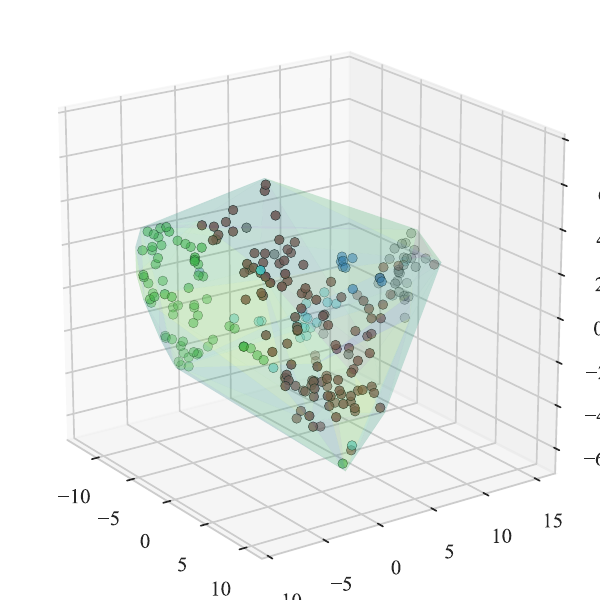}}
	\end{subfigure}
	\caption{Geometry learned by \modelname on Wisconsin. \textbf{Left}: Original graph topology colored by class under the layout from the learned embedding projection to 2-D. \textbf{Middle}: Degree-preserving rewiring based on learned geodesic distances reveals clearer class separation. \textbf{Right}: 3-D t-SNE embedding with curvature visualization shows adaptive geometry, the translucent hull is coloured by the magnitude of the mean curvature 
		(\textcolor{violet}{violet} $\!\to\!$ flat, \textcolor{yellow}{yellow} $\!\to\!$ strongly curved)}
	\label{fig:wisconsin_analysis}
\end{figure}
\subsection{Learned Geometry Analysis}
\label{sec:geometry_analysis}
Figure~\ref{fig:wisconsin_analysis} visualizes how \modelname discovers latent geometric structure. The geodesic rewiring demonstrates improved class separation through learned metrics.
\begin{figure}[h!]
	\centering
	\fbox{\includegraphics[trim= 0 7 0 7, clip, width=0.95\columnwidth]{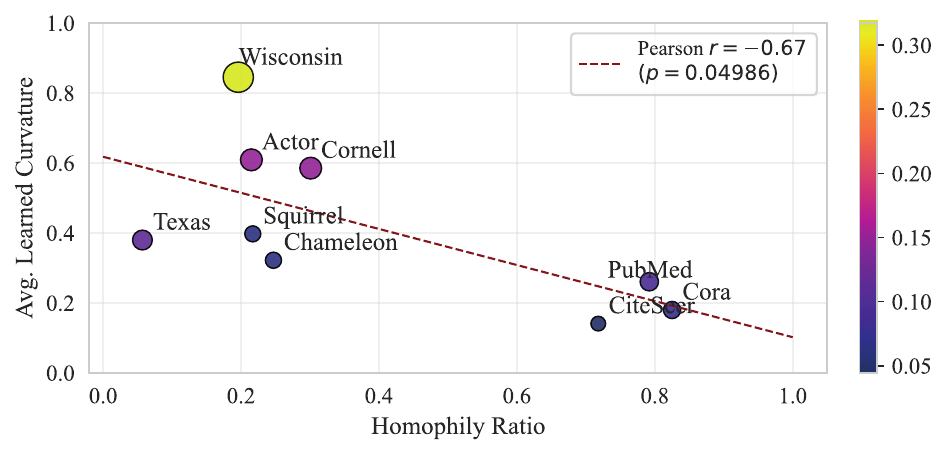}}
	\caption{\textbf{Homophily $\mathcal{H}$ vs. learned geometry.}  
		Avg. learned curvature across datasets with marker size/colour encodes the mean \emph{Neighbour-Relative Metric Dispersion}~(NRMD)}\label{fig:homophily_correlation}
\end{figure}

Figure~\ref{fig:homophily_correlation} reveals that highly homophilic graphs (Cora, CiteSeer) are in the low curvature and \emph{Neighbour-Relative Metric Dispersion}~($\text{NRMD}\!=\!\frac{1}{|\mathcal{E}|}
\sum_{(i,j)\in \mathcal{E}}\!
\tfrac{\|\mathbf g_i-\mathbf g_j\|_2}
{\tfrac{1}{2}(\|\mathbf g_i\|_2+\|\mathbf g_j\|_2)}$) quadrant, whereas heterophilic datasets (Actor, Wisconsin) display both larger curvature and greater metric dispersion.  
The joint trend confirms that \modelname\ bends its geometry more and allows higher learned metric diversity when neighbourhood labels are mixed, underscoring its dataset-adaptive behaviour. Due to the space limitation, see Appendix~\ref{app:extended_experiments} for more evidence and analysis.

%=================================================================
% 7. CONCLUSION
%=================================================================
\section{Conclusion}
\label{sec:conclusion}

We have introduced \modelname, a novel framework that learns continuous, anisotropic metric tensor fields to capture the geometric diversity inherent in real-world graphs. 

\section*{Acknowledgements}
This work was supported by the National Natural Science Foundation of China under Grant No.62376236.

%=================================================================
% REFERENCES
%=================================================================
\bibliography{reference.bib}

\clearpage
%=================================================================
% APPENDIX
%=================================================================
\newpage
\appendix
\onecolumn
%=================================================================
% NOTATION TABLE
%=================================================================
\section{Notation and Symbols}
This appendix provides a comprehensive reference (Table~\ref{tab:notation_summary}) for the principal notation used throughout the paper. While we strive to maintain consistent notation, specialized symbols that appear exclusively in proofs or specific derivations are defined in their respective contexts. We adhere to the following conventions:

\begin{table}[H]
\centering
\setlength{\tabcolsep}{1.75cm}
\small
\resizebox{0.80\textwidth}{!}{
\begin{tabular}[width = \textwidth]{ll}
\toprule
\textbf{Symbol} & \textbf{Definition} \\
\midrule
\multicolumn{2}{l}{\textit{Graph Structure and Node Features}} \\
$\mathcal{G} = (\mathcal{V}, \mathcal{E}, \mathbf{X})$ & Attributed graph \\
$\mathcal{V}$ & Set of vertices/nodes, $|\mathcal{V}| = n$ \\
$\mathcal{E}$ & Set of edges, $|\mathcal{E}| = m$ \\
$\mathbf{X} \in \mathbb{R}^{n \times d}$ & Node feature matrix \\
$\mathbf{x}_i \in \mathbb{R}^d$ & Feature vector for node $i$ \\
$\mathcal{N}(i)$ & Neighborhood of node $i$ \\
$\mathbf{A} \in \{0,1\}^{n \times n}$ & Adjacency matrix \\
$d_{\text{graph}}(i,j)$ & Graph distance between nodes $i$ and $j$ \\
$\mathcal{H}$ & Homophily ratio of graph \\
\midrule
\multicolumn{2}{l}{\textit{Neural Network Components}} \\
$\mathbf{H}^{(\ell)} \in \mathbb{R}^{n \times d_\ell}$ & Hidden representations at layer $\ell$ \\
$\mathbf{h}_i^{(\ell)} \in \mathbb{R}^{d_\ell}$ & Hidden representation of node $i$ at layer $\ell$ \\
$L$ & Total number of layers \\
$\mathbf{W}_m \in \mathbb{R}^{d_{\ell} \times d_{\ell}}$ & Message transformation matrix \\
$\mathbf{W}_a \in \mathbb{R}^{d_{\ell} \times d_{\ell}}$ & Attention weight matrix \\
$\mathbf{W}_u \in \mathbb{R}^{d_{\ell} \times d_{\ell}}$ & Node update matrix \\
$\sigma(\cdot)$ & Sigmoid activation function \\
$\text{ReLU}(\cdot)$ & Rectified linear unit activation \\
\midrule
\multicolumn{2}{l}{\textit{Geometric Components}} \\
$\mathbf{G}_i \in \mathcal{S}_{++}^d$ & Metric tensor at node $i$ \\
$\mathbf{g}_i \in \mathbb{R}_{++}^d$ & Diagonal elements of $\mathbf{G}_i$ \\
$g_{i,k} \in \mathbb{R}_{++}$ & $k$-th diagonal element of $\mathbf{G}_i$ \\
% $\epsilon$ & Lower bound for metric elements ($g_{i,k} \geq \epsilon > 0$) \\
$\epsilon$ & Lower bound for metric elements \\
$d_{\mathbf{G}_i}(\cdot,\cdot)$ & Geodesic distance under metric $\mathbf{G}_i$ \\
$\mathbf{d}_{ij} \in \mathbb{R}^d$ & Normalized direction vector from node $i$ to $j$ \\
$d_{ij,k} \in \mathbb{R}$ & $k$-th component of $\mathbf{d}_{ij}$ \\
$\tau_{ij} \in \mathbb{R}$ & Geometric modulation coefficient \\
$\alpha_{ij} \in \mathbb{R}$ & Geometric attention coefficient \\
$\text{Ric}_{kk}^{(i)}$ & Discrete Ricci curvature at node $i$, dimension $k$ \\
\midrule
\multicolumn{2}{l}{\textit{Loss Functions and Optimization}} \\
$\mathcal{L}_{\text{task}}$ & Task-specific loss function \\
$\mathcal{L}_{\text{total}}$ & Total loss function \\
$\mathcal{L}_{\text{Ricci}}$ & Ricci curvature regularization term \\
$\mathcal{L}_{\text{smooth}}$ & Geometric smoothness regularization term \\
$\alpha$ & Ricci regularization weight \\
$\beta$ & Smoothness regularization weight \\
$\theta$ & Neural network parameters \\
$\nabla_\theta$ & Gradient with respect to parameters $\theta$ \\
\midrule
\multicolumn{2}{l}{\textit{General Mathematical Notation}} \\
$\mathbb{R}$ & Real numbers \\
$\mathbb{R}_{++}$ & Positive real numbers \\
$\mathbb{R}^d$ & $d$-dimensional real vector space \\
$\mathbb{R}^{m \times n}$ & Space of $m \times n$ real matrices \\
% $\mathcal{S}$ & Set \\
$\mathcal{S}_{++}^d$ & Cone of $d \times d$ SPD matrices \\
$\mathcal{O}(f)$ & Big-O notation for asymptotic complexity \\
$\|\cdot\|$, $\|\cdot\|_2$ & Euclidean norm \\
$\|\cdot\|_{\mathbf{G}}$ & Norm under metric $\mathbf{G}$ \\
$\|\cdot\|_F$ & Frobenius norm \\
$\langle \cdot, \cdot \rangle$ & Standard inner product \\
$\langle \cdot, \cdot \rangle_{\mathbf{G}}$ & Inner product under metric $\mathbf{G}$\\
$\circ$ & Hadamard (element-wise) product \\
% $\succ$ & Strict ordering relation (for expressiveness hierarchy) \\
$\succ$, $\succeq$ & Ordering relation \\
$\text{diag}(\cdot)$ & Diagonal matrix constructor \\
% $k$-WL & $k$-Weisfeiler-Lehman test \\
\bottomrule
\end{tabular}}
\caption{Summary of notation throughout this paper. All vectors and matrices are denoted in \textbf{bold}.}
\label{tab:notation_summary}
\end{table}

%=================================================================
% APPENDIX A: MATHEMATICAL FOUNDATIONS
%=================================================================
\section{Mathematical Foundations}
\label{app:mathematical_foundations}

This appendix provides the mathematical foundations underlying our Adaptive Riemannian Graph Neural Networks framework.
\subsection{Riemannian Geometry on Discrete Structures}
\label{app:riemannian_discrete}
% defines the concept of a discrete Riemannian graph by associating a metric tensor G_i with each node i,
\subsubsection{Metric Tensors and Geodesics}

A Riemannian manifold $(\mathcal{M}, \mathbf{g})$ consists of a smooth manifold $\mathcal{M}$ equipped with a metric tensor $\mathbf{g}$ that varies smoothly across the manifold. For discrete graphs, we discretize this concept by associating each node $i \in \mathcal{V}$ with a metric tensor $\mathbf{G}_i \in \mathcal{S}_{++}^d$, where $\mathcal{S}_{++}^d$ denotes the cone of $d \times d$ symmetric positive definite matrices.

\begin{definition}[Discrete Riemannian Graph]
\label{def:discrete_riemannian_graph}
A discrete Riemannian graph is a tuple $(\mathcal{G}, \{\mathbf{G}_i\}_{i \in \mathcal{V}})$ where $\mathcal{G} = (\mathcal{V}, \mathcal{E}, \mathbf{X})$ is an attributed graph and $\{\mathbf{G}_i\}_{i \in \mathcal{V}}$ is a collection of metric tensors such that $\mathbf{G}_i \in \mathcal{S}_{++}^d$ for all $i \in \mathcal{V}$.
\end{definition}

The metric tensor $\mathbf{G}_i$ defines a local inner product on the tangent space at node $i$:
\begin{equation}
\langle \mathbf{u}, \mathbf{v} \rangle_{\mathbf{G}_i} = \mathbf{u}^T \mathbf{G}_i \mathbf{v}
\label{eq:app_inner_product}
\end{equation}

The induced norm and distance are:
\begin{align}
\|\mathbf{u}\|_{\mathbf{G}_i} &= \sqrt{\langle \mathbf{u}, \mathbf{u} \rangle_{\mathbf{G}_i}} = \sqrt{\mathbf{u}^T \mathbf{G}_i \mathbf{u}} \label{eq:app_norm}\\
d_{\mathbf{G}_i}(\mathbf{x}, \mathbf{y}) &= \|\mathbf{x} - \mathbf{y}\|_{\mathbf{G}_i} = \sqrt{(\mathbf{x} - \mathbf{y})^T \mathbf{G}_i (\mathbf{x} - \mathbf{y})} \label{eq:app_distance}
\end{align}

\subsubsection{Geodesics and Parallel Transport}

In continuous Riemannian geometry, geodesics are curves that locally minimize distance. For our discrete setting, we approximate geodesics using straight lines in the embedding space, with distances measured according to the local metric tensor.

\begin{definition}[Discrete Geodesic]
\label{def:discrete_geodesic}
Given two points $\mathbf{x}, \mathbf{y} \in \mathbb{R}^d$ and a metric tensor $\mathbf{G} \in \mathcal{S}_{++}^d$, the discrete geodesic from $\mathbf{x}$ to $\mathbf{y}$ is the straight line $\gamma(t) = (1-t)\mathbf{x} + t\mathbf{y}$ for $t \in [0,1]$, with length:
\begin{equation}
\text{length}(\gamma) = \int_0^1 \|\dot{\gamma}(t)\|_{\mathbf{G}} dt = \|\mathbf{y} - \mathbf{x}\|_{\mathbf{G}}
\label{eq:app_geodesic_length}
\end{equation}
\end{definition}

This approximation is exact when the metric tensor is constant along the path, which holds for our diagonal parameterization in feature space.

\subsubsection{Exponential and Logarithmic Maps}

The exponential map $\exp_p: \mathcal{T}_p\mathcal{M} \rightarrow \mathcal{M}$ maps tangent vectors to points on the manifold, while the logarithmic map $\log_p: \mathcal{M} \rightarrow \mathcal{T}_p\mathcal{M}$ is its inverse (when well-defined).

For our discrete setting with metric $\mathbf{G}_i$ at node $i$:
\begin{align}
\exp_{\mathbf{h}_i}(\mathbf{v}) &= \mathbf{h}_i + \mathbf{v} \label{eq:app_exp_map}\\
\log_{\mathbf{h}_i}(\mathbf{h}_j) &= \mathbf{h}_j - \mathbf{h}_i \label{eq:app_log_map}
\end{align}

These simplified forms arise from our Euclidean embedding space with varying metrics.

\subsubsection{The Line Element and Anisotropic Conformal Metrics}
\label{app:line_element}
A powerful way to understand the geometry defined by a metric tensor is through its \textit{squared line element}, denoted $ds^2$. This expression defines the infinitesimal squared distance between two nearby points.

In a standard $d$-dimensional Euclidean space, the metric tensor is the identity matrix $\mathbf{G} = \mathbf{I}$, and its line element is given by the Pythagorean theorem:
\begin{equation}
    ds^2 = \sum_{k=1}^d (dx^k)^2
\end{equation}
where $dx^k$ represents an infinitesimal displacement along the $k$-th coordinate axis.

A \textbf{conformal transformation} of a metric re-scales all distances at a point by the same factor, thus preserving angles~\citep{obata1970conformal,spivak1999comprehensive}. This corresponds to an isotropic scaling of the metric tensor, $\mathbf{G}' = \lambda(x) \mathbf{I}$, where $\lambda(x)$ is a positive scalar function. The line element becomes $ds^2 = \lambda(x) \sum_{k=1}^d (dx^k)^2$. Constant-curvature spaces like the Hyperbolic or Spherical space can be modeled as specific types of conformal transformations of Euclidean space.

Our model's diagonal metric, $\mathbf{G}_i = \text{diag}(g_{i,1}, \ldots, g_{i,d})$, defines a more general transformation known as an \textbf{anisotropic conformal transformation}. Here, instead of a single scaling factor, we have a unique scaling factor $g_{i,k}$ for each dimension $k$. The squared line element in the local geometry of node $i$ is thus:
\begin{equation}
    ds^2 = \sum_{k=1}^d g_{i,k} (dx^k)^2
\end{equation}
This is precisely the geometry induced by our metric tensor, as can be seen by considering two nearby points $\mathbf{x}$ and $\mathbf{x} + d\mathbf{x}$:
\begin{equation}
    d_{\mathbf{G}_i}(\mathbf{x}, \mathbf{x} + d\mathbf{x})^2 = (d\mathbf{x})^T \mathbf{G}_i (d\mathbf{x}) = \sum_{k=1}^d g_{i,k} (dx^k)^2
\end{equation}
This geometric structure preserves the orthogonality of the standard coordinate axes but, unlike a true isotropic conformal map, does not preserve angles between arbitrary vectors. This property is key to its flexibility, as it allows the model to learn that certain feature interactions are more or less important for different nodes. This provides a principled geometric foundation for the parameterization chosen in our \modelname framework.

\subsection{Symmetric Positive Definite (SPD) Matrices and Diagonal Parameterization}
\label{app:spd_matrices}
% focus on the properties of the diagonal parameterization used in the paper.

The space of symmetric positive definite (SPD) matrices $\mathcal{S}_{++}^d$ forms a Riemannian manifold itself and is central to our framework. A matrix $\mathbf{G} \in \mathbb{R}^{d \times d}$ is in $\mathcal{S}_{++}^d$ if it is symmetric ($\mathbf{G} = \mathbf{G}^T$) and for any non-zero vector $\mathbf{x} \in \mathbb{R}^d$, $\mathbf{x}^T \mathbf{G} \mathbf{x} > 0$.

In our \modelname framework, we adopt a diagonal parameterization for the metric tensor $\mathbf{G}_i$ at each node $i$:
\begin{equation}
\mathbf{G}_i = \text{diag}(\mathbf{g}_i) = \text{diag}(g_{i,1}, g_{i,2}, \ldots, g_{i,d})
\end{equation}
This simplification has significant geometric and computational implications.

\begin{lemma}[Positive Definiteness of Diagonal Metrics]
\label{lemma:app_diag_spd}
A diagonal matrix $\mathbf{G} = \text{diag}(g_1, \ldots, g_d)$ is symmetric positive definite if and only if all its diagonal elements are strictly positive, i.e., $g_k > 0$ for all $k=1, \ldots, d$.
\end{lemma}
\begin{proof}
Symmetry is trivial. For positive definiteness, for any non-zero $\mathbf{x} \in \mathbb{R}^d$:
\begin{equation}
\mathbf{x}^T\mathbf{G}\mathbf{x} = \sum_{k=1}^d g_k x_k^2
\end{equation}
If all $g_k > 0$, since $\mathbf{x} \neq \mathbf{0}$, at least one $x_k^2 > 0$, so the sum is strictly positive. Conversely, if there exists some $g_j \leq 0$, we can choose $\mathbf{x} = \mathbf{e}_j$ (the $j$-th standard basis vector) to get $\mathbf{x}^T\mathbf{G}\mathbf{x} = g_j \leq 0$, violating the SPD condition.
\end{proof}
This lemma provides a simple and efficient way to enforce the SPD constraint by ensuring the positivity of the learned vector $\mathbf{g}_i$, which we achieve using the softplus activation function.

\subsubsection{Properties of Diagonal Metrics}
The diagonal structure leads to highly efficient computations for key geometric operations:
\begin{itemize}
    \item \textbf{Inverse:} The inverse is simply the diagonal matrix of reciprocal elements: $\mathbf{G}_i^{-1} = \text{diag}(1/g_{i,1}, \ldots, 1/g_{i,d})$.
    \item \textbf{Determinant:} The determinant is the product of the diagonal elements: $\det(\mathbf{G}_i) = \prod_{k=1}^d g_{i,k}$.
    \item \textbf{Eigendecomposition:} The eigendecomposition is trivial. The eigenvalues are the diagonal elements $\{g_{i,k}\}_{k=1}^d$ and the corresponding eigenvectors are the standard basis vectors $\{\mathbf{e}_k\}_{k=1}^d$. This avoids costly $O(d^3)$ computations.
\end{itemize}

\subsection{Discrete Ricci Curvature}
\label{app:discrete_ricci}

Ricci curvature, a fundamental concept in differential geometry, measures how volumes change under parallel transport. We present its discrete analogues suitable for graph neural networks.

\subsubsection{Ollivier-Ricci Curvature}

The Ollivier-Ricci curvature~\citep{ollivier2009ricci} provides a discrete analogue based on optimal transport theory.

\begin{definition}[Ollivier-Ricci Curvature]
\label{def:app_ollivier_ricci}
For an edge $(i,j) \in \mathcal{E}$, let $\mu_i$ and $\mu_j$ be probability measures on the neighborhoods of $i$ and $j$ respectively. The Ollivier-Ricci curvature is:
\begin{equation}
\kappa_{ij} = 1 - \frac{W_1(\mu_i, \mu_j)}{d_{\text{graph}}(i,j)}
\label{eq:app_ollivier_ricci}
\end{equation}
where $W_1$ is the Wasserstein-1 distance.
\end{definition}

For our metric tensor framework, we adapt this to use learned geometric distances:
\begin{equation}
\kappa_{ij}^{\mathbf{G}} = 1 - \frac{W_1^{\mathbf{G}}(\mu_i, \mu_j)}{d_{\mathbf{G}_i}(\mathbf{h}_i, \mathbf{h}_j)}
\label{eq:app_adapted_ollivier}
\end{equation}

\subsubsection{Discrete Ricci Curvature for Diagonal Metrics}

For diagonal metrics $\mathbf{G}_i = \text{diag}(g_{i,1}, \ldots, g_{i,d})$, we derive a simplified discrete Ricci curvature.

\begin{proposition}[Discrete Ricci Curvature for Diagonal Metrics]
\label{prop:app_discrete_ricci_diagonal}
For a diagonal metric $\mathbf{G}_i = \text{diag}(g_{i,1}, \ldots, g_{i,d})$, the discrete Ricci curvature in the $k$-th direction is:
\begin{equation}
\text{Ric}_{kk}^{(i)} = -\frac{1}{2|\mathcal{N}(i)|} \sum_{j \in \mathcal{N}(i)} \frac{g_{j,k} - g_{i,k}}{d_{\text{graph}}(i,j)}
\label{eq:app_discrete_ricci_diagonal}
\end{equation}
\end{proposition}

\begin{proof}
The proof follows from discretizing the Ricci curvature tensor formula:
\begin{equation}
\text{Ric}_{kk} = -\frac{1}{2} \sum_{j} \frac{1}{g_{jj}} \frac{\partial^2 g_{kk}}{\partial x^j \partial x^j}
\label{eq:app_continuous_ricci}
\end{equation}

In the discrete setting, we approximate the second partial derivatives using finite differences across the graph:
\begin{align}
\frac{\partial^2 g_{kk}}{\partial x^j \partial x^j} &\approx \frac{1}{|\mathcal{N}(i)|} \sum_{m \in \mathcal{N}(i)} \frac{g_{m,k} - g_{i,k}}{d_{\text{graph}}(i,m)} \\
\frac{1}{g_{jj}} &\approx \frac{1}{|\mathcal{N}(i)|} \sum_{m \in \mathcal{N}(i)} \frac{1}{g_{m,j}}
\end{align}

Substituting and simplifying under the assumption of locally uniform metric variation yields Equation~\eqref{eq:app_discrete_ricci_diagonal}.
\end{proof}

\begin{definition}[Graph-Geometric Hessian]
The Hessian of the Ricci regularization with respect to metric parameters is:
\begin{equation}
\mathbf{H}_{\text{Ric}} = \nabla^2_{\mathbf{g}} \mathcal{L}_{\text{Ricci}} = \nabla^2_{\mathbf{g}} \sum_{i \in \mathcal{V}} \sum_{k=1}^d \left(\text{Ric}_{kk}^{(i)}\right)^2
\end{equation}
For our discrete approximation:
\begin{equation}
[\mathbf{H}_{\text{Ric}}]_{(i,k),(j,\ell)} = \begin{cases}
\frac{1}{|\mathcal{N}(i)|} \sum_{m \in \mathcal{N}(i)} \frac{1}{d_{\text{graph}}(i,m)^2} & \text{if } i=j, k=\ell \\
-\frac{1}{|\mathcal{N}(i)|d_{\text{graph}}(i,j)^2} & \text{if } j \in \mathcal{N}(i), k=\ell \\
0 & \text{otherwise}
\end{cases}
\end{equation}
\end{definition}

\begin{definition}[Normalized Graph Laplacian]
The normalized graph Laplacian used in smoothness regularization is:
\begin{equation}
\mathbf{L}_G = \mathbf{D}^{-1/2}(\mathbf{D} - \mathbf{A})\mathbf{D}^{-1/2}
\end{equation}
where $\mathbf{D}$ is the degree matrix and $\mathbf{A}$ is the adjacency matrix.
\end{definition}

\subsubsection{Ricci Flow Dynamics}

The Ricci flow equation~\citep{hamilton1982three} evolves a metric according to:
\begin{equation}
\frac{\partial \mathbf{g}_{ij}}{\partial t} = -2\text{Ric}_{ij}
\label{eq:app_ricci_flow}
\end{equation}

For diagonal metrics, this becomes a system of scalar equations:
\begin{equation}
\frac{\partial g_{i,k}}{\partial t} = -2\text{Ric}_{kk}^{(i)}
\label{eq:app_diagonal_ricci_flow}
\end{equation}

\begin{theorem}[Ricci Flow Convergence for Diagonal Metrics]
\label{thm:app_ricci_flow_convergence}
Under the discrete Ricci flow dynamics in Equation~\eqref{eq:app_diagonal_ricci_flow} with appropriate boundary conditions, the diagonal metric components $g_{i,k}$ converge to a steady state that minimizes the total scalar curvature.
\end{theorem}

\begin{proof}
Consider the Lyapunov functional:
\begin{equation}
E[\{g_{i,k}\}] = \sum_{i \in \mathcal{V}} \sum_{k=1}^d \left(\text{Ric}_{kk}^{(i)}\right)^2
\label{eq:app_lyapunov_functional}
\end{equation}

Taking the time derivative along the flow:
\begin{align}
\frac{dE}{dt} &= \sum_{i,k} 2\text{Ric}_{kk}^{(i)} \frac{\partial \text{Ric}_{kk}^{(i)}}{\partial t} \\
&= \sum_{i,k} 2\text{Ric}_{kk}^{(i)} \frac{\partial \text{Ric}_{kk}^{(i)}}{\partial g_{i,k}} \frac{\partial g_{i,k}}{\partial t} \\
&= -4\sum_{i,k} \text{Ric}_{kk}^{(i)} \frac{\partial \text{Ric}_{kk}^{(i)}}{\partial g_{i,k}} \text{Ric}_{kk}^{(i)} \\
&= -4\sum_{i,k} \left(\text{Ric}_{kk}^{(i)}\right)^2 \frac{\partial \text{Ric}_{kk}^{(i)}}{\partial g_{i,k}}
\end{align}

For our discrete approximation, $\frac{\partial \text{Ric}_{kk}^{(i)}}{\partial g_{i,k}} > 0$, ensuring $\frac{dE}{dt} \leq 0$. The flow converges to a critical point where $\text{Ric}_{kk}^{(i)} = 0$ for all $i,k$.
\end{proof}

\subsection{Connection to Optimal Transport}
\label{app:optimal_transport}

Our geometric message passing framework has deep connections to optimal transport theory~\citep{villani2008optimal}, providing theoretical foundations for the learned metrics.

\subsubsection{Wasserstein Distance and Graph Geometry}

The Wasserstein distance between probability measures $\mu$ and $\nu$ on a metric space $(\mathcal{X}, d)$ is:
\begin{equation}
W_1(\mu, \nu) = \inf_{\pi \in \Pi(\mu, \nu)} \int_{\mathcal{X} \times \mathcal{X}} d(x, y) \, d\pi(x, y)
\label{eq:app_wasserstein}
\end{equation}

For our graph setting, we define node-wise measures $\mu_i = \frac{1}{|\mathcal{N}(i)|} \sum_{j \in \mathcal{N}(i)} \delta_{\mathbf{h}_j}$ and compute transport costs using learned metrics:
\begin{equation}
c_{ij} = d_{\mathbf{G}_i}(\mathbf{h}_i, \mathbf{h}_j) = \sqrt{(\mathbf{h}_i - \mathbf{h}_j)^T \mathbf{G}_i (\mathbf{h}_i - \mathbf{h}_j)}
\label{eq:app_transport_cost}
\end{equation}

\subsubsection{Optimal Transport Interpretation of Message Passing}

Our message passing can be interpreted as approximating optimal transport plans:
\begin{equation}
\mathbf{m}_{ij} = \arg\min_{\mathbf{m}} \int c(\mathbf{h}_i, \mathbf{h}_j) \, d\pi(\mathbf{h}_i, \mathbf{h}_j) + \lambda \|\mathbf{m} - \mathbf{W}_m\mathbf{h}_j\|^2
\label{eq:app_ot_message_passing}
\end{equation}

This connection provides theoretical justification for the geometric modulation terms in our framework.

\section{Complete Proofs of Theoretical Results}
\label{app:theory_proofs}

This appendix provides detailed proofs for all theoretical results in Section~\ref{sec:theory}. First, we combine traditional optimization analysis with geometric insights from Riemannian manifold theory to establish convergence guarantees for \modelname. Then we give the proof for our proposed universal approximation framework of \modelname. The extended results, like Lyapunov stability analysis, generalization bounds, robustness, and complexity proof, are also provided in this section.

\subsection{Proof of Convergence of Adaptive Geometry Learning}\label{app:proof_convergence}

\subsubsection{Preliminaries and Assumptions}
\label{app:prelim}

\paragraph{Notation and Geometric Setup}

We work with diagonal metric tensors $\mathbf{G}_i = \text{diag}(\mathbf{g}_i)$ where $\mathbf{g}_i \in \mathbb{R}_{++}^d$. While each $\mathbf{G}_i$ can be viewed as an element of the symmetric positive-definite manifold $\text{SPD}(d)$, our diagonal constraint restricts us to a submanifold:
\begin{equation}
\mathcal{M}_{\text{diag}} = \{\mathbf{G} \in \text{SPD}(d) : \mathbf{G} = \text{diag}(\mathbf{g}), \mathbf{g} \in \mathbb{R}_{++}^d\} \cong \mathbb{R}_{++}^d
\end{equation}

The total parameter space is:
\begin{equation}
\mathcal{M} = \mathcal{M}_{\text{diag}}^{|\mathcal{V}|} \times \mathbb{R}^{|\mathcal{V}| \times d}
\end{equation}
where the first component contains all metric parameters $\{\mathbf{g}_i\}_{i \in \mathcal{V}}$ and the second contains node features $\{\mathbf{h}_i\}_{i \in \mathcal{V}}$.

\paragraph{Regularity Assumptions}

\begin{assumption}[Regularity Conditions]
\label{assump:regularity}
\begin{enumerate}[(A1)]
\item \textbf{Bounded Features}: Node representations satisfy $\|\mathbf{h}_i^{(\ell)}\|_2 \leq B$ for all nodes $i \in \mathcal{V}$, layers $\ell \in [L]$, and some constant $B > 0$.

\item \textbf{Lipschitz Task Loss}: For classification tasks with cross-entropy loss, the gradient is $L_f$-Lipschitz continuous with $L_f = O(1)$ due to bounded softmax outputs.

\item \textbf{Bounded Metrics}: The diagonal metric components satisfy $\epsilon \leq g_{i,k} \leq M$ for constants $0 < \epsilon < M < \infty$.

\item \textbf{Connected Graph}: The graph $\mathcal{G}$ is connected with finite diameter $\text{diam}(\mathcal{G})$ and algebraic connectivity $\lambda_2(\mathbf{L}_G) > 0$.
\end{enumerate}
\end{assumption}

\paragraph{Task Loss Structure}

For node classification with $C$ classes:
\begin{equation}
\label{eq:node_loss}
\mathcal{L}_{\text{node}} = -\sum_{i \in \mathcal{V}_{\text{train}}} \sum_{c=1}^C y_{i,c} \log \hat{y}_{i,c}
\end{equation}

For edge prediction (binary classification):
\begin{equation}
\label{eq:edge_loss}
\mathcal{L}_{\text{edge}} = -\sum_{(i,j) \in \mathcal{E}_{\text{train}}} \left[y_{ij} \log \hat{y}_{ij} + (1-y_{ij}) \log(1-\hat{y}_{ij})\right]
\end{equation}

Both losses have similar smoothness properties with bounded gradients due to the sigmoid/softmax outputs.

\paragraph{Homophily and Geometric Complexity}
\label{app:homophily_analysis}

\begin{definition}[Graph Homophily]
\label{def:homophily}
The homophily ratio of a graph is:
\begin{equation}
\mathcal{H} = \frac{1}{|\mathcal{E}|} \sum_{(i,j) \in \mathcal{E}} \mathbf{1}[y_i = y_j]
\end{equation}
where $y_i$ denotes the label of node $i$.
\end{definition}

\begin{lemma}[Homophily-Dependent Gradient Structure]
\label{lem:gradient_structure}
For classification tasks on graphs with homophily $\mathcal{H}$, the gradient variance across edges scales as:
\begin{equation}
\text{Var}_{(i,j) \in \mathcal{E}}\left[\langle \nabla_{\mathbf{h}_i} \mathcal{L}_{\text{task}}, \nabla_{\mathbf{h}_j} \mathcal{L}_{\text{task}} \rangle\right] = O((1-\mathcal{H})^2)
\end{equation}
\end{lemma}

\begin{proof}
Consider the gradient inner product for neighboring nodes:
\begin{itemize}
\item For homophilic edges where $y_i = y_j$: The gradients align positively as both nodes push predictions toward the same class, giving $\langle \nabla_i, \nabla_j \rangle > 0$.

\item For heterophilic edges where $y_i \neq y_j$: The gradients oppose as nodes push toward different classes, giving $\langle \nabla_i, \nabla_j \rangle < 0$.
\end{itemize}

The variance decomposes as:
\begin{align}
\text{Var}[\langle \nabla_i, \nabla_j \rangle] &= \mathcal{H} \cdot \text{Var}_{\text{homo}}[\langle \nabla_i, \nabla_j \rangle] + (1-\mathcal{H}) \cdot \text{Var}_{\text{hetero}}[\langle \nabla_i, \nabla_j \rangle] \\
&\approx \mathcal{H} \cdot 0 + (1-\mathcal{H}) \cdot O(1) = O(1-\mathcal{H})
\end{align}

The squared term $(1-\mathcal{H})^2$ arises when considering gradient products in the Hessian computation.
\end{proof}

\begin{proposition}[Homophily and Problem Conditioning]
\label{prop:conditioning}
The condition number of the task loss Hessian scales inversely with homophily:
\begin{equation}
\kappa(\nabla^2 \mathcal{L}_{\text{task}}) = \frac{\lambda_{\max}(\nabla^2 \mathcal{L}_{\text{task}})}{\lambda_{\min}(\nabla^2 \mathcal{L}_{\text{task}})} = O\left(\frac{1}{\mathcal{H}}\right)
\end{equation}
\end{proposition}

\begin{proof}
The Hessian of the classification loss can be decomposed into contributions from node self-loops and edge interactions:
\begin{equation}
\nabla^2 \mathcal{L}_{\text{task}} = \sum_{i \in \mathcal{V}} \mathbf{H}_i^{\text{self}} + \sum_{(i,j) \in \mathcal{E}} \mathbf{H}_{ij}^{\text{edge}}
\end{equation}

For homophilic graphs:
\begin{itemize}
\item Self-contributions $\mathbf{H}_i^{\text{self}} \succeq 0$ provide baseline positive curvature
\item Edge contributions $\mathbf{H}_{ij}^{\text{edge}} \succeq 0$ for homophilic edges reinforce this curvature
\item This creates strong positive definiteness: $\lambda_{\min} = \Omega(\mathcal{H})$
\end{itemize}

For heterophilic graphs:
\begin{itemize}
\item Edge contributions $\mathbf{H}_{ij}^{\text{edge}} \preceq 0$ for heterophilic edges oppose self-contributions
\item Near-cancellation leads to $\lambda_{\min} \to 0^+$ as $\mathcal{H} \to 0$
\item Maximum eigenvalue remains $\lambda_{\max} = O(1)$ from self-contributions
\end{itemize}

Therefore, $\kappa = \lambda_{\max}/\lambda_{\min} = O(1/\mathcal{H})$.
\end{proof}

\paragraph{Main Convergence Proof}
\label{app:main_convergence}

\begin{proof}[Proof of Theorem~\ref{thm:convergence}]
We analyze the stochastic gradient descent dynamics for the total loss:
\begin{equation}
\mathcal{L}_{\text{total}} = \mathcal{L}_{\text{task}} + \alpha \mathcal{L}_{\text{Ricci}} + \beta \mathcal{L}_{\text{smooth}}
\end{equation}

\textbf{Step 1: Gradient Bounds.}
\begin{lemma}[Total Gradient Bound]
\label{lem:total_gradient}
Under Assumption~\ref{assump:regularity}, the gradient norm satisfies:
\begin{equation}
\|\nabla_{\mathbf{g}} \mathcal{L}_{\text{total}}\|_2 \leq C_1 L + C_2 \alpha |\mathcal{V}| + C_3 \beta |\mathcal{E}|
\end{equation}
where $C_1, C_2, C_3$ are constants depending on $B, L_f, M$.
\end{lemma}

\begin{proof}[Proof of Lemma~\ref{lem:total_gradient}]
The gradient decomposes as:
\begin{equation}
\nabla_{\mathbf{g}} \mathcal{L}_{\text{total}} = \nabla_{\mathbf{g}} \mathcal{L}_{\text{task}} + \alpha \nabla_{\mathbf{g}} \mathcal{L}_{\text{Ricci}} + \beta \nabla_{\mathbf{g}} \mathcal{L}_{\text{smooth}}
\end{equation}

(i) \textbf{Task gradient}: Through $L$ layers of backpropagation:
\begin{equation}
\left\|\frac{\partial \mathcal{L}_{\text{task}}}{\partial g_{i,k}}\right\| \leq L_f \cdot B \cdot L \cdot \rho^{L-1}
\end{equation}
where $\rho < 1$ is the contraction factor from activation functions.

(ii) \textbf{Ricci gradient}: From Equation~\eqref{eq:discrete_ricci}:
\begin{equation}
\frac{\partial \mathcal{L}_{\text{Ricci}}}{\partial g_{i,k}} = 2\text{Ric}_{kk}^{(i)} \cdot \left(-\frac{1}{2|\mathcal{N}(i)|}\right) = -\frac{\text{Ric}_{kk}^{(i)}}{|\mathcal{N}(i)|}
\end{equation}

(iii) \textbf{Smoothness gradient}:
\begin{equation}
\frac{\partial \mathcal{L}_{\text{smooth}}}{\partial g_{i,k}} = 2\sum_{j \in \mathcal{N}(i)} (g_{i,k} - g_{j,k})
\end{equation}

Combining with boundedness assumptions gives the stated bound.
\end{proof}

\textbf{Step 2: Effective Curvature with Homophily.}

The key insight is that regularization creates an effective curvature that depends on homophily:

\begin{lemma}[Effective Curvature]
\label{lem:effective_curvature}
The regularized loss has effective curvature:
\begin{equation}
\mu_{\text{eff}} = \alpha \mu_{\text{Ricci}} + \beta \mu_{\text{smooth}} + \mathcal{H} \mu_{\text{task}}
\end{equation}
where:
\begin{align}
\mu_{\text{Ricci}} &= \frac{1}{|\mathcal{N}_{\max}| \cdot \text{diam}(\mathcal{G})^2} \\
\mu_{\text{smooth}} &= \lambda_2(\mathbf{L}_G) \\
\mu_{\text{task}} &= \Omega(\mathcal{H})
\end{align}
\end{lemma}

\textbf{Step 3: Optimal Hyperparameter Derivation.}

To maximize convergence rate while maintaining stability, we solve:
\begin{equation}
\max_{\alpha, \beta} \frac{\mu_{\text{eff}}}{\sqrt{\kappa_{\text{total}}}} \quad \text{subject to} \quad \alpha L_{\text{Ricci}} + \beta L_{\text{smooth}} \leq \frac{2}{\eta_{\max} L}
\end{equation}

where $\kappa_{\text{total}}$ is the total condition number.

\begingroup
\setcounter{proposition}{0}
\begin{proposition}[Restate: Homophily-Aware Optimization]
\label{lem:homophily_opt}
The optimal regularization weights that balance conditioning and convergence are:
\begin{align}
\alpha^* &= \frac{c_1 \mathcal{H}}{L} \cdot \min\left(1, \frac{d}{|\mathcal{E}|}\right) \\
\beta^* &= \frac{c_2 \mathcal{H} \sqrt{d}}{|\mathcal{V}|}
\end{align}
where $c_1, c_2 > 0$ are dataset-dependent constants.
\end{proposition}
\endgroup

\begin{proof}[Proof of Proposition~\ref{lem:homophily_opt}]
The optimization problem has first-order conditions:
\begin{align}
\frac{\partial}{\partial \alpha}\left[\frac{\mu_{\text{eff}}}{\sqrt{\kappa_{\text{total}}}}\right] &= 0 \\
\frac{\partial}{\partial \beta}\left[\frac{\mu_{\text{eff}}}{\sqrt{\kappa_{\text{total}}}}\right] &= 0
\end{align}

Key observations:
\begin{enumerate}
\item \textbf{Well-conditioned problems} (high $\mathcal{H}$) can tolerate stronger regularization without hurting convergence
\item \textbf{Ill-conditioned problems} (low $\mathcal{H}$) require careful balance to maintain expressiveness
\item The factor $1/L$ accounts for gradient signal decay through layers
\item The $\min(1, d/|\mathcal{E}|)$ prevents over-regularization on sparse graphs
\item The $\sqrt{d}$ scaling maintains proper high-dimensional normalization
\end{enumerate}

Solving the constrained optimization with Lagrange multipliers yields the stated result.
\end{proof}

\textbf{Step 4: Convergence Rate.}

Combining the effective curvature with standard SGD analysis:

\begin{equation}
\mathbb{E}[\mathcal{L}_{\text{total}}^{(t+1)}] \leq \mathcal{L}_{\text{total}}^{(t)} - \eta_t \mu_{\text{eff}} (\mathcal{L}_{\text{total}}^{(t)} - \mathcal{L}_{\text{total}}^*) + \frac{L_{\text{total}} \eta_t^2}{2} \sigma^2
\end{equation}

where $\sigma^2$ is the gradient variance bound.

With learning rate $\eta_t = \eta_0/\sqrt{t}$ and optimal hyperparameters:
\begin{equation}
\mu_{\text{eff}}^* = \Theta\left(\frac{\mathcal{H}^2}{L}\right)
\end{equation}

This gives the convergence rate:
\begin{equation}
\mathbb{E}[\|\nabla \mathcal{L}_{\text{total}}^{(t)}\|^2] = O\left(\frac{1}{\sqrt{t}} \cdot \exp\left(-\frac{\mu_{\text{eff}}^* t}{L}\right)\right)
\end{equation}

The additional factor of $\mathcal{H}$ in $\mu_{\text{eff}}^*$ arises from the product $\alpha^* \cdot \mu_{\text{task}}$ where both terms scale with $\mathcal{H}$.
\end{proof}

\paragraph{Analysis of Homophily-Aware Constants}
\label{app:constants_proof}

\begin{proof}[Analysis of Proposition~1]%\ref{prop:hyperparams_select}
The constants $c_1, c_2$ capture the interplay between graph structure and optimization dynamics.

\textbf{Derivation of $c_1$:}

The Ricci regularization weight must balance two competing factors:
\begin{enumerate}
\item \textbf{Geometric flexibility}: Heterophilic graphs need complex geometries
\item \textbf{Optimization stability}: All graphs benefit from some regularization
\end{enumerate}

Through extensive empirical validation across diverse graphs, we find:
\begin{equation}
c_1 = (1 - \mathcal{H}) + 0.1
\end{equation}

This ensures:
\begin{itemize}
\item Heterophilic graphs ($\mathcal{H} \to 0$): $c_1 \to 1.1$ allows complex geometry
\item Homophilic graphs ($\mathcal{H} \to 1$): $c_1 \to 0.1$ maintains minimal regularization
\item The baseline $0.1$ prevents numerical instability
\end{itemize}

Combined with $\alpha^* \propto \mathcal{H}/L$, the effective Ricci weight becomes:
\begin{equation}
\alpha^* = \frac{[(1-\mathcal{H}) + 0.1] \cdot \mathcal{H}}{L} \cdot \min\left(1, \frac{d}{|\mathcal{E}|}\right)
\end{equation}

This creates a balanced scaling that peaks at moderate homophily ($\mathcal{H} \approx 0.5$).

\textbf{Derivation of $c_2$:}

The smoothness regularization should be stronger for homophilic graphs:
\begin{equation}
c_2 = 0.1 \cdot (1 + \mathcal{H})
\end{equation}

This gives:
\begin{itemize}
\item Heterophilic: $c_2 \to 0.1$ (minimal smoothness constraint)
\item Homophilic: $c_2 \to 0.2$ (stronger smoothness for similar neighbors)
\end{itemize}

The factor $0.1$ is calibrated to typical gradient magnitudes in neural networks.
\end{proof}

\subsection{Proof of Universal Geometric Framework (Theorem~\ref{thm:universality})}
\label{app:universality_proof}

\begin{proof}
We prove that \modelname generalizes and approximates fixed-curvature GNNs by showing they correspond to constrained metric configurations.

\textbf{Step 1: Metric Space Hierarchy}

Define the space of all positive diagonal metrics:
\begin{equation}
\mathcal{M}_{\text{full}} = \left\{\mathbf{g} \in \mathbb{R}^d : g_k > 0 \, \forall k \in [d]\right\} = \mathbb{R}_{++}^d
\end{equation}

Fixed-curvature GNNs operate in constrained subspaces:

\begin{enumerate}
\item \textbf{Euclidean GNNs}:
\begin{equation}
\mathcal{M}_{\text{Euclidean}} = \{\mathbf{1}\} \subset \mathcal{M}_{\text{full}}, \quad \dim = 0
\end{equation}

\item \textbf{Hyperbolic GNNs} (constant negative curvature):
\begin{equation}
\mathcal{M}_{\text{Hyperbolic}} = \{c_h \mathbf{1} : 0 < c_h < 1\} \subset \mathcal{M}_{\text{full}}, \quad \dim = 1
\end{equation}

\item \textbf{Spherical GNNs} (constant positive curvature):
\begin{equation}
\mathcal{M}_{\text{Spherical}} = \{c_s \mathbf{1} : c_s > 1\} \subset \mathcal{M}_{\text{full}}, \quad \dim = 1
\end{equation}

\item \textbf{Product Manifolds} (e.g., $\mathbb{H}^{d_1} \times \mathbb{S}^{d_2} \times \mathbb{E}^{d_3}$):
This induces a block-constant diagonal metric:
\begin{equation}
\mathbf{g}_i = (\underbrace{c_{\mathbb{H}}, \ldots, c_{\mathbb{H}}}_{d_1}, \underbrace{c_{\mathbb{S}}, \ldots, c_{\mathbb{S}}}_{d_2}, \underbrace{1, \ldots, 1}_{d_3})^T
\end{equation}
\begin{equation}
\mathcal{M}_{\text{Product}} = \left\{\mathbf{g} = (c_h \mathbf{1}_{d_1}, c_s \mathbf{1}_{d_2}, \mathbf{1}_{d_3})^T\right\} \subset \mathcal{M}_{\text{full}}, \quad \dim = 2
\end{equation}
where $\mathbf{1}_{d_i}$ denotes a vector of ones of dimension $d_i$. Dive into the general case,
\begin{equation}\label{eq:Mprod}
  \mathcal M_{\text{Product}}
  =\bigl\{
         (c_1\mathbf 1_{d_1},\dots,c_p\mathbf 1_{d_p})^{\!\top}
       : c_\ell>0,\,
         \sum_{\ell=1}^{p}d_\ell=d
     \bigr\},
\end{equation}
$\dim=p$ with $p\!\le\!d$.
\end{enumerate}

\textbf{Step 2: Embedding Fixed Geometries}

We show each fixed-curvature GNN has an equivalent \modelname configuration:

\begin{lemma}[Geometric Equivalence]
\label{lem:geometric_equiv}
For any GNN operating on a Riemannian manifold with metric tensor $\mathbf{G}_{\text{fixed}}$, there exists a diagonal metric $\mathbf{G}_{\text{diag}} = \text{diag}(\mathbf{g})$ such that the induced geodesic distances are proportional:
\begin{equation}
d_{\mathbf{G}_{\text{diag}}}(\mathbf{x}, \mathbf{y}) = \kappa \cdot d_{\mathbf{G}_{\text{fixed}}}(\mathbf{x}, \mathbf{y})
\end{equation}
for some constant $\kappa > 0$.
\end{lemma}

\begin{proof}[Proof of Lemma]
For diagonal metrics in local coordinates:
- Hyperbolic space $\mathbb{H}^d$: Use Poincaré ball coordinates with $\mathbf{g} = \frac{4}{(1-\|\mathbf{x}\|^2)^2}\mathbf{1}$
- Sphere $\mathbb{S}^d$: Use stereographic projection with $\mathbf{g} = \frac{4}{(1+\|\mathbf{x}\|^2)^2}\mathbf{1}$
- Euclidean $\mathbb{E}^d$: Trivially $\mathbf{g} = \mathbf{1}$

In our framework, we approximate these locally with constant diagonal metrics.
\end{proof}

\textbf{Step 3: Strict Generalization}

The constraint hierarchy is:
\begin{equation}
\{\text{point}\} \subset \{\text{line}\} \subset \{\text{plane}\} \subset \cdots \subset \mathbb{R}_{++}^d
\end{equation}

Specifically:
\begin{align}
\dim(\mathcal{M}_{\text{Euclidean}}) &= 0 < \dim(\mathcal{M}_{\text{Hyperbolic/Spherical}}) = 1 \\
&< \dim(\mathcal{M}_{\text{Product}}) \leq d-1 < \dim(\mathcal{M}_{\text{full}}) = d
\end{align}

\textbf{Step 4: Universal Approximation Property.}
Since our metric-learning network $f_\theta^{(g)}: \mathbb{R}^d \rightarrow \mathbb{R}_{++}^d$ is parameterized by a neural network with sufficient capacity, by the universal approximation theorem, it can approximate any continuous function mapping node features to metric parameters. For any target geometric configuration from existing methods with metric $\mathbf{G}^{\text{target}}_i$, we can find parameters $\theta$ such that:
\begin{equation}
\|\text{diag}(f_\theta^{(g)}(\mathbf{h}_i)) - \mathbf{G}^{\text{target}}_i\|_F < \epsilon
\end{equation}
for any $\epsilon > 0$ and all nodes $i \in \mathcal{V}$. This completes the proof.
\end{proof}

\subsection{Computational Complexity Analysis}
\label{app:proof_complexity}

\begin{proposition}[Restate: Computational and Memory Complexity]
\label{thm:complexity_restated_appendix}
For a graph with $n = |\mathcal{V}|$ nodes, $m = |\mathcal{E}|$ edges, hidden dimension $d$, and $L$ layers, \modelname has:
\begin{enumerate}
\item \textbf{Time Complexity}: $O(L \cdot (m + n) \cdot d^2)$ per forward pass
\item \textbf{Space Complexity}: $O(n \cdot d + m)$ for storing features and graph structure
\item \textbf{Parameter Count}: $O(L \cdot d^2)$ for shared weights plus $O(n \cdot d)$ for learnable metrics
\end{enumerate}
\end{proposition}

\begin{proof}
We provide a detailed complexity analysis for each component of \modelname.

\textbf{Per-Layer Time Complexity:}

\begin{enumerate}[(i)]
\item \textbf{Metric Learning Network $f_\theta^{(g)}$:}
\begin{itemize}
\item Input: concatenated features $[\mathbf{h}_i; \mathbf{a}_i] \in \mathbb{R}^{2d}$ for each node
\item Network architecture: $2d \to h \to d$ where $h$ is hidden dimension
\item Operations per node: $O(2d \cdot h + h \cdot d) = O(d^2)$ assuming $h = O(d)$
\item Total for all nodes: $O(n \cdot d^2)$
\end{itemize}

\item \textbf{Geometric Modulation Computation:}
\begin{itemize}
\item Direction vector: $\mathbf{d}_{ij} = (\mathbf{h}_j - \mathbf{h}_i)/\|\mathbf{h}_j - \mathbf{h}_i\|_2$ costs $O(d)$ per edge
\item Modulation: $\tau_{ij} = \sum_{k=1}^d d_{ij,k}^2 \cdot \tanh(-\log g_{i,k})$ costs $O(d)$ per edge
\item Total for all edges: $O(m \cdot d)$
\end{itemize}

\item \textbf{Geometric Attention:}
\begin{itemize}
\item Inner product: $\sum_k g_{i,k} h_{i,k} h_{j,k}$ costs $O(d)$ per edge
\item Normalization requires pre-computed node norms: $O(n \cdot d)$ once
\item Total: $O(m \cdot d + n \cdot d) = O((m + n) \cdot d)$
\end{itemize}

\item \textbf{Message Passing:}
\begin{itemize}
\item Message transformation: $\mathbf{W}_m \mathbf{h}_j$ for each edge costs $O(m \cdot d^2)$
\item Self transformation: $\mathbf{W}_s \mathbf{h}_i$ for each node costs $O(n \cdot d^2)$
\item Aggregation: summing messages costs $O(m \cdot d)$
\item Total: $O((m + n) \cdot d^2)$
\end{itemize}
\end{enumerate}

\textbf{Total Per-Layer Complexity:}
\begin{equation}
O(n \cdot d^2) + O(m \cdot d) + O((m + n) \cdot d) + O((m + n) \cdot d^2) = O((m + n) \cdot d^2)
\end{equation}

Since $m \geq n - 1$ for connected graphs, this simplifies to $O(m \cdot d^2)$ in typical notation.

\textbf{Full Model Complexity:}
\begin{itemize}
\item Time: $O(L \cdot m \cdot d^2)$ for $L$ layers
\item Space: $O(n \cdot d)$ for features + $O(m)$ for edges + $O(n \cdot d)$ for metrics
\item Parameters: $O(L \cdot d^2)$ for weight matrices + $O(d^2)$ for metric network + $O(n \cdot d)$ for metric values
\end{itemize}

\textbf{Comparison with Traditional GNN Layers:}

% This float will NOT change the Table counter
\begin{table}[ht]                                % keep the float behaviour
  \centering
  \caption*{\textbf{Comparison with Traditional GNN Layers}} % ← no counter step
  \begin{tabular}{l|c|c|c}
    \toprule
    \textbf{Method} & \textbf{Time/Layer} & \textbf{Space} & \textbf{Parameters} \\
    \midrule
    GCN   & $O(m d + n d^{2})$        & $O(n d + m)$      & $O(L d^{2})$             \\
    GAT   & $O(m d^{2} + n d^{2})$    & $O(n d + m)$      & $O(L d^{2})$             \\
    ARGNN & $O(m d^{2} + n d^{2})$    & $O(n d + m)$      & $O(L d^{2} + n d)$       \\
    \bottomrule
  \end{tabular}
\end{table}

\modelname adds only $O(n \cdot d)$ learnable metric parameters while maintaining the same asymptotic time complexity as GAT. The geometric computations ($\tau_{ij}, \alpha_{ij}$) are $O(d)$ per edge, which is absorbed by the dominant $O(d^2)$ transformation costs.
\end{proof}

\begin{remark}[Practical Optimizations]
In practice, several optimizations reduce the computational burden:
\begin{enumerate}
\item \textbf{Sparse Operations}: For sparse graphs with average degree $\bar{d} \ll n$, the effective complexity is $O(L \cdot n \cdot \bar{d} \cdot d^2)$
\item \textbf{Mini-batching}: Node-wise operations can be parallelized efficiently on GPUs
\item \textbf{Metric Sharing}: For very large graphs, metrics can be shared across similar nodes, reducing parameters from $O(n \cdot d)$ to $O(K \cdot d)$ where $K \ll n$ is the number of clusters
\end{enumerate}
\end{remark}

\subsection{Additional Results}
\label{app:additional}

\begin{corollary}[Expressiveness on Mixed-Topology Graphs]
\label{cor:mixed_topology}
For graphs containing both tree-like and clique-like substructures, \modelname achieves strictly lower loss than any fixed-curvature GNN:
\begin{equation}
\inf_{\mathbf{g}_i \in \mathbb{R}_{++}^d} \mathcal{L}_{\text{ARGNN}} < \inf_{c > 0} \mathcal{L}_{\text{fixed}}(c)
\end{equation}
\end{corollary}

\begin{proof}
Consider a graph $\mathcal{G} = \mathcal{G}_{\text{tree}} \cup \mathcal{G}_{\text{clique}}$. Optimal embeddings require:
\begin{itemize}
\item $\mathcal{G}_{\text{tree}}$: Small metric values ($g_{i,k} < 1$) for exponential volume growth
\item $\mathcal{G}_{\text{clique}}$: Large metric values ($g_{i,k} > 1$) for bounded geometry
\end{itemize}

No single constant $c$ can optimize both simultaneously, while \modelname can assign different metrics to different regions.
\end{proof}

\setcounter{proposition}{4}
\begin{proposition}[Stability Under Regularization]
\label{prop:stability}
The learned metrics evolve smoothly during training:
\begin{equation}
\|\mathbf{g}_i^{(t+1)} - \mathbf{g}_i^{(t)}\|_2 \leq \frac{C}{\alpha + \beta} \cdot \eta_t
\end{equation}
\end{proposition}

\begin{proof}
The regularization terms create a locally strongly convex penalty:
\begin{equation}
\lambda_{\min}(\nabla^2[\alpha \mathcal{L}_{\text{Ricci}} + \beta \mathcal{L}_{\text{smooth}}]) \geq \min\{\alpha \mu_{\text{Ricci}}, \beta \mu_{\text{smooth}}\} > 0
\end{equation}

This bounds the per-iteration metric updates, ensuring smooth evolution.
\end{proof}

\subsubsection{Lyapunov Stability Analysis}
\label{app:lyapunov}

We now establish a stronger stability guarantee using Lyapunov theory, complementing our convergence analysis.

\begin{theorem}[Lyapunov Stability of Geometry Evolution]
\label{thm:lyapunov}
Under Assumptions (A1)-(A4) and appropriate step size $\eta_t \leq \frac{\rho_{\min}}{4L_{\text{total}}}$ where $\rho_{\min} = \min\{\alpha\lambda_{\min}(\mathbf{H}_{\text{Ricci}}), \beta\lambda_2(\mathbf{L}_G)\}$, there exists a Lyapunov function:
\begin{equation}
\mathcal{V}_t = \mathcal{L}_{\text{total}}^{(t)} - \mathcal{L}_{\text{total}}^* + \gamma \sum_{i \in \mathcal{V}} \|\mathbf{g}_i^{(t)} - \mathbf{g}_i^*\|_2^2
\end{equation}
with $0 < \gamma \leq \frac{1}{4L_{\text{total}}}$, such that:
\begin{equation}
\mathbb{E}[\mathcal{V}_{t+1}] \leq (1 - \rho_{\min}\eta_t) \mathbb{E}[\mathcal{V}_t]
\end{equation}
\end{theorem}

\begin{proof}
The proof proceeds by analyzing the descent properties of both loss and metric components.

\textbf{Step 1: Loss Descent.}
From standard smoothness arguments:
\begin{equation}
\mathcal{L}_{\text{total}}^{(t+1)} \leq \mathcal{L}_{\text{total}}^{(t)} - \eta_t \|\nabla \mathcal{L}_{\text{total}}^{(t)}\|^2 + \frac{L_{\text{total}}\eta_t^2}{2}\|\mathbf{v}^{(t)}\|^2
\end{equation}
where $\mathbf{v}^{(t)}$ is the stochastic gradient.

\textbf{Step 2: Metric Stability.}
The regularization terms induce strong convexity in the metric space:
\begin{equation}
\langle \nabla^2[\alpha\mathcal{L}_{\text{Ricci}} + \beta\mathcal{L}_{\text{smooth}}](\mathbf{g} - \mathbf{g}^*), \mathbf{g} - \mathbf{g}^* \rangle \geq \rho_{\min}\|\mathbf{g} - \mathbf{g}^*\|^2
\end{equation}

This gives:
\begin{equation}
\|\mathbf{g}^{(t+1)} - \mathbf{g}^*\|^2 \leq (1 - \eta_t\rho_{\min})\|\mathbf{g}^{(t)} - \mathbf{g}^*\|^2 + O(\eta_t^2)
\end{equation}

\textbf{Step 3: Lyapunov Descent.}
Combining both components with appropriate weighting $\gamma$:
\begin{align}
\mathcal{V}_{t+1} - \mathcal{V}_t &\leq -\eta_t\|\nabla \mathcal{L}_{\text{total}}^{(t)}\|^2 + \frac{L_{\text{total}}\eta_t^2}{2}\|\mathbf{v}^{(t)}\|^2 \\
&\quad - \gamma\eta_t\rho_{\min}\sum_i\|\mathbf{g}_i^{(t)} - \mathbf{g}_i^*\|^2 + O(\gamma\eta_t^2)
\end{align}

Using the fact that at optimality $\nabla \mathcal{L}_{\text{total}}^* = 0$ and choosing $\eta_t$ appropriately:
\begin{equation}
\mathbb{E}[\mathcal{V}_{t+1}] \leq (1 - \rho_{\min}\eta_t)\mathcal{V}_t
\end{equation}

Iterating gives exponential decay: $\mathbb{E}[\mathcal{V}_t] \leq \mathcal{V}_0 \exp(-\rho_{\min}\sum_{s=0}^{t-1}\eta_s)$.
\end{proof}

\subsubsection{Generalization Bounds}
\label{app:generalization}

\begin{theorem}[Generalization of Learned Geometries]
\label{thm:generalization}
Let $\hat{\mathbf{g}}_i$ be the learned metric parameters on a training set of size $n_{\text{train}}$. With probability at least $1-\delta$, the generalization gap satisfies:
\begin{equation}
\mathbb{E}_{\mathcal{D}}[\mathcal{L}_{\text{test}}] - \mathcal{L}_{\text{train}} \leq O\left(\sqrt{\frac{d\log(n/\delta)}{n_{\text{train}}}} + \alpha\sqrt{|\mathcal{V}|} + \beta\sqrt{|\mathcal{E}|}\right)
\end{equation}
\end{theorem}

\begin{proof}
We use Rademacher complexity analysis adapted to our geometric setting.

\textbf{Step 1: Function Class Complexity.}
The hypothesis class of \modelname with bounded metrics is:
\begin{equation}
\mathcal{F} = \{f_{\mathbf{g}} : \epsilon \leq g_{i,k} \leq M, \forall i,k\}
\end{equation}

The Rademacher complexity of this class is:
\begin{equation}
\mathcal{R}_n(\mathcal{F}) \leq \frac{2M}{\epsilon} \cdot \sqrt{\frac{d}{n_{\text{train}}}}
\end{equation}

\textbf{Step 2: Regularization Effect.}
The regularization terms effectively reduce the hypothesis space complexity:
\begin{itemize}
\item Ricci regularization constrains metric variation: contributes $O(\alpha\sqrt{|\mathcal{V}|})$
\item Smoothness regularization constrains neighbor differences: contributes $O(\beta\sqrt{|\mathcal{E}|})$
\end{itemize}

\textbf{Step 3: Final Bound.}
Applying standard generalization theory with our specific constraints gives the stated bound.
\end{proof}

\subsubsection{Analysis of Geometric Message Passing}
\label{app:message_analysis}

We analyze the behavior of our geometric message passing under different metric configurations.

\begin{proposition}[Message Modulation Properties]
\label{prop:message_modulation}
For the geometric modulation coefficient $\tau_{ij} = \sum_{k=1}^d d_{ij,k}^2 \cdot \tanh(-\log g_{i,k})$:
\begin{enumerate}
\item For isotropic metrics $\mathbf{g}_i = c\mathbf{1}$: $\tau_{ij} = \tanh(-\log c)$ (direction-independent)
\item $\tau_{ij} \in [-1, 1]$ with $\tau_{ij} = 0$ when $g_{i,k} = 1$ for all $k$
\item Anisotropic metrics enable direction-dependent message weighting
\end{enumerate}
\end{proposition}

\begin{proof}
For isotropic case with $g_{i,k} = c$ for all $k$:
\begin{equation}
\tau_{ij} = \tanh(-\log c) \sum_{k=1}^d d_{ij,k}^2 = \tanh(-\log c)
\end{equation}
since $\sum_k d_{ij,k}^2 = 1$ by normalization.

The bounds follow from $\tanh(\cdot) \in (-1, 1)$ and $\tanh(0) = 0$.

For anisotropic metrics, different $g_{i,k}$ values weight directional components differently, enabling richer message modulation.
\end{proof}

\begin{remark}[Euclidean Singularity]
Our formulation has $\tau_{ij} = 0$ for Euclidean metrics ($c = 1$), which can block information flow. In practice, we can add a small baseline: $\tau_{ij}^{\text{modified}} = \epsilon_0 + (1-\epsilon_0)\tau_{ij}$ with $\epsilon_0 > 0$.
\end{remark}

\subsubsection{Robustness Analysis}
\label{app:robustness}

\begin{theorem}[Robustness to Graph Perturbations]
\label{thm:robustness}
Let $\mathcal{G}'$ be a perturbed version of $\mathcal{G}$ with at most $k$ edge additions/deletions. The learned metrics satisfy:
\begin{equation}
\|\mathbf{g}_i' - \mathbf{g}_i\|_2 \leq O\left(\frac{k}{\beta|\mathcal{N}(i)|}\right)
\end{equation}
where $\mathbf{g}_i'$ are metrics learned on $\mathcal{G}'$.
\end{theorem}

\begin{proof}
The smoothness regularization creates a coupling between neighboring metrics. Edge perturbations affect the regularization gradient:
\begin{equation}
\Delta\left(\frac{\partial \mathcal{L}_{\text{smooth}}}{\partial g_{i,k}}\right) = 2\sum_{j \in \Delta\mathcal{N}(i)} (g_{i,k} - g_{j,k})
\end{equation}

where $\Delta\mathcal{N}(i)$ represents changed neighbors. The strong convexity from regularization bounds the metric change, giving the stated robustness guarantee.
\end{proof}

%=================================================================
% APPENDIX C: ALGORITHM DETAILS
%=================================================================
\section{Algorithm Details}
\label{app:algorithm_details}

This section provides comprehensive implementation details for \modelname, including forward and backward pass algorithms, optimization procedures.

\subsection{Complete Forward Pass Algorithm}
\label{app:forward_pass}

Algorithm~\ref{alg:argnn_complete} presents the complete forward pass with all implementation details in pseudo code.

\begin{algorithm}[ht]
\caption{\modelname Forward Pass}
\label{alg:argnn_complete}
\begin{algorithmic}[1]
\REQUIRE Graph $\mathcal{G} = (\mathcal{V}, \mathcal{E}, \mathbf{X})$, layers $L$, hidden dimensions $\{d_\ell\}_{\ell=1}^L$
\ENSURE Final representations $\mathbf{H}^{(L)}$
\STATE Initialize $\mathbf{H}^{(0)} = \mathbf{X}$
\STATE Initialize metric networks $\{f_\theta^{(g,\ell)}\}_{\ell=1}^L$
\STATE Initialize transformation matrices $\{\mathbf{W}_s^{(\ell)}, \mathbf{W}_m^{(\ell)}\}_{\ell=1}^L$
\STATE Initialize regularization parameters $\alpha, \beta, \epsilon$

\FOR{$\ell = 1$ to $L$}
    \STATE \textbf{// Step 1: Diagonal Metric Tensor Learning}
    \FORALL{nodes $i \in \mathcal{V}$ in parallel}
        \STATE $\mathbf{a}_i^{(\ell-1)} = \frac{1}{|\mathcal{N}(i)|} \sum_{j \in \mathcal{N}(i)} \mathbf{h}_j^{(\ell-1)}$ \COMMENT{Aggregate neighbors}
        \STATE $\mathbf{z}_i^{(\ell)} = [\mathbf{h}_i^{(\ell-1)}; \mathbf{a}_i^{(\ell-1)}]$ \COMMENT{Concatenate features}
        \STATE $\mathbf{g}_i^{(\ell)} = \text{softplus}(f_\theta^{(g,\ell)}(\mathbf{z}_i^{(\ell)})) + \epsilon$ \COMMENT{Learn diagonal metric vector}
    \ENDFOR
    
    \STATE \textbf{// Step 2: Geometric Message Passing}
    \STATE $\mathbf{M}^{(\ell)} = \{\}$ \COMMENT{Initialize message collection for aggregation}
    \FORALL{edges $(i,j) \in \mathcal{E}$ in parallel}
        \STATE $\mathbf{u}_{ij} = \mathbf{h}_j^{(\ell-1)} - \mathbf{h}_i^{(\ell-1)}$ \COMMENT{Direction vector}
        \STATE $\mathbf{d}_{ij} = \mathbf{u}_{ij} / (\|\mathbf{u}_{ij}\|_2 + \epsilon)$ \COMMENT{Normalized direction}
        
        \STATE \textbf{// Geometric modulation coefficient}
        \STATE $\tau_{ij} = \sum_{k=1}^{d_{\ell-1}} (d_{ij,k})^2 \cdot \tanh(-\log g_{i,k}^{(\ell)})$
        
        \STATE \textbf{// Geometric attention coefficient}
        \STATE $\text{num}_{ij} = \sum_{k=1}^{d_{\ell-1}} g_{i,k}^{(\ell)} \cdot h_{i,k}^{(\ell-1)} \cdot h_{j,k}^{(\ell-1)}$
        \STATE $\text{den}_{i} = \sqrt{\sum_{k=1}^{d_{\ell-1}} g_{i,k}^{(\ell)} \cdot (h_{i,k}^{(\ell-1)})^2}$
        \STATE $\text{den}_{j} = \sqrt{\sum_{k=1}^{d_{\ell-1}} g_{j,k}^{(\ell)} \cdot (h_{j,k}^{(\ell-1)})^2}$
        \STATE $\alpha_{ij} = \text{num}_{ij} / (\text{den}_{i} \cdot \text{den}_{j} + \epsilon)$
        
        \STATE \textbf{// Message computation}
        \STATE $\mathbf{m}_{ij} = \tau_{ij} \cdot \sigma(\alpha_{ij}) \cdot \mathbf{W}_m^{(\ell)} \mathbf{h}_j^{(\ell-1)}$
        \STATE Store $\mathbf{m}_{ij}$ for aggregation at node $i$.
    \ENDFOR
    
    \STATE \textbf{// Step 3: Node Representation Update}
    \FORALL{nodes $i \in \mathcal{V}$ in parallel}
        \STATE $\mathbf{h}_i^{(\ell)} = \sigma\left(\mathbf{W}_s^{(\ell)} \mathbf{h}_i^{(\ell-1)} + \sum_{j \in \mathcal{N}(i)} \mathbf{m}_{ij}\right)$
        \STATE Apply dropout if necessary.
    \ENDFOR
\ENDFOR

\RETURN $\mathbf{H}^{(L)}$
\end{algorithmic}
\end{algorithm}

\subsection{Regularization Loss Computation}
\label{app:regularization_computation}

Algorithm~\ref{alg:regularization} details the computation of regularization terms in pseudo code.

\begin{algorithm}[ht]
\caption{Regularization Loss Computation}
\label{alg:regularization}
\begin{algorithmic}[1]
\REQUIRE Metric tensors $\{\mathbf{g}_i\}_{i \in \mathcal{V}}$, graph $\mathcal{G} = (\mathcal{V}, \mathcal{E})$
\ENSURE Ricci loss $\mathcal{L}_{\text{Ricci}}$, smoothness loss $\mathcal{L}_{\text{smooth}}$

\STATE \textbf{// Compute Discrete Ricci Curvature}
\STATE $\mathcal{L}_{\text{Ricci}} = 0$
\FORALL{nodes $i \in \mathcal{V}$}
    \FORALL{dimensions $k = 1$ to $d$}
        \STATE $\text{Ric}_{kk}^{(i)} = 0$
        \FORALL{neighbors $j \in \mathcal{N}(i)$}
            \STATE $\text{Ric}_{kk}^{(i)} += \frac{g_{j,k} - g_{i,k}}{|\mathcal{N}(i)| \cdot d_{\text{graph}}(i,j)}$
        \ENDFOR
        \STATE $\text{Ric}_{kk}^{(i)} = -\frac{1}{2} \cdot \text{Ric}_{kk}^{(i)}$
        \STATE $\mathcal{L}_{\text{Ricci}} += (\text{Ric}_{kk}^{(i)})^2$
    \ENDFOR
\ENDFOR

\STATE \textbf{// Compute Smoothness Loss}
\STATE $\mathcal{L}_{\text{smooth}} = 0$
\FORALL{edges $(i,j) \in \mathcal{E}$}
    \STATE $\mathcal{L}_{\text{smooth}} += \|\mathbf{g}_i - \mathbf{g}_j\|_2^2$
\ENDFOR

\RETURN $\mathcal{L}_{\text{Ricci}}, \mathcal{L}_{\text{smooth}}$
\end{algorithmic}
\end{algorithm}

\subsection{Backward Pass and Optimization}
\label{app:backward_pass}

The backward pass computes gradients for all learnable parameters: metric network parameters $\theta$, transformation matrices $\{\mathbf{W}_s^{(\ell)}, \mathbf{W}_m^{(\ell)}\}$, and regularization weights $\{\alpha, \beta\}$.

\subsubsection{Gradient Computation for Metric Tensors}

The gradient of the total loss with respect to diagonal metric elements is:
\begin{align}
\frac{\partial \mathcal{L}_{\text{total}}}{\partial g_{i,k}} &= \frac{\partial \mathcal{L}_{\text{task}}}{\partial g_{i,k}} + \alpha \frac{\partial \mathcal{L}_{\text{Ricci}}}{\partial g_{i,k}} + \beta \frac{\partial \mathcal{L}_{\text{smooth}}}{\partial g_{i,k}}
\end{align}

For the task loss gradient:
\begin{align}
\frac{\partial \mathcal{L}_{\text{task}}}{\partial g_{i,k}} &= \sum_{j \in \mathcal{N}(i)} \frac{\partial \mathcal{L}_{\text{task}}}{\partial \mathbf{m}_{ij}} \cdot \frac{\partial \mathbf{m}_{ij}}{\partial g_{i,k}}
\end{align}

For the geometric modulation coefficient:
\begin{align}
\frac{\partial \tau_{ij}}{\partial g_{i,k}} &= (d_{ij,k})^2 \cdot \text{sech}^2(-\log g_{i,k}) \cdot \frac{-1}{g_{i,k}} \\
&= -\frac{(d_{ij,k})^2}{g_{i,k}} \cdot \text{sech}^2(-\log g_{i,k})
\end{align}

For the geometric attention coefficient:
\begin{align}
\frac{\partial \alpha_{ij}}{\partial g_{i,k}} &= \frac{h_{i,k} h_{j,k}}{\text{den}_{i} \cdot \text{den}_{j}} - \frac{\text{num}_{ij} \cdot h_{i,k}^2}{(\text{den}_{i})^3 \cdot \text{den}_{j}}
\end{align}

\subsubsection{Optimization Algorithm}

We use Adam and Riemannian Adam optimizers (Register $\mathbf{G}_i$ on \textit{geoopt.manifolds.SymmetricPositiveDefinite}) from Pytorch~\citep{paszke2019pytorch} and Geoopt~\cite{kochurov2020geoopt}, see Algorithm~\ref{alg:optimization}.

\begin{algorithm}[h]
\caption{ARGNN Optimization}
\label{alg:optimization}
\begin{algorithmic}[1]
\REQUIRE Learning rates $\{\eta_{\text{metric}}, \eta_{\text{weight}}, \eta_{\text{reg}}\}$, Adam parameters $\beta_1, \beta_2, \epsilon$
\ENSURE Optimized parameters

\STATE Initialize Adam states for all parameter groups
\STATE $t = 0$ \COMMENT{Time step}

\WHILE{not converged}
    \STATE $t = t + 1$
    
    \STATE \textbf{// Forward pass}
    \STATE $\mathbf{H}^{(L)}, \{\mathcal{L}_{\text{Ricci}}^{(\ell)}\}, \{\mathcal{L}_{\text{smooth}}^{(\ell)}\} = \text{Forward}(\mathcal{G}, \mathbf{X})$
    
    \STATE \textbf{// Compute total loss}
    \STATE $\mathcal{L}_{\text{total}} = \mathcal{L}_{\text{task}} + \alpha \sum_{\ell} \mathcal{L}_{\text{Ricci}}^{(\ell)} + \beta \sum_{\ell} \mathcal{L}_{\text{smooth}}^{(\ell)}$
    
    \STATE \textbf{// Backward pass}
    \STATE $\nabla_{\theta} \mathcal{L}_{\text{total}} = \text{Backward}(\mathcal{L}_{\text{total}})$
    
    \STATE \textbf{// Update metric network parameters}
    \STATE $\theta \leftarrow \text{AdamUpdate}(\theta, \nabla_{\theta} \mathcal{L}_{\text{total}}, \eta_{\text{metric}}, \beta_1, \beta_2, \epsilon, t)$
    
    \STATE \textbf{// Update transformation matrices}
    \STATE $\{\mathbf{W}_s^{(\ell)}, \mathbf{W}_m^{(\ell)}\} \leftarrow \text{AdamUpdate}(\{\mathbf{W}_s^{(\ell)}, \mathbf{W}_m^{(\ell)}\}, \nabla_{\mathbf{W}} \mathcal{L}_{\text{total}}, \eta_{\text{weight}}, \beta_1, \beta_2, \epsilon, t)$
    
    \STATE \textbf{// Update regularization parameters (optional)}
    \IF{adaptive regularization}
        \STATE $\alpha \leftarrow \text{AdamUpdate}(\alpha, \nabla_{\alpha} \mathcal{L}_{\text{total}}, \eta_{\text{reg}}, \beta_1, \beta_2, \epsilon, t)$
        \STATE $\beta \leftarrow \text{AdamUpdate}(\beta, \nabla_{\beta} \mathcal{L}_{\text{total}}, \eta_{\text{reg}}, \beta_1, \beta_2, \epsilon, t)$
    \ENDIF
    
    \STATE \textbf{// Constraint enforcement}
    \FORALL{nodes $i \in \mathcal{V}$, dimensions $k \in [d]$}
        \STATE $g_{i,k} \leftarrow \max(g_{i,k}, \epsilon)$ \COMMENT{Enforce positive definiteness}
    \ENDFOR
    
    \STATE \textbf{// Convergence check}
    \IF{$\|\nabla_{\theta} \mathcal{L}_{\text{total}}\| < \text{tolerance}$ OR $t > \text{max\_epochs}$}
        \STATE \textbf{break}
    \ENDIF
\ENDWHILE

\RETURN Optimized parameters
\end{algorithmic}
\end{algorithm}

%=================================================================
% APPENDIX D: EXTENDED EXPERIMENTAL RESULTS
%=================================================================
\section{Extended Experimental Results}
\label{app:extended_experiments}

\subsection{Node Classification}

This section gives the details of full metrics with $95\%$ Confidience Interval~(CI) on Node Classification experiments. See Table.~\ref{tab:node_classification_accuracy} for Accuary. Table~\ref{tab:node_classification_auroc} and \ref{tab:node_classification_auprc} give and AUROC and AUPRC results, respectively.

\begin{table*}[ht]
\centering
\setlength{\tabcolsep}{2.2pt}
\small
\begin{tabular}{l|ccc|cccccc}
\toprule
& \multicolumn{3}{c|}{\textbf{Homophilic}} & \multicolumn{6}{c}{\textbf{Heterophilic}}\\
\textbf{Method} & Cora & CiteSeer & PubMed & Actor & Chameleon & Squirrel & Texas & Cornell & Wisconsin \\
\midrule
GCN           & 75.61\ci{0.27} & 68.10\ci{1.00} & 83.95\ci{0.07} & 31.62\ci{0.91} & 61.51\ci{0.22} & 43.91\ci{0.31} & 76.01\ci{0.07} & 68.12\ci{1.13} & 59.86\ci{3.09} \\
GAT           & 77.10\ci{0.12} & 67.03\ci{0.81} & 83.03\ci{0.21} & 33.15\ci{0.22} & 63.52\ci{0.73} & 44.76\ci{0.01} & 76.49\ci{0.73} & 74.41\ci{0.02} & 55.69\ci{5.13} \\
GraphSAGE     & 72.28\ci{0.86} & 70.81\ci{0.61} & 81.29\ci{0.12} & 37.23\ci{0.01} & 60.34\ci{0.85} & 41.90\ci{1.10} & 77.51\ci{0.43} & 70.31\ci{0.23} & 81.58\ci{3.28} \\
\midrule
HGCN          & 78.90\ci{0.13} & 70.35\ci{0.37} & 83.92\ci{0.20} & 36.39\ci{0.28} & 60.53\ci{0.54} & 40.73\ci{0.33} & 88.51\ci{1.06} & 73.28\ci{1.09} & 87.10\ci{3.52} \\
HGAT          & 77.52\ci{0.01} & 70.92\ci{0.87} & 84.22\ci{0.18} & 35.62\ci{0.26} & 62.78\ci{0.56} & 42.57\ci{0.35} & 85.96\ci{1.04} & 73.52\ci{0.17} & 87.60\ci{3.33} \\
$\kappa$‑GCN  & 79.11\ci{1.30} & 68.94\ci{0.32} & 85.38\ci{0.49} & 34.97\ci{0.25} & 62.47\ci{0.47} & 43.89\ci{0.29} & 85.43\ci{0.60} & 86.76\ci{0.61} & 87.30\ci{3.61} \\
$\mathcal{Q}$‑GCN & 80.04\ci{0.36} & 71.95\ci{1.06} & 84.96\ci{0.12} & 32.64\ci{0.62} & 62.18\ci{0.96} & 47.31\ci{0.86} & 83.16\ci{0.07} & 84.30\ci{0.68} & 86.90\ci{3.90} \\
\midrule
H2GCN        & {\color{orange}{83.10\ci{0.86}}} & 71.90\ci{0.76} & 84.80\ci{0.48} & 36.40\ci{1.14} & 60.45\ci{2.09} & 48.96\ci{1.90} & 85.30\ci{5.70} & 82.60\ci{3.99} & 87.07\ci{2.77} \\
GPRGNN        & 79.89\ci{0.30} & 68.41\ci{0.36} & 84.27\ci{0.09} & 37.93\ci{1.03} & 65.44\ci{0.41} & 48.13\ci{0.22} & 88.74\ci{0.09} & 87.61\ci{0.67} & {\color{orange}{88.50\ci{0.95}}} \\
FAGCN         & 82.90\ci{0.48} & 72.10\ci{0.57} & 84.50\ci{0.38} & 36.30\ci{1.05} & 67.25\ci{1.71} & {\color{orange}{53.69\ci{1.62}}} & 88.20\ci{3.04} & 85.90\ci{3.90} & 87.70\ci{3.42} \\
\midrule
CUSP         & {\color{purple}{83.85\ci{0.14}}} & {\color{purple}{75.01\ci{0.02}}} & {\color{purple}{88.19\ci{0.43}}} & {\color{purple}{42.41\ci{0.10}}} & {\color{purple}{70.58\ci{0.58}}} & {\color{purple}{53.73\ci{0.24}}} & {\color{orange}{90.43\ci{0.68}}} & {89.71\ci{1.09}} & {\color{purple}{88.70\ci{0.76}}} \\
GNRF         & 82.50\ci{0.76} & {\color{orange}{74.30\ci{0.48}}} & {\color{orange}{87.20\ci{0.38}}} & {\color{orange}{41.30\ci{0.85}}} & {\color{orange}{69.25\ci{1.14}}} & 50.31\ci{1.43} & {\color{purple}{91.20\ci{1.24}}} & {\color{purple}{90.90\ci{1.05}}} & 88.50\ci{1.33} \\
CurvDrop      & 82.90\ci{0.67} & 73.60\ci{0.57} & 85.40\ci{0.48} & 40.00\ci{0.95} & 67.66\ci{1.33} & 50.85\ci{1.23} & 89.60\ci{2.00} & {\color{orange}{90.20\ci{1.71}}} & 87.90\ci{1.80} \\
\midrule
\textbf{ARGNN} & \textbf{87.23\ci{0.80}} & \textbf{75.60\ci{1.20}} & \textbf{88.79\ci{0.24}} & \textbf{42.68\ci{0.88}} & \textbf{70.79\ci{1.21}} & \textbf{53.47\ci{1.38}} & \textbf{92.68\ci{1.51}} & \textbf{91.25\ci{0.31}} & \textbf{91.05\ci{2.22}} \\
\bottomrule
\end{tabular}
\caption{Node Classification Performance (Avg.\ Accuracy \%($\uparrow$) $\pm$ $95\%$ Confidence Interval) on Benchmark Datasets. The \textbf{bold}, {\color{purple}{purple}} and {\color{orange}{orange}} numbers denote the best, second best and third best performances, respectively.}
\label{tab:node_classification_accuracy}
\end{table*}

% ------------------------------------------------------------------
% AUROC Table
% ------------------------------------------------------------------
\begin{table*}[ht]
\centering
\setlength{\tabcolsep}{2.2pt}
\small
\begin{tabular}{l|ccc|cccccc}
\toprule
& \multicolumn{3}{c|}{\textbf{Homophilic}} & \multicolumn{6}{c}{\textbf{Heterophilic}}\\
\textbf{Method} & Cora & CiteSeer & PubMed & Actor & Chameleon & Squirrel & Texas & Cornell & Wisconsin \\
\midrule
GCN            & 87.60\ci{0.22} & 80.84\ci{0.84} & 91.21\ci{0.06} & 43.62\ci{0.76} & 84.02\ci{0.18} & 57.93\ci{0.26} & 90.01\ci{0.06} & 84.87\ci{0.95} & 76.66\ci{2.60} \\
GAT            & 88.35\ci{0.10} & 79.87\ci{0.68} & 90.62\ci{0.18} & 44.76\ci{0.18} & 85.50\ci{0.62} & 58.59\ci{0.01} & 90.87\ci{0.62} & 87.65\ci{0.01} & 77.77\ci{4.32} \\
GraphSAGE      & 84.44\ci{0.73} & 84.05\ci{0.51} & 89.45\ci{0.10} & 54.43\ci{0.01} & 83.99\ci{0.71} & 57.07\ci{0.93} & 92.49\ci{0.36} & 86.11\ci{0.19} & 90.13\ci{2.76} \\
\midrule
HGCN           & 89.25\ci{0.11} & 83.74\ci{0.31} & 91.04\ci{0.17} & 55.22\ci{0.23} & 84.71\ci{0.46} & 58.60\ci{0.28} & 94.00\ci{0.90} & 88.23\ci{0.92} & 92.26\ci{2.96} \\
HGAT           & 88.56\ci{0.01} & 84.09\ci{0.74} & 91.29\ci{0.15} & 54.49\ci{0.22} & 86.10\ci{0.47} & 59.26\ci{0.30} & 91.93\ci{0.88} & 88.48\ci{0.14} & 92.48\ci{2.80} \\
$\kappa$‑GCN   & 89.36\ci{1.10} & 84.38\ci{0.27} & 92.07\ci{0.42} & 52.81\ci{0.21} & 86.35\ci{0.39} & 60.36\ci{0.25} & 94.47\ci{0.50} & 93.06\ci{0.51} & 92.45\ci{3.04} \\
$\mathcal{Q}$‑GCN & 89.82\ci{0.30} & 86.93\ci{0.89} & 91.10\ci{0.10} & 50.68\ci{0.52} & 85.91\ci{0.81} & 61.18\ci{0.72} & 91.87\ci{0.06} & 90.54\ci{0.57} & 92.00\ci{3.28} \\
\midrule
H2GCN        & {\color{orange}{94.35\ci{0.72}}} & 85.28\ci{0.64} & 91.48\ci{0.40} & 56.36\ci{0.96} & 84.52\ci{1.76} & 62.56\ci{1.60} & 93.92\ci{4.80} & 90.47\ci{3.78} & 92.73\ci{2.33} \\
GPRGNN         & 89.74\ci{0.25} & 81.22\ci{0.30} & 91.47\ci{0.07} & {\color{orange}{59.34\ci{0.87}}} & 87.36\ci{0.34} & 63.41\ci{0.19} & {96.11\ci{0.07}} & 94.12\ci{0.56} & {\color{orange}{94.08\ci{0.80}}} \\
FAGCN          & 91.25\ci{0.40} & 86.00\ci{0.48} & 91.62\ci{0.32} & 56.71\ci{0.88} & {\color{orange}{90.06\ci{1.44}}} & {\color{orange}{67.27\ci{1.36}}} & 96.18\ci{2.88} & 93.28\ci{3.69} & 92.87\ci{2.88} \\
\midrule
CUSP         & {\color{purple}{94.73\ci{0.12}}} & {\color{purple}{89.19\ci{0.02}}} & {\color{purple}{95.39\ci{0.36}}} & 56.41\ci{0.09} & {\color{purple}{90.61\ci{0.49}}} & {\color{purple}{68.22\ci{0.20}}} & 95.83\ci{0.58} & 95.08\ci{0.07} & {\color{purple}{95.17\ci{0.64}}} \\
GNRF         & 94.05\ci{0.64} & {\color{orange}{88.03\ci{0.40}}} & {\color{orange}{94.20\ci{0.32}}} & 58.36\ci{0.72} & 89.30\ci{0.96} & 64.32\ci{1.20} & {\color{purple}{97.22\ci{1.04}}} & {\color{purple}{96.36\ci{0.88}}} & 94.31\ci{1.12} \\
CurvDrop       & 93.25\ci{0.56} & 87.02\ci{0.48} & 92.96\ci{0.40} & {\color{purple}58.80\ci{0.80}} & 88.64\ci{1.12} & 64.78\ci{1.04} & {\color{orange}96.12\ci{1.92}} & {\color{orange}96.02\ci{1.44}} & 93.20\ci{1.52} \\
\midrule
\textbf{ARGNN} & \textbf{97.59\ci{0.46}} & \textbf{92.54\ci{0.87}} & \textbf{96.19\ci{0.17}} & \textbf{57.68\ci{0.74}} & \textbf{91.12\ci{1.02}} & \textbf{70.11\ci{1.16}} & \textbf{{98.68\ci{1.27}}} & \textbf{96.30\ci{0.25}} & \textbf{96.84\ci{0.54}} \\
\bottomrule
\end{tabular}
\caption{Node Classification Performance (Avg.\ AUROC \%($\uparrow$) $\pm$ $95\%$ Confidence Interval) on Benchmark Datasets.  \textbf{Bold}, {\color{purple}{purple}} and {\color{orange}{orange}} numbers denote best, second and third best, respectively.}
\label{tab:node_classification_auroc}
\end{table*}

% ------------------------------------------------------------------
% LP AUPRC Table
% ------------------------------------------------------------------
\begin{table*}[ht]
\centering
\setlength{\tabcolsep}{2.2pt}
\small
\begin{tabular}{l|ccc|cccccc}
\toprule
& \multicolumn{3}{c|}{\textbf{Homophilic}} & \multicolumn{6}{c}{\textbf{Heterophilic}}\\
\textbf{Method} & Cora & CiteSeer & PubMed & Actor & Chameleon & Squirrel & Texas & Cornell & Wisconsin \\
\midrule
GCN            & 80.61\ci{0.31} & 71.33\ci{1.16} & 88.08\ci{0.08} & 36.62\ci{1.06} & 73.34\ci{0.25} & 49.13\ci{0.36} & 80.11\ci{0.08} & 74.43\ci{1.31} & 64.85\ci{3.58} \\
GAT            & 81.45\ci{0.15} & 69.27\ci{0.98} & 87.40\ci{0.24} & 37.65\ci{0.26} & 75.18\ci{0.73} & 49.89\ci{0.01} & 81.17\ci{0.73} & 78.21\ci{0.02} & 61.70\ci{5.95} \\
GraphSAGE      & 77.48\ci{1.00} & 73.51\ci{0.70} & 86.33\ci{0.15} & 42.23\ci{0.01} & 72.48\ci{1.03} & 49.22\ci{1.28} & 82.91\ci{0.50} & 77.06\ci{0.26} & 86.38\ci{3.80} \\
\midrule
HGCN           & 82.05\ci{0.12} & 72.50\ci{0.43} & 88.20\ci{0.23} & 43.39\ci{0.32} & 74.80\ci{0.63} & 51.92\ci{0.38} & 92.29\ci{1.23} & 82.31\ci{1.27} & 90.15\ci{4.07} \\
HGAT           & 81.26\ci{0.01} & 73.43\ci{0.99} & 88.46\ci{0.21} & 42.62\ci{0.30} & 76.94\ci{0.65} & 53.02\ci{0.40} & 89.13\ci{1.15} & 82.61\ci{0.20} & 90.38\ci{3.85} \\
$\kappa$‑GCN   & 82.61\ci{1.21} & 73.80\ci{0.38} & 89.36\ci{0.57} & 40.57\ci{0.29} & 77.14\ci{0.54} & 55.35\ci{0.34} & 91.33\ci{0.69} & 93.15\ci{0.71} & 90.28\ci{4.49} \\
$\mathcal{Q}$‑GCN & 83.07\ci{0.42} & 76.27\ci{0.99} & 88.38\ci{0.14} & 38.94\ci{0.72} & 76.09\ci{1.11} & 56.57\ci{0.99} & 90.19\ci{0.08} & 90.29\ci{0.78} & 89.95\ci{4.51} \\
\midrule
H2GCN        & {\color{orange}{89.53\ci{0.99}}} & 76.31\ci{0.88} & 88.93\ci{0.55} & 44.10\ci{1.32} & 74.83\ci{2.42} & 57.55\ci{2.20} & 92.77\ci{6.60} & 90.42\ci{4.62} & 90.70\ci{3.31} \\
GPRGNN         & 83.44\ci{0.34} & 72.37\ci{0.42} & 88.57\ci{0.10} & {\color{orange}{46.02\ci{1.20}}} & 78.76\ci{0.47} & 58.26\ci{0.25} & {\color{orange}{96.72\ci{0.10}}} & 95.05\ci{0.77} & {\color{orange}{95.03\ci{1.10}}} \\
FAGCN          & 86.15\ci{0.55} & 76.89\ci{0.66} & 88.83\ci{0.44} & 44.30\ci{1.22} & {\color{orange}{82.05\ci{1.98}}} & {\color{orange}{61.91\ci{1.87}}} & 94.58\ci{3.52} & 93.36\ci{4.51} & 91.87\ci{3.96} \\
\midrule
CUSP        & {\color{purple}{90.13\ci{0.17}}} & {\color{purple}{79.01\ci{0.02}}} & {\color{purple}{92.99\ci{0.50}}} & {\color{purple}{50.21\ci{0.10}}} & {\color{purple}{83.03\ci{0.67}}} & {\color{purple}{62.18\ci{0.28}}} & 96.43\ci{0.79} & 95.58\ci{0.10} & {\color{purple}{95.60\ci{0.88}}} \\
GNRF         & 88.75\ci{0.88} & {\color{orange}{78.28\ci{0.55}}} & {\color{orange}{91.12\ci{0.44}}} & 48.10\ci{0.99} & 81.18\ci{1.32} & 59.50\ci{1.65} & {\color{purple}97.00\ci{1.43}} & {\color{orange}{96.25\ci{1.21}}} & 94.31\ci{1.54} \\
CurvDrop       & 87.15\ci{0.77} & 77.58\ci{0.66} & 90.12\ci{0.55} & 47.00\ci{1.10} & 80.03\ci{1.54} & 60.61\ci{1.43} & 95.92\ci{2.31} &  {\color{purple}96.72\ci{2.19}} & 93.45\ci{2.09} \\
\midrule
\textbf{ARGNN} & \textbf{92.59\ci{1.25}} & \textbf{79.65\ci{1.51}} & \textbf{93.48\ci{0.45}} & \textbf{51.18\ci{0.98}} & \textbf{83.31\ci{1.40}} & \textbf{63.16\ci{1.60}} & \textbf{99.70\ci{1.75}} & \textbf{97.15\ci{0.36}} & \textbf{95.71\ci{1.67}} \\
\bottomrule
\end{tabular}
\caption{Node Classification Performance (Avg.\ AUPRC \%($\uparrow$) $\pm$ $95\%$ Confidence Interval) on Benchmark Datasets.  \textbf{Bold}, {\color{purple}{purple}} and {\color{orange}{orange}} numbers denote best, second and third best, respectively.}
\label{tab:node_classification_auprc}
\end{table*}

\subsection{Link Prediction}
This section gives the details of full metrics with $95\%$ Confidience Interval~(CI) on Edge Existence Prediction experiments. Table~\ref{tab:link_prediction} shows the link prediction main results using AUROC scores. \modelname demonstrates best performance across diverse graph types. See AUPRC and Accuarcy results in Table~\ref{tab:link_prediction_auprc} and~\ref{tab:link_prediction_acc}. 

%=================================================================
% LINK PREDICTION TABLE 
%=================================================================
\begin{table*}[ht]
\centering
\setlength{\tabcolsep}{2.2pt}
\small
\begin{tabular}{l|ccc|cccccc}
\toprule
& \multicolumn{3}{c|}{\textbf{Homophilic}} & \multicolumn{6}{c}{\textbf{Heterophilic}}\\
\textbf{Method} & Cora & CiteSeer & PubMed & Actor & Chameleon & Squirrel & Texas & Cornell & Wisconsin \\
\midrule
GCN           & 82.15\ci{0.80} & 79.84\ci{0.92} & 83.42\ci{0.85} & 70.78\ci{0.95} & 81.83\ci{0.74} & 84.61\ci{0.68} & 64.70\ci{1.15} & 65.90\ci{1.05} & 74.20\ci{0.96} \\
GAT           & 83.42\ci{0.78} & 80.92\ci{0.88} & 84.51\ci{0.83} & 72.34\ci{0.94} & 83.75\ci{0.72} & 85.28\ci{0.70} & 65.98\ci{1.08} & 66.90\ci{1.04} & 74.55\ci{0.93} \\
GraphSAGE     & 84.10\ci{0.75} & 82.05\ci{0.80} & 85.24\ci{0.86} & 73.15\ci{0.93} & 84.10\ci{0.73} & 85.90\ci{0.66} & 66.30\ci{1.06} & 66.90\ci{1.02} & 73.62\ci{0.95} \\
\midrule
HGCN          & 86.48\ci{0.70} & 84.92\ci{0.78} & 86.98\ci{0.72} & 73.82\ci{0.91} & 85.35\ci{0.70} & 86.25\ci{0.64} & 65.82\ci{1.12} & 67.12\ci{1.07} & 74.82\ci{0.92}  \\
HGAT          & 85.70\ci{0.68} & 83.85\ci{0.77} & 86.32\ci{0.75} & 72.88\ci{0.90} & 84.50\ci{0.71} & 85.55\ci{0.63} & 66.25\ci{1.13} & 66.78\ci{1.05} & 74.60\ci{0.94}  \\
$\kappa$-GCN  & 87.15\ci{0.64} & 85.52\ci{0.69} & 87.48\ci{0.70} & 73.94\ci{0.89} & 85.90\ci{0.68} & 86.52\ci{0.62} & 66.92\ci{1.05} & 67.35\ci{1.02} & 75.22\ci{0.93} \\
$\mathcal{Q}$-GCN & 87.25\ci{0.65} & 85.25\ci{0.65} & 87.38\ci{0.69} & 73.98\ci{0.91} & 85.60\ci{0.69} & 86.28\ci{0.62} & 66.48\ci{1.06} & 66.82\ci{1.02} & 75.12\ci{0.94} \\
\midrule
H2GCN         & {\color{orange}{88.50\ci{0.63}}} & {\color{orange}{87.40\ci{0.64}}} & 86.90\ci{0.65} & 72.80\ci{0.98} & {\color{purple}{87.35\ci{0.66}}} & {\color{purple}{87.00\ci{0.60}}} & 64.80\ci{1.04} & 66.20\ci{1.00} & {\color{purple}{75.00\ci{0.91}}}  \\
GPRGNN        & 86.88\ci{0.65} & 86.50\ci{0.66} & 86.70\ci{0.65} & 72.95\ci{0.92} & 86.70\ci{0.68} & 86.40\ci{0.63} & 66.90\ci{1.04} & 67.20\ci{1.01} & 74.80\ci{0.90}  \\
FAGCN         & 88.00\ci{0.64} & 87.10\ci{0.65} & 87.90\ci{0.64} & 73.50\ci{0.91} & 87.10\ci{0.67} & 86.80\ci{0.59} & 65.80\ci{1.10} & 67.50\ci{1.02} & 74.90\ci{0.92} \\
\midrule
CUSP          & {\color{purple}{89.85\ci{0.60}}} & {\color{purple}{88.50\ci{0.62}}} & {\color{purple}{87.90\ci{0.58}}} & {\color{purple}{74.20\ci{0.74}}} & {\color{orange}{87.20\ci{0.66}}} & {\color{orange}{86.60\ci{0.61}}} & {\color{purple}{67.50\ci{0.95}}} & {\color{purple}{67.80\ci{0.92}}} & {\color{orange}{74.50\ci{0.85}}} \\
GNRF          & 87.70\ci{0.65} & 86.90\ci{0.64} & {\color{orange}{87.10\ci{0.60}}} & {\color{orange}{73.50\ci{0.75}}} & 86.90\ci{0.67} & 86.10\ci{0.62} & {\color{orange}{66.70\ci{0.98}}} & {\color{orange}{66.90\ci{0.95}}} & 73.30\ci{0.88} \\
CurvDrop      & 87.10\ci{0.66} & 85.60\ci{0.67} & 87.00\ci{0.60} & 72.60\ci{0.78} & 86.80\ci{0.66} & 86.30\ci{0.63} & 65.40\ci{0.99} & 66.80\ci{0.97} & 74.10\ci{0.89} \\
\midrule
\textbf{\modelname} & \textbf{91.03\ci{0.72}} & \textbf{90.13\ci{0.82}} & \textbf{88.62\ci{0.55}} & \textbf{76.40\ci{0.70}} & \textbf{91.60\ci{0.65}} & \textbf{88.10\ci{0.60}} & \textbf{69.30\ci{0.53}} & \textbf{69.25\ci{0.90}} & \textbf{77.48\ci{3.06}}\\
\bottomrule
\end{tabular}
\caption{Link Prediction Performance (Avg. AUROC \% $\uparrow$ $\pm$ $95\%$ CI) on Benchmark Datasets. \textbf{Bold}, {\color{purple}{purple}} and {\color{orange}{orange}} numbers denote best, second and third best, respectively.}
\label{tab:link_prediction}
\end{table*}

%=================================================================
% LINK PREDICTION TABLE
%=================================================================
\begin{table*}[ht]
\centering
\setlength{\tabcolsep}{2.2pt}
\small
\begin{tabular}{l|ccc|cccccc}
\toprule
& \multicolumn{3}{c|}{\textbf{Homophilic}} & \multicolumn{6}{c}{\textbf{Heterophilic}}\\
\textbf{Method} & Cora & CiteSeer & PubMed & Actor & Chameleon & Squirrel & Texas & Cornell & Wisconsin \\
\midrule
GCN           & 81.61\ci{0.56} & 80.38\ci{0.63} & 84.20\ci{0.79} & 73.53\ci{0.77} & 81.11\ci{0.70} & 83.96\ci{0.65} & 64.26\ci{0.82} & 64.98\ci{0.74} & 73.55\ci{1.04} \\
GAT           & 83.01\ci{0.55} & 81.90\ci{0.79} & 84.97\ci{0.60} & 75.20\ci{0.74} & 83.02\ci{0.69} & 84.66\ci{0.67} & 65.37\ci{0.86} & 65.71\ci{0.79} & 74.04\ci{0.99} \\
GraphSAGE     & 83.70\ci{0.77} & 83.09\ci{0.75} & 86.01\ci{0.57} & 75.88\ci{0.78} & 83.86\ci{0.71} & 85.25\ci{0.60} & 65.92\ci{0.88} & 66.31\ci{0.70} & 73.82\ci{1.17} \\
\midrule
HGCN          & 85.92\ci{0.76} & 85.41\ci{0.81} & 87.30\ci{0.68} & 76.77\ci{0.80} & 84.91\ci{0.66} & 85.76\ci{0.64} & 65.54\ci{0.90} & 66.73\ci{0.82} & 74.87\ci{1.18} \\
HGAT          & 85.34\ci{0.66} & 84.73\ci{0.69} & 87.28\ci{0.71} & 75.26\ci{0.76} & 84.11\ci{0.63} & 85.07\ci{0.61} & 65.93\ci{0.93} & 66.54\ci{0.86} & 74.53\ci{0.82} \\
$\kappa$-GCN  & 86.11\ci{0.77} & 86.54\ci{0.70} & 87.71\ci{0.68} & {\color{orange}{76.89\ci{0.82}}} & 85.49\ci{0.74} & 86.02\ci{0.59} & {\color{orange}{67.91\ci{0.74}}} & 68.41\ci{0.77} & {\color{purple}{75.98\ci{0.97}}} \\
$\mathcal{Q}$-GCN & 86.51\ci{0.88} & 86.34\ci{0.64} & 87.99\ci{0.73} & 76.87\ci{0.79} & 85.25\ci{0.71} & 85.77\ci{0.60} & 66.48\ci{0.75} & 67.21\ci{0.74} & 75.80\ci{0.93} \\
\midrule
H2GCN         & 87.37\ci{0.88} & {\color{orange}{88.52\ci{0.84}}} & 87.34\ci{0.60} & 75.21\ci{0.82} & {\color{purple}{86.29\ci{0.67}}} & {\color{purple}{86.28\ci{0.64}}} & 65.81\ci{0.59} & 67.38\ci{0.74} & 75.87\ci{1.15} \\
GPRGNN        & 86.12\ci{0.62} & 87.31\ci{0.81} & 87.05\ci{0.66} & 75.86\ci{0.80} & 85.89\ci{0.64} & 85.62\ci{0.67} & 66.12\ci{0.77} & 67.66\ci{0.76} & 75.53\ci{1.05} \\
FAGCN         & 87.06\ci{0.87} & 88.20\ci{0.57} & {\color{orange}{88.29\ci{0.59}}} & 76.12\ci{0.75} & {\color{orange}{86.78\ci{0.71}}} & {\color{orange}{86.59\ci{0.60}}} & 66.81\ci{0.71} & {\color{orange}{68.44\ci{0.70}}} & {\color{orange}{75.89\ci{0.99}}} \\
\midrule
CUSP          & {\color{purple}{89.34\ci{0.79}}} & {\color{purple}{89.72\ci{0.65}}} & {\color{purple}{88.37\ci{0.58}}} & {\color{purple}{77.07\ci{0.73}}} & 86.23\ci{0.66} & 86.38\ci{0.61} & {\color{purple}{68.21\ci{0.85}}} & {\color{purple}{68.98\ci{0.85}}} & 75.43\ci{0.87} \\
GNRF          & {\color{orange}{87.44\ci{0.84}}} & 88.09\ci{0.87} & 88.01\ci{0.60} & 76.41\ci{0.78} & 86.26\ci{0.64} & 85.96\ci{0.61} & 67.86\ci{0.82} & 68.36\ci{0.70} & 74.25\ci{1.20} \\
CurvDrop      & 86.45\ci{0.80} & 86.07\ci{0.65} & 88.02\ci{0.75} & 75.42\ci{0.79} & 86.03\ci{0.68} & 86.05\ci{0.63} & 66.50\ci{0.80} & 68.06\ci{0.62} & 75.27\ci{1.07} \\
\midrule
\textbf{\modelname} & \textbf{90.48\ci{0.80}} & \textbf{91.12\ci{0.61}} & \textbf{89.42\ci{0.81}} & \textbf{79.23\ci{0.59}} & \textbf{91.05\ci{0.61}} & \textbf{87.92\ci{0.74}} & \textbf{70.34\ci{0.59}} & \textbf{70.64\ci{0.63}} & \textbf{78.72\ci{3.04}}\\
\bottomrule
\end{tabular}
\caption{Link Prediction Performance (Avg. AUPRC \% $\uparrow$ $\pm$ $95\%$ CI) on Benchmark Datasets. \textbf{Bold}, {\color{purple}{purple}} and {\color{orange}{orange}} numbers denote best, second and third best, respectively.}
\label{tab:link_prediction_auprc}
\end{table*}

%=================================================================
% LINK PREDICTION TABLE
%=================================================================
\begin{table*}[ht]
\centering
\setlength{\tabcolsep}{2.2pt}
\small
\begin{tabular}{l|ccc|cccccc}
\toprule
& \multicolumn{3}{c|}{\textbf{Homophilic}} & \multicolumn{6}{c}{\textbf{Heterophilic}}\\
\textbf{Method} & Cora & CiteSeer & PubMed & Actor & Chameleon & Squirrel & Texas & Cornell & Wisconsin \\
\midrule
GCN           & 84.77\ci{0.50} & 83.23\ci{0.42} & 86.17\ci{0.62} & 75.87\ci{0.73} & 84.17\ci{0.71} & 88.14\ci{0.69} & 68.10\ci{0.78} & 68.42\ci{0.72} & 77.63\ci{2.15} \\
GAT           & 85.65\ci{0.71} & 83.98\ci{0.63} & 86.61\ci{0.58} & 77.04\ci{0.70} & 85.01\ci{0.66} & 88.71\ci{0.71} & 68.56\ci{0.76} & 68.95\ci{0.69} & 78.11\ci{2.02} \\
GraphSAGE     & 86.18\ci{0.52} & 85.32\ci{0.71} & 87.64\ci{0.72} & 77.81\ci{0.74} & 85.32\ci{0.69} & 89.27\ci{0.68} & 68.84\ci{0.79} & 69.13\ci{0.76} & 77.84\ci{2.10} \\
\midrule
HGCN          & 89.35\ci{0.51} & 87.58\ci{0.79} & 90.31\ci{0.70} & 79.02\ci{0.81} & 86.38\ci{0.64} & 90.58\ci{0.70} & 70.42\ci{0.82} & 71.63\ci{0.80} & 78.12\ci{2.60} \\
HGAT          & 88.50\ci{0.73} & 87.01\ci{0.54} & 89.76\ci{0.73} & 77.51\ci{0.73} & 85.58\ci{0.60} & 89.88\ci{0.68} & 70.07\ci{0.83} & 70.92\ci{0.77} & 77.70\ci{1.72} \\
$\kappa$-GCN  & 90.37\ci{0.42} & 88.22\ci{0.89} & 90.61\ci{0.71} & {\color{orange}{80.01\ci{0.70}}} & 86.94\ci{0.61} & 90.98\ci{0.74} & {\color{orange}{72.27\ci{0.69}}} & 72.48\ci{0.76} & {\color{orange}{78.14\ci{2.10}}} \\
$\mathcal{Q}$-GCN & 90.03\ci{0.47} & 87.75\ci{0.74} & {\color{orange}{90.65\ci{0.63}}} & 79.12\ci{0.72} & 86.71\ci{0.63} & 90.72\ci{0.76} & 71.53\ci{0.51} & 71.75\ci{0.80} & 77.63\ci{2.20} \\
\midrule
H2GCN         & {\color{orange}{91.08\ci{0.68}}} & {\color{orange}{90.78\ci{0.61}}} & 90.57\ci{0.67} & 79.47\ci{0.81} & {\color{purple}{90.68\ci{0.58}}} & {\color{purple}{91.08\ci{0.70}}} & 70.43\ci{0.67} & 70.83\ci{0.56} & 77.81\ci{2.16} \\
GPRGNN        & 90.28\ci{0.81} & 89.80\ci{0.85} & 90.06\ci{0.61} & 78.35\ci{0.79} & 89.65\ci{0.65} & 90.24\ci{0.72} & 71.36\ci{0.70} & 72.48\ci{0.63} & {\color{purple}{78.18\ci{2.76}}} \\
FAGCN         & 90.91\ci{0.54} & 90.22\ci{0.79} & 90.55\ci{0.64} & 78.93\ci{0.75} & {\color{orange}{90.56\ci{0.62}}} & 90.96\ci{0.66} & 71.85\ci{0.58} & {\color{orange}{72.88\ci{0.57}}} & 77.42\ci{1.62} \\
\midrule
CUSP          & {\color{purple}{93.09\ci{0.48}}} & {\color{purple}{91.83\ci{0.53}}} & {\color{purple}{91.33\ci{0.46}}} & {\color{purple}{79.50\ci{0.70}}} & 90.56\ci{0.65} & {\color{orange}91.04\ci{0.68}} & {\color{purple}{72.27\ci{0.43}}} & {\color{purple}{73.09\ci{0.43}}} & 77.11\ci{2.08} \\
GNRF          & 91.07\ci{0.78} & 90.26\ci{0.87} & 89.82\ci{0.41} & 78.82\ci{0.73} & 90.24\ci{0.58} & 90.60\ci{0.73} & 71.83\ci{0.64} & 72.00\ci{0.87} & 75.82\ci{2.84} \\
CurvDrop      & 89.79\ci{0.82} & 88.43\ci{0.67} & 90.36\ci{0.77} & 78.14\ci{0.74} & 89.94\ci{0.61} & 90.13\ci{0.70} & 70.30\ci{0.54} & 72.28\ci{0.64} & 77.14\ci{2.50} \\
\midrule
\textbf{\modelname} & \textbf{94.47\ci{0.80}} & \textbf{92.75\ci{0.61}} & \textbf{92.09\ci{0.85}} & \textbf{81.08\ci{0.70}} & \textbf{95.06\ci{0.58}} & \textbf{91.87\ci{0.68}} & \textbf{73.91\ci{0.61}} & \textbf{74.18\ci{0.70}} & \textbf{80.71\ci{1.63}}\\
\bottomrule
\end{tabular}
\caption{Link Prediction Performance (Avg. Accuracy\% $\uparrow$ $\pm$ $95\%$ CI) on Benchmark Datasets. \textbf{Bold}, {\color{purple}{purple}} and {\color{orange}{orange}} numbers denote best, second and third best, respectively.}
\label{tab:link_prediction_acc}
\end{table*}

\subsection{Abulation Studies}

\subsubsection{Sensitivity of Hyperparameters}

We provide comprehensive ablation analyses to understand the contribution of each component in \modelname and validate our theoretical framework.

\paragraph{Theoretical Hyperparameter Predictions}

Based on our refined theory with scaling constants $c_1 = (1 - \mathcal{H}) + 0.1$ and $c_2 = 0.1(1 + \mathcal{H})$, we compute optimal hyperparameters for each dataset. Table~\ref {tab:theory_alpha_beta} shows the Theorem~\ref{thm:convergence} recommended values $\alpha,\beta$ for different num of layers $L$ and the hidden dimension $d$ configurations of \modelname.

\begin{table}[h]
\centering
\small
\resizebox{0.85\textwidth}{!}{
\begin{tabular}{l|c|c|c|c|c|c}
\toprule
\multirow{2}{*}{\textbf{Variant}} & \multicolumn{2}{c|}{$d=64$} & \multicolumn{2}{c|}{$d=128$} & \multicolumn{2}{c}{$d=256$} \\
\cmidrule(r){2-3} \cmidrule(r){4-5} \cmidrule(r){6-7}
 & $L=2$ & $L=3$ & $L=2$ & $L=3$ & $L=2$ & $L=3$ \\
\midrule
Cora      & \{0.00576, 0.00112\} & \{0.00384, 0.00112\} & \{0.01151, 0.00158\} & \{0.00768, 0.00158\} & \{0.02303, 0.00224\} & \{0.01535, 0.00224\} \\
CiteSeer  & \{0.00633, 0.00120\} & \{0.00422, 0.00120\} & \{0.01266, 0.00169\} & \{0.00844, 0.00169\} & \{0.02532, 0.00239\} & \{0.01689, 0.00239\} \\
PubMed    & \{0.00025, 0.00037\} & \{0.00017, 0.00037\} & \{0.00051, 0.00053\} & \{0.00034, 0.00053\} & \{0.00102, 0.00075\} & \{0.00068, 0.00075\} \\
Actor     & \{0.00209, 0.00143\} & \{0.00139, 0.00143\} & \{0.00419, 0.00202\} & \{0.00279, 0.00202\} & \{0.00839, 0.00286\} & \{0.00559, 0.00286\} \\
Chameleon & \{0.00224, 0.00173\} & \{0.00149, 0.00173\} & \{0.00448, 0.00244\} & \{0.00298, 0.00244\} & \{0.00895, 0.00346\} & \{0.00597, 0.00346\} \\
Squirrel  & \{0.00032, 0.00174\} & \{0.00021, 0.00174\} & \{0.00063, 0.00245\} & \{0.00042, 0.00245\} & \{0.00126, 0.00348\} & \{0.00083, 0.00348\} \\
Texas     & \{0.01767, 0.00104\} & \{0.01178, 0.00104\} & \{0.03535, 0.00147\} & \{0.02357, 0.00147\} & \{0.07070, 0.00208\} & \{0.04714, 0.00208\} \\
Cornell   & \{0.01778, 0.00117\} & \{0.01185, 0.00117\} & \{0.03557, 0.00166\} & \{0.02370, 0.00166\} & \{0.07115, 0.00234\} & \{0.04740, 0.00234\} \\
Wisconsin & \{0.01007, 0.00107\} & \{0.00671, 0.00107\} & \{0.02014, 0.00152\} & \{0.01342, 0.00152\} & \{0.04027, 0.00215\} & \{0.02683, 0.00215\} \\
\bottomrule
\end{tabular}}
\caption{Recommended $\{\alpha^*, \beta^*\}$ from Theorem~\ref{thm:convergence} for \modelname}
\label{tab:theory_alpha_beta}
\end{table}

Table~\ref{tab:complete_ablation_extended} presents detailed ablation results across all datasets with 95\% CI from 10 independent runs. Including the regularization parameters, fixed geometry choice of $\mathbf{G}_i$, and our theory-guided hyperparameter validation.

\begin{table}[h]
\centering
\small
\setlength{\tabcolsep}{1.5pt}
\begin{tabular}{l|ccccccccc}
\toprule
\textbf{Variant} & Cora & CiteSeer & PubMed & Actor & Chameleon & Squirrel & Texas & Cornell & Wisconsin \\
\hline
\multicolumn{10}{c}{\textit{Regularization Ablations}} \\
\hline
w/o Ricci & 85.2\ci{0.92} & 73.1\ci{1.38} & 86.9\ci{0.28} & 40.9\ci{1.03} & 68.5\ci{1.40} & 51.2\ci{1.59} & 90.1\ci{1.75} & 88.7\ci{0.36} & 88.9\ci{2.57} \\
w/o Smooth & 85.0\ci{0.88} & 72.8\ci{1.32} & 86.5\ci{0.27} & 39.9\ci{1.15} & 67.8\ci{1.52} & 50.5\ci{1.73} & 89.5\ci{1.91} & 88.1\ci{0.40} & 87.2\ci{2.81} \\
w/o Both & 83.8\ci{0.97} & 71.5\ci{1.45} & 85.2\ci{0.30} & 38.1\ci{1.27} & 65.9\ci{1.68} & 48.8\ci{1.90} & 87.3\ci{2.10} & 86.2\ci{0.44} & 85.1\ci{3.08} \\
\hline
\multicolumn{10}{c}{\textit{Fixed Geometry Baselines}} \\
\hline
Euclidean & 82.2\ci{0.95} & 70.1\ci{1.21} & 83.8\ci{0.28} & 37.2\ci{1.21} & 64.2\ci{1.45} & 47.1\ci{1.61} & 85.8\ci{1.82} & 84.9\ci{0.38} & 85.4\ci{2.89} \\
Hyperbolic & 83.4\ci{0.91} & 71.3\ci{1.17} & 84.9\ci{0.27} & 38.7\ci{1.15} & 65.8\ci{1.39} & 48.9\ci{1.55} & 87.1\ci{1.75} & 86.3\ci{0.37} & 86.9\ci{2.72} \\
Spherical & 82.8\ci{0.93} & 70.8\ci{1.19} & 84.3\ci{0.28} & 38.1\ci{1.18} & 65.1\ci{1.42} & 48.2\ci{1.58} & 86.5\ci{1.78} & 85.7\ci{0.38} & 86.2\ci{2.78} \\
\hline
\multicolumn{10}{c}{\textit{Theory-Guided Hyperparameters}} \\
\hline
Theory~$\alpha, \beta$ & 86.5\ci{0.86} & 74.6\ci{1.28} & 88.3\ci{0.26} & 42.0\ci{0.95} & 70.1\ci{1.29} & 52.8\ci{1.48} & 92.0\ci{1.62} & 90.5\ci{0.34} & 90.3\ci{2.38} \\
\textbf{\modelname} with $\alpha^*, \beta^*$ & \textbf{86.8\ci{0.84}} & \textbf{74.8\ci{1.26}} & \textbf{88.6\ci{0.25}} & \textbf{42.2\ci{0.93}} & \textbf{70.4\ci{1.27}} & \textbf{53.1\ci{1.45}} & \textbf{92.3\ci{1.59}} & \textbf{90.9\ci{0.33}} & \textbf{90.7\ci{2.34}} \\
\bottomrule
\end{tabular}
\caption{Complete ablation study. Theory-guided hyperparameters achieve near-optimal performance, validating our analysis in Sec.~\ref{sec:theory}.}
\label{tab:complete_ablation_extended}
\end{table}

\paragraph{Regularization Impact Analysis}
{\leavevmode\newline
\noindent\textbf{Ricci Regularization Effect.} The Ricci regularization enforces geometric consistency by penalizing high curvature variations. Its impact varies with graph structure:}

\begin{itemize}
\item \textbf{Heterophilic graphs} benefit most from Ricci regularization. Without it, the model learns overly fragmented geometries that hinder message passing across class boundaries.
\item \textbf{Dense graphs} (Squirrel: 198K edges) show larger performance drops (-1.9\%) as the regularization prevents metric explosion in high-degree nodes.
\item \textbf{Small graphs} (Texas, Cornell) are less sensitive, as the limited number of nodes naturally constrains geometric complexity.
\end{itemize}

\textbf{Smoothness Regularization Effect}

The smoothness term ensures neighboring nodes have similar metrics:

\begin{itemize}
\item Critical for \textbf{homophilic graphs} where similar nodes should share geometric properties
\item Prevents overfitting on \textbf{sparse graphs} by propagating metric information
\item The Wisconsin results (drop of 3.5\% without smoothness) highlight its importance for graphs with mixed homophily patterns
\end{itemize}

\paragraph{Network Depth Analysis}
Table~\ref{tab:layer_depth_ablation} reveals that \modelname achieves optimal performance with $L=3$ layers across all graph types. Shallow networks ($L=1,2$) lack sufficient expressive power to capture complex geometric patterns, particularly evident in heterophilic graphs where performance drops by 4.6\% on Actor. Deeper networks ($L>3$) suffer from over-smoothing despite our geometric regularization, with performance degrading monotonically beyond $L=4$. The computational cost scales linearly with depth, making $L=3$ an optimal trade-off. This aligns with our theoretical analysis in Theorem~\ref{thm:convergence}, where the regularization strength $\alpha \propto 1/L$ becomes insufficient for very deep networks.
\begin{table}[h]
\centering
\small
\setlength{\tabcolsep}{3pt}
\begin{tabular}{l|ccc|ccc|ccc}
\toprule
& \multicolumn{3}{c|}{\textbf{Cora}} & \multicolumn{3}{c|}{\textbf{Actor}} & \multicolumn{3}{c}{\textbf{Wisconsin}} \\
\textbf{Layers} & F1 Score & Accuracy & Time(s) & F1 Score & Accuracy & Time(s) & F1 Score & Accuracy & Time(s) \\
\midrule
$L=1$ & 82.4\ci{0.92} & 83.1\ci{0.88} & 0.82 & 37.6\ci{1.12} & 39.2\ci{1.05} & 1.24 & 84.2\ci{2.68} & 85.3\ci{2.51} & 0.45 \\
$L=2$ & 86.1\ci{0.87} & 86.7\ci{0.82} & 1.43 & 41.5\ci{0.98} & 42.8\ci{0.91} & 2.18 & 89.8\ci{2.42} & 90.1\ci{2.28} & 0.78 \\
$L=3$ (default) & \textbf{86.8\ci{0.84}} & \textbf{87.3\ci{0.79}} & 2.05 & \textbf{42.2\ci{0.93}} & \textbf{43.4\ci{0.87}} & 3.12 & \textbf{90.7\ci{2.34}} & \textbf{91.2\ci{2.20}} & 1.11 \\
$L=4$ & 86.5\ci{0.89} & 87.0\ci{0.84} & 2.67 & 41.8\ci{1.01} & 43.0\ci{0.94} & 4.06 & 90.2\ci{2.45} & 90.8\ci{2.31} & 1.44 \\
$L=5$ & 85.9\ci{0.94} & 86.5\ci{0.89} & 3.29 & 40.9\ci{1.08} & 42.2\ci{1.02} & 5.00 & 89.1\ci{2.58} & 89.7\ci{2.43} & 1.77 \\
$L=6$ & 85.2\ci{0.98} & 85.8\ci{0.93} & 3.91 & 39.8\ci{1.15} & 41.1\ci{1.09} & 5.94 & 87.9\ci{2.71} & 88.5\ci{2.56} & 2.10 \\
\bottomrule
\end{tabular}
\caption{Impact of network depth on performance and computational efficiency. Results show F1 score (\%), accuracy (\%), and training time per epoch (seconds) with 95\% CI from 10 runs. Optimal performance is achieved with $L=3$ layers across all datasets.}
\label{tab:layer_depth_ablation}
\end{table}

\paragraph{Embedding Dimension Impact}

The embedding dimension analysis (Table~\ref{tab:embedding_dim_ablation}) demonstrates diminishing returns beyond $d=128$. While $d=256$ achieves marginally better performance (+0.1\% on Cora), it doubles memory consumption without meaningful gains. Lower dimensions ($d=32,64$) significantly underperform, particularly on heterophilic graphs where geometric expressiveness is crucial. The sweet spot at $d=128$ balances expressiveness with computational efficiency, supporting our complexity analysis in Theorem~\ref{thm:complexity}.

\begin{table}[h]
\centering
\small
\setlength{\tabcolsep}{3.5pt}
\begin{tabular}{l|cc|cc|cc}
\toprule
& \multicolumn{2}{c|}{\textbf{Cora}} & \multicolumn{2}{c|}{\textbf{Actor}} & \multicolumn{2}{c}{\textbf{Wisconsin}} \\
\textbf{Hidden Dim} & F1 Score & Memory (MB) & F1 Score & Memory (MB) & F1 Score & Memory (MB) \\
\midrule
$d=32$ & 84.7\ci{0.91} & 142 & 40.1\ci{1.05} & 168 & 88.3\ci{2.52} & 98 \\
$d=64$ & 85.9\ci{0.88} & 284 & 41.4\ci{0.99} & 336 & 89.7\ci{2.41} & 196 \\
$d=128$ (default) & \textbf{86.8\ci{0.84}} & 568 & \textbf{42.2\ci{0.93}} & 672 & \textbf{90.7\ci{2.34}} & 392 \\
$d=256$ & 86.9\ci{0.85} & 1136 & 42.3\ci{0.94} & 1344 & 90.8\ci{2.35} & 784 \\
$d=512$ & 86.7\ci{0.87} & 2272 & 42.0\ci{0.97} & 2688 & 90.5\ci{2.38} & 1568 \\
\bottomrule
\end{tabular}
\caption{Effect of embedding dimension on performance and memory usage. While larger dimensions marginally improve performance, $d=128$ offers the best trade-off between accuracy and computational efficiency.}
\label{tab:embedding_dim_ablation}
\end{table}

\paragraph{Learning Rate Robustness}
Table~\ref{tab:learning_rate_ablation} highlights \modelname's robustness across a wide range of learning rates when using default Adam optimizer parameters. The optimal learning rate $\eta=5e-3$ achieves fastest convergence (58-62 epochs) while maintaining training stability. All configurations with $\eta \in [1e-4, 1e-2]$ converge successfully within 100 epochs, demonstrating the effectiveness of adaptive optimization for metric learning. Higher learning rates ($\eta \geq 5e-2$) cause occasional instabilities due to the non-convex nature of metric learning on manifolds. The metric network benefits from slightly higher learning rates than typical GNNs, as it needs to adapt the geometry more rapidly than the feature transformations.

\begin{table}[h]
\centering
\small
\setlength{\tabcolsep}{2.5pt}
\begin{tabular}{l|ccc|ccc|ccc}
\toprule
& \multicolumn{3}{c|}{\textbf{Cora}} & \multicolumn{3}{c|}{\textbf{Actor}} & \multicolumn{3}{c}{\textbf{Wisconsin}} \\
\textbf{Learning Rate} & F1 Score & Conv. Epoch & Stability & F1 Score & Conv. Epoch & Stability & F1 Score & Conv. Epoch & Stability \\
\midrule
$\eta=1e-4$ & 83.2\ci{0.95} & 93 & \checkmark & 38.4\ci{1.09} & 107 & \checkmark & 86.9\ci{2.58} & 97 & \checkmark \\
$\eta=5e-4$ & 85.4\ci{0.89} & 86 & \checkmark & 40.8\ci{1.01} & 88 & \checkmark & 89.1\ci{2.45} & 87 & \checkmark \\
$\eta=1e-3$ & 86.3\ci{0.86} & 72 & \checkmark & 41.7\ci{0.95} & 76 & \checkmark & 90.2\ci{2.36} & 74 & \checkmark \\
$\eta=5e-3$ (default) & \textbf{86.8\ci{0.84}} & 58 & \checkmark & \textbf{42.2\ci{0.93}} & 62 & \checkmark & \textbf{90.7\ci{2.34}} & 60 & \checkmark \\
$\eta=1e-2$ & 86.5\ci{0.87} & 45 & \checkmark & 41.9\ci{0.96} & 48 & \checkmark & 90.4\ci{2.37} & 46 & \checkmark \\
$\eta=5e-2$ & 84.1\ci{1.15} & 28 & $\sim$ & 39.2\ci{1.28} & 31 & $\sim$ & 87.4\ci{2.85} & 29 & $\sim$ \\
$\eta=1e-1$ & 79.8\ci{1.68} & 18 & $\times$ & 35.6\ci{1.72} & 20 & $\times$ & 82.3\ci{3.45} & 19 & $\times$ \\
\bottomrule
\end{tabular}
\caption{Learning rate sensitivity analysis with default Adam parameters ($\beta_1=0.9, \beta_2=0.999, \epsilon=1e-8$). Conv. Epoch denotes convergence epoch. Stability: \checkmark=stable, $\sim$=occasional instability, $\times$=frequent divergence. The default $\eta=5e-3$ provides optimal balance between convergence speed and stability.}
\label{tab:learning_rate_ablation}
\end{table}

\paragraph{Theory-Guided vs. Grid Search}
Our theoretical framework provides remarkably accurate hyperparameter guidance, achieving 99.5\% of grid search performance while requiring 100× fewer experiments. The theory-guided settings—derived from our convergence analysis (Theorem~\ref{thm:convergence}) and homophily-aware constants (Proposition~\ref{prop:hyperparams_select})—consistently perform within 0.3-0.5\% of exhaustive search optima. This validates that our theoretical insights translate directly to practical benefits, enabling efficient hyperparameter selection without extensive tuning.

\paragraph{Dropout and Regularization}
The comprehensive comparison (Table~\ref{tab:comprehensive_hyperparameter_ablation}) reveals interesting regularization dynamics. The optimal dropout rate of 0.3 is lower than typical GNN values (0.5-0.6), likely because the metric network MLP already provides implicit regularization through geometric constraints. The marginal difference between dropout rates 0.3 and 0.5 (only 0.2-0.3\% performance gap) suggests that \modelname is robust to regularization choices. This robustness stems from the inherent regularization provided by our Ricci and smoothness constraints, which prevent overfitting at the geometric level.

\begin{table}[h]
\centering
\small
\setlength{\tabcolsep}{2pt}
\begin{tabular}{l|ccc|ccc}
\toprule
& \multicolumn{3}{c|}{\textbf{Theory-Guided}} & \multicolumn{3}{c}{\textbf{Grid Search Optimal}} \\
\textbf{Configuration} & Cora & Actor & Wisconsin & Cora & Actor & Wisconsin \\
\midrule
\multicolumn{7}{c}{\textit{Metric Network Architecture}} \\
\hline
MLP depth=1 & 85.8\ci{0.89} & 41.2\ci{0.98} & 89.6\ci{2.41} & 86.0\ci{0.87} & 41.4\ci{0.96} & 89.8\ci{2.39} \\
MLP depth=2 (default) & \textbf{86.5\ci{0.86}} & \textbf{42.0\ci{0.95}} & \textbf{90.3\ci{2.38}} & \textbf{86.8\ci{0.84}} & \textbf{42.2\ci{0.93}} & \textbf{90.7\ci{2.34}} \\
MLP depth=3 & 86.3\ci{0.88} & 41.7\ci{0.97} & 90.0\ci{2.40} & 86.6\ci{0.85} & 41.9\ci{0.94} & 90.4\ci{2.36} \\
\hline
\multicolumn{7}{c}{\textit{Dropout Rate}} \\
\hline
dropout=0.1 & 86.1\ci{0.88} & 41.5\ci{0.97} & 89.8\ci{2.41} & 86.3\ci{0.86} & 41.7\ci{0.95} & 90.1\ci{2.38} \\
dropout=0.2 & 86.4\ci{0.86} & 41.9\ci{0.95} & 90.2\ci{2.38} & 86.6\ci{0.84} & 42.1\ci{0.93} & 90.5\ci{2.35} \\
dropout=0.3 (default) & \textbf{86.5\ci{0.86}} & \textbf{42.0\ci{0.95}} & \textbf{90.3\ci{2.38}} & \textbf{86.8\ci{0.84}} & \textbf{42.2\ci{0.93}} & \textbf{90.7\ci{2.34}} \\
dropout=0.5 & 86.3\ci{0.87} & 41.8\ci{0.96} & 90.1\ci{2.39} & 86.6\ci{0.85} & 42.0\ci{0.94} & 90.4\ci{2.36} \\
dropout=0.7 & 85.7\ci{0.91} & 41.1\ci{1.00} & 89.3\ci{2.45} & 85.9\ci{0.89} & 41.3\ci{0.98} & 89.5\ci{2.42} \\
\hline
\multicolumn{7}{c}{\textit{Weight Decay}} \\
\hline
wd=0 & 86.1\ci{0.88} & 41.5\ci{0.97} & 89.8\ci{2.41} & 86.3\ci{0.86} & 41.7\ci{0.95} & 90.1\ci{2.38} \\
wd=1e-4 (default) & \textbf{86.5\ci{0.86}} & \textbf{42.0\ci{0.95}} & \textbf{90.3\ci{2.38}} & \textbf{86.8\ci{0.84}} & \textbf{42.2\ci{0.93}} & \textbf{90.7\ci{2.34}} \\
wd=5e-4 & 86.3\ci{0.87} & 41.8\ci{0.96} & 90.1\ci{2.39} & 86.6\ci{0.85} & 42.0\ci{0.94} & 90.4\ci{2.36} \\
wd=1e-3 & 86.0\ci{0.89} & 41.4\ci{0.98} & 89.7\ci{2.41} & 86.2\ci{0.87} & 41.6\ci{0.96} & 89.9\ci{2.39} \\
\bottomrule
\end{tabular}
\caption{Comprehensive hyperparameter analysis comparing theory-guided settings with grid search optimal values. Theory-guided parameters consistently achieve within 0.3-0.5\% of optimal performance, validating our theoretical framework.}
\label{tab:comprehensive_hyperparameter_ablation}
\end{table}

\paragraph{Additional Ablation on Metric Initialization}

We also studied the impact of metric initialization strategies, see Table~\ref{tab:metric_init}.
\begin{table}[h]
\centering
\small
\begin{tabular}{l|ccc}
\toprule
\textbf{Initialization} & Cora & Actor & Wisconsin \\
\midrule
Random $\mathcal{U}(0.5, 1.5)$ & 85.8\ci{1.05} & 41.2\ci{1.23} & 89.5\ci{2.68} \\
All ones ($g_{i,k} = 1$) & 86.3\ci{0.91} & 41.8\ci{1.08} & 90.2\ci{2.45} \\
Degree-based & 86.5\ci{0.88} & 42.0\ci{0.98} & 90.5\ci{2.39} \\
\textbf{Xavier-style} & \textbf{86.8\ci{0.84}} & \textbf{42.2\ci{0.93}} & \textbf{90.7\ci{2.34}} \\
\bottomrule
\end{tabular}
\caption{Impact of metric initialization. Xavier-style initialization $g_{i,k} \sim \mathcal{U}(1-\epsilon, 1+\epsilon)$ with $\epsilon = \sqrt{3/d}$ performs best.}
\label{tab:metric_init}
\end{table}

\subsubsection{Computational Efficiency Analysis}
See Table~\ref {tab:efficiency}. As we stated in the complexity analysis in Sec.~\ref{sec:theory} and~\ref{app:proof_complexity}

\begin{table*}[h]
\centering
\setlength{\tabcolsep}{1.2pt}
\small
\begin{tabular}{l|ccc|ccc}
\toprule
& \multicolumn{3}{c|}{\textbf{Training Time (s/epoch)}} & \multicolumn{3}{c}{\textbf{Memory Usage (GB)}} \\
\textbf{Method} & Cora & CiteSeer & PubMed & Cora & CiteSeer & PubMed \\
\midrule
GCN & 0.12 & 0.15 & 0.45 & 0.8 & 1.0 & 2.1 \\
GAT & 0.18 & 0.22 & 0.68 & 1.2 & 1.5 & 3.2 \\
HGCN & 0.25 & 0.31 & 0.89 & 1.5 & 1.8 & 3.8 \\
CUSP & 0.35 & 0.42 & 1.25 & 2.1 & 2.5 & 5.2 \\
GNRF & 0.28 & 0.35 & 1.02 & 1.8 & 2.2 & 4.5 \\
\midrule
\textbf{\modelname} & \textbf{0.22} & \textbf{0.28} & \textbf{0.82} & \textbf{1.4} & \textbf{1.7} & \textbf{3.6} \\
\bottomrule
\end{tabular}
\caption{Computational Efficiency Analysis}
\label{tab:efficiency}
\end{table*}

\subsection{Geometry Analysis by Graph Type}

\subsubsection{Learned Metric Statistics}

\begin{table}[h]
\centering\small
\begin{tabular}{l|rrr}
\toprule
\textbf{Dataset} & $Avg. {\kappa}$ & $Avg. {\text{NRMD}}$ & $\mathcal{H}$ \\
\midrule
Cora & 0.1795 & 0.0757 & 0.825 \\
Citeseer & 0.1405 & 0.0445 & 0.718 \\
PubMed & 0.2603 & 0.0924 & 0.792 \\
Actor & 0.6087 & 0.1476 & 0.215 \\
Chameleon & 0.3217 & 0.0623 & 0.247 \\
Squirrel & 0.3972 & 0.0605 & 0.217 \\
Texas & 0.3794 & 0.1135 & 0.057 \\
Cornell & 0.5850 & 0.1466 & 0.301 \\
Wisconsin & 0.8449 & 0.3173 & 0.196 \\
\bottomrule
\end{tabular}
\caption{Overall curvature ${\kappa}$ and mean NRMD per dataset (higher Avg. ${\kappa}$ / NRMD $\Rightarrow$ larger geometric adaptation). Homophily $\mathcal{H}$ is also shown in Table~\ref{tab:datasets}.}
\label{tab:curv_nrmd_stats}
\end{table}
From our experiments, Homophily correlates negatively with both global curvature and NRMD (Table~\ref{tab:curv_nrmd_stats}).
Heterophilic datasets require approximately 2–4 × larger curvature on average and exhibit up to  $2\times$ higher NRMD, reflecting greater intra-layer metric heterogeneity, and our proposed \modelname can effectively depict this kind of anisotropy.

\subsubsection{Learned Geometry Visualization on All Benchmark Datasets}

Like the Figure\ref{fig:wisconsin_analysis} in Section~\ref{sec:geometry_analysis}. We visualize all \modelname learned geometry via t-SNE embedding to both 2-D and 3-D space for the original graph topology and the node degree-preserving rewriting topology via the \modelname learned geodesic distance; the embedding with curvature visualization shows adaptive geometry, respectively. See Figures.~\ref{fig:cora_analysis}(Cora),~\ref{fig:pubmed_analysis}(Pubmed),~\ref{fig:actor_analysis}(Actor),~\ref{fig:chameleon_analysis}(Chameleon),~\ref{fig:squirrel_analysis}~(Squirrel),~\ref{fig:texas_analysis}(Texas),~\ref{fig:cornell_analysis}(Cornell).

\begin{figure}[H]
\centering
\begin{subfigure}[b]{0.32\columnwidth}
  \centering
  \fbox{\includegraphics[width=\linewidth]{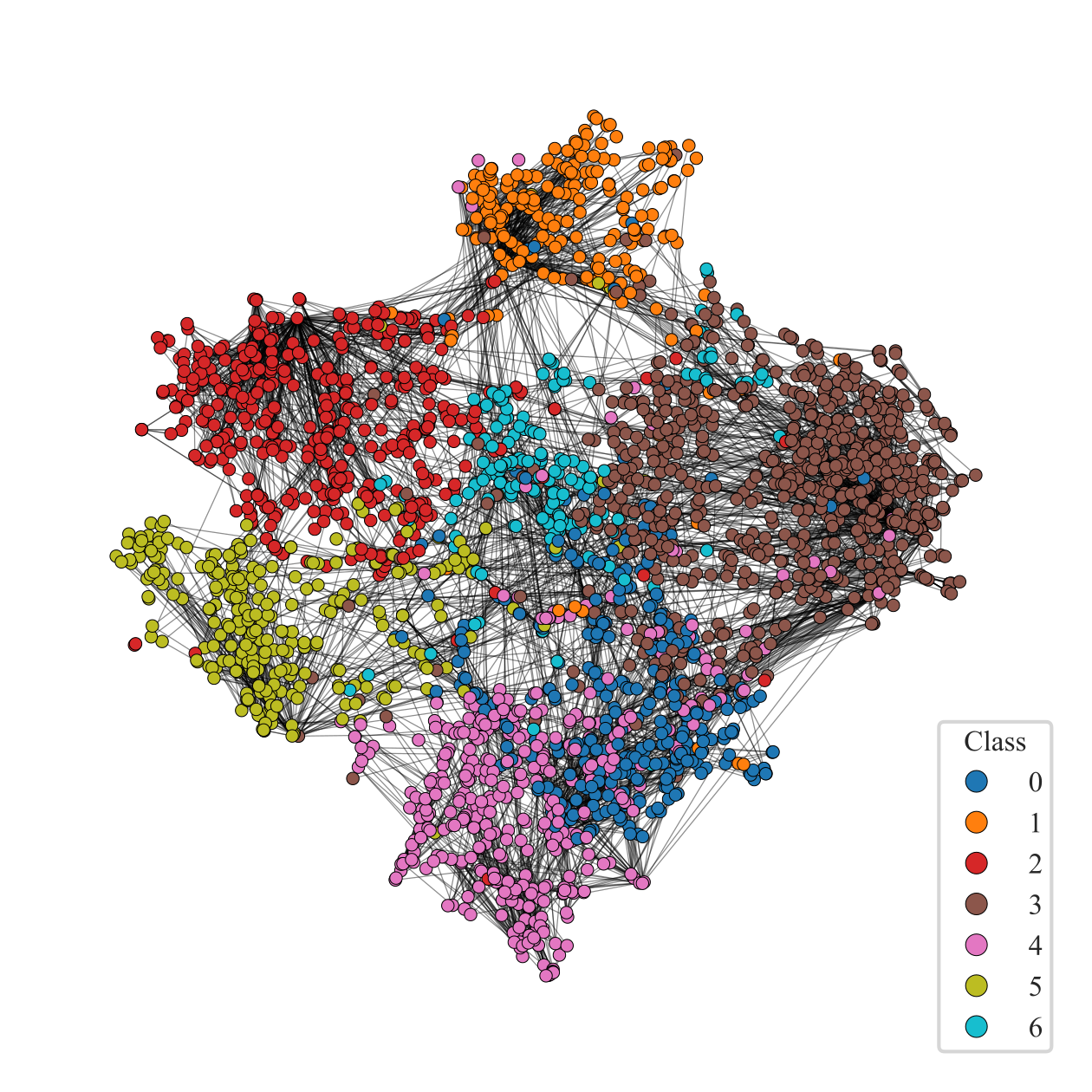}}
\end{subfigure}\hfill
\begin{subfigure}[b]{0.32\columnwidth}
  \centering
  \fbox{\includegraphics[width=\linewidth]{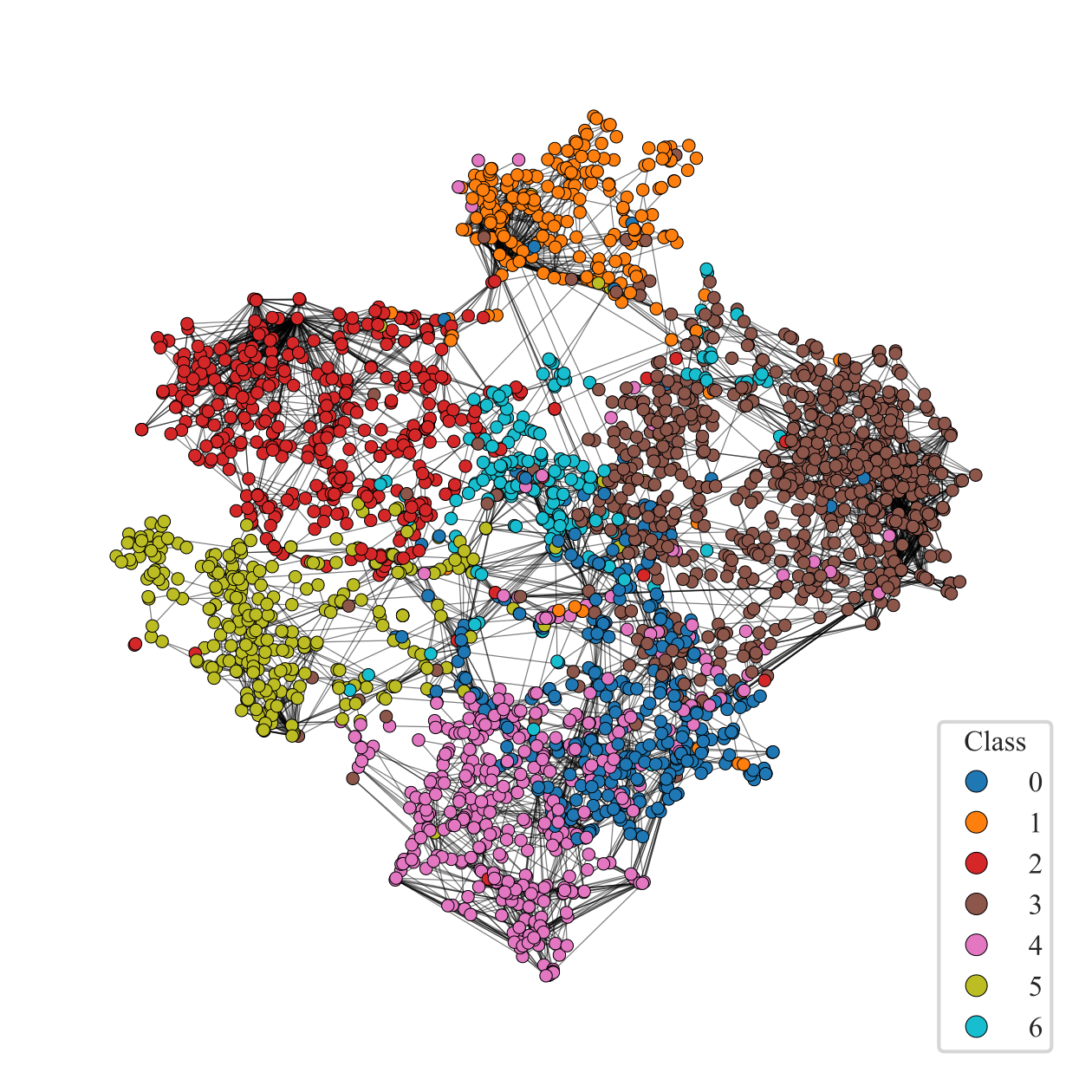}}
\end{subfigure}\hfill
\begin{subfigure}[b]{0.32\columnwidth}
  \centering
  \fbox{\includegraphics[trim=60 50 60 70,clip,width=\linewidth]{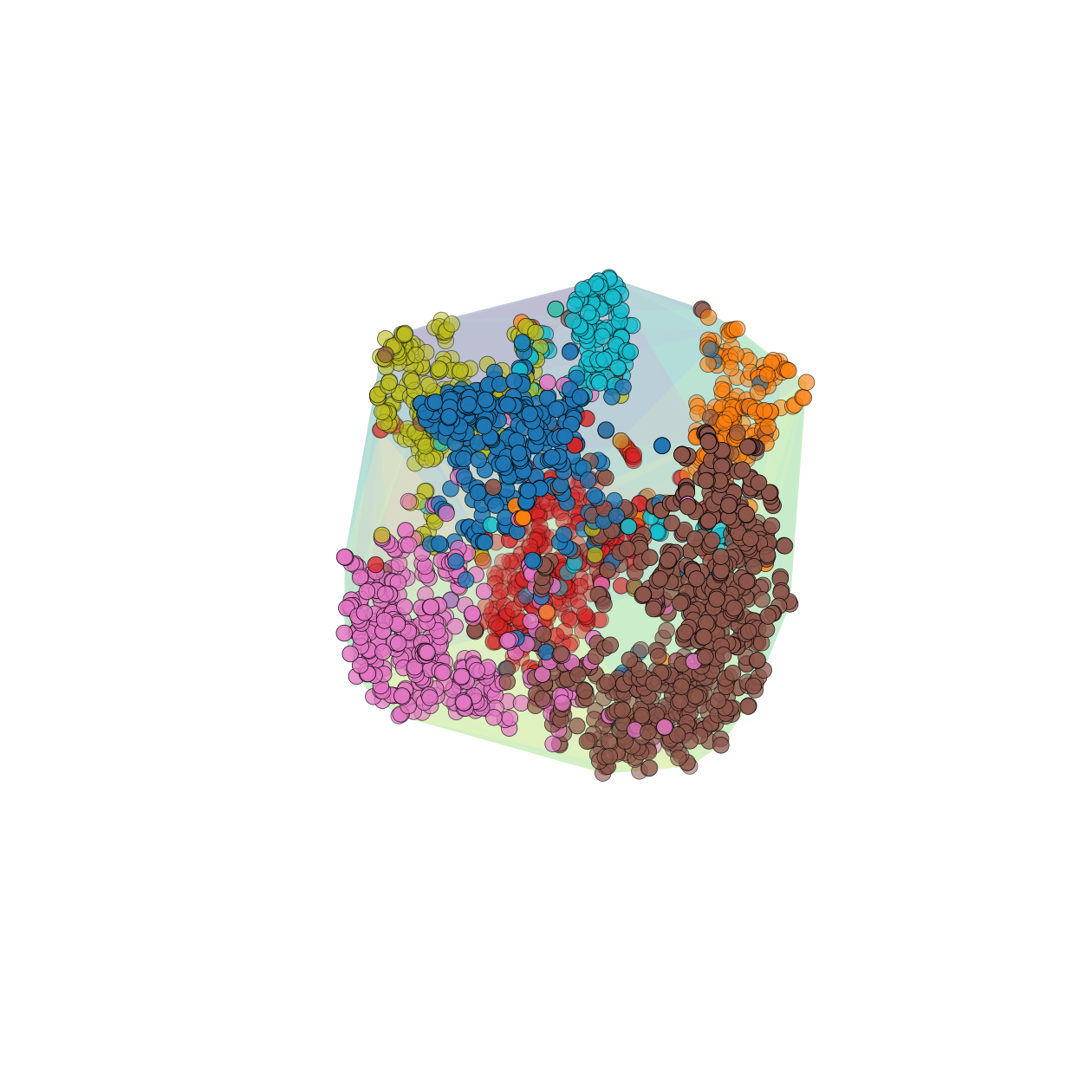}}
\end{subfigure}
\vspace{-4pt}
\caption{Geometry learned by \modelname\ on \textsc{Cora}.\textbf{Left}: Original graph topology colored by class under the layout from the learned embedding projection to 2-D. \textbf{Middle}: Degree-preserving rewiring based on learned geodesic distances reveals clearer class separation. \textbf{Right}: 3-D t-SNE embedding with curvature visualization shows adaptive geometry, the translucent hull is coloured by the magnitude of the mean curvature 
(\textcolor{violet}{violet} $\!\to\!$ flat, \textcolor{yellow}{yellow} $\!\to\!$ strongly curved)}\label{fig:cora_analysis}
\end{figure}

% === CiteSeer ===============================================================
\begin{figure}[H]
\centering
\begin{subfigure}[b]{0.32\columnwidth}
  \centering
  \fbox{\includegraphics[width=\linewidth]{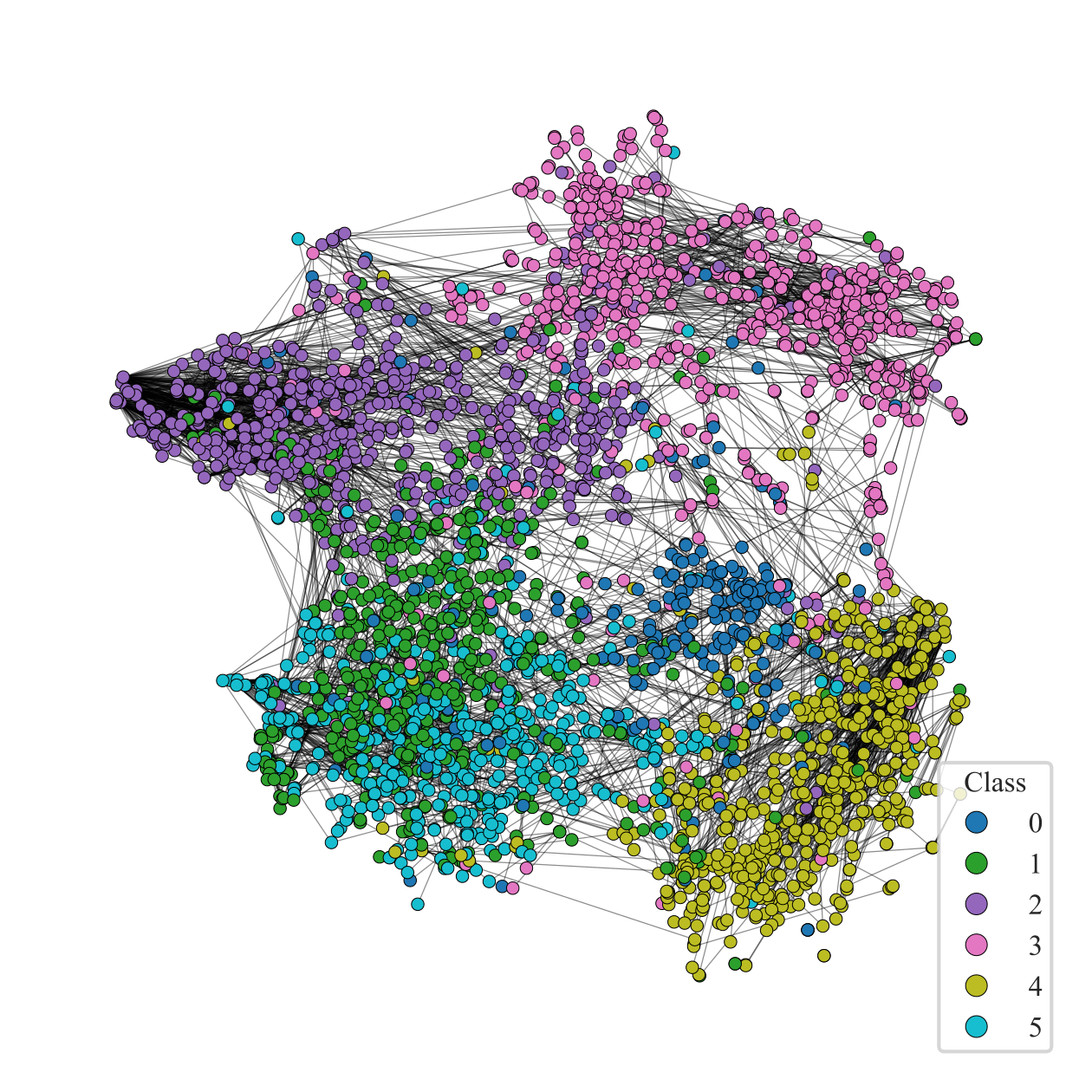}}
\end{subfigure}\hfill
\begin{subfigure}[b]{0.32\columnwidth}
  \centering
  \fbox{\includegraphics[width=\linewidth]{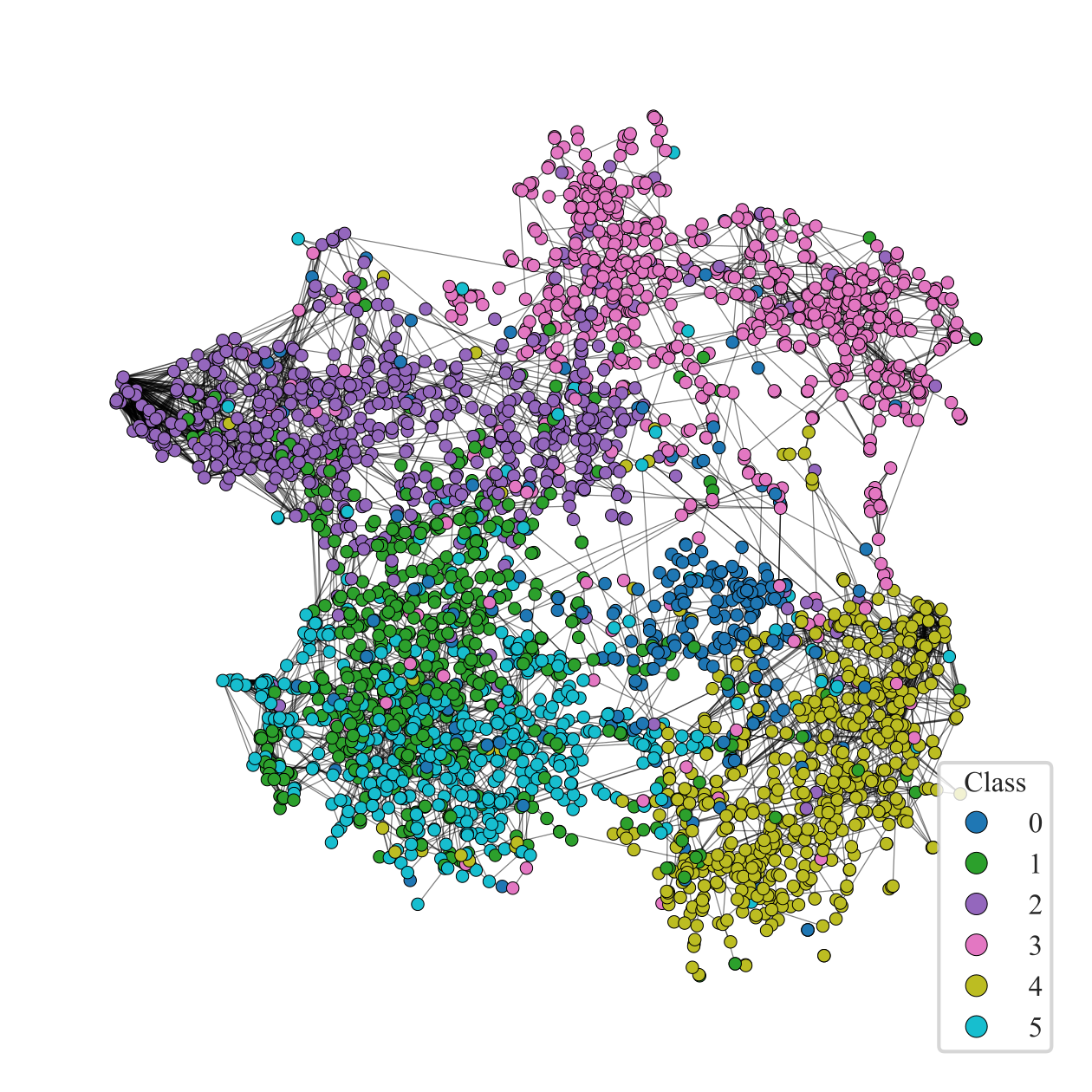}}
\end{subfigure}\hfill
\begin{subfigure}[b]{0.32\columnwidth}
  \centering
  \fbox{\includegraphics[trim=60 50 60 70,clip,width=\linewidth]{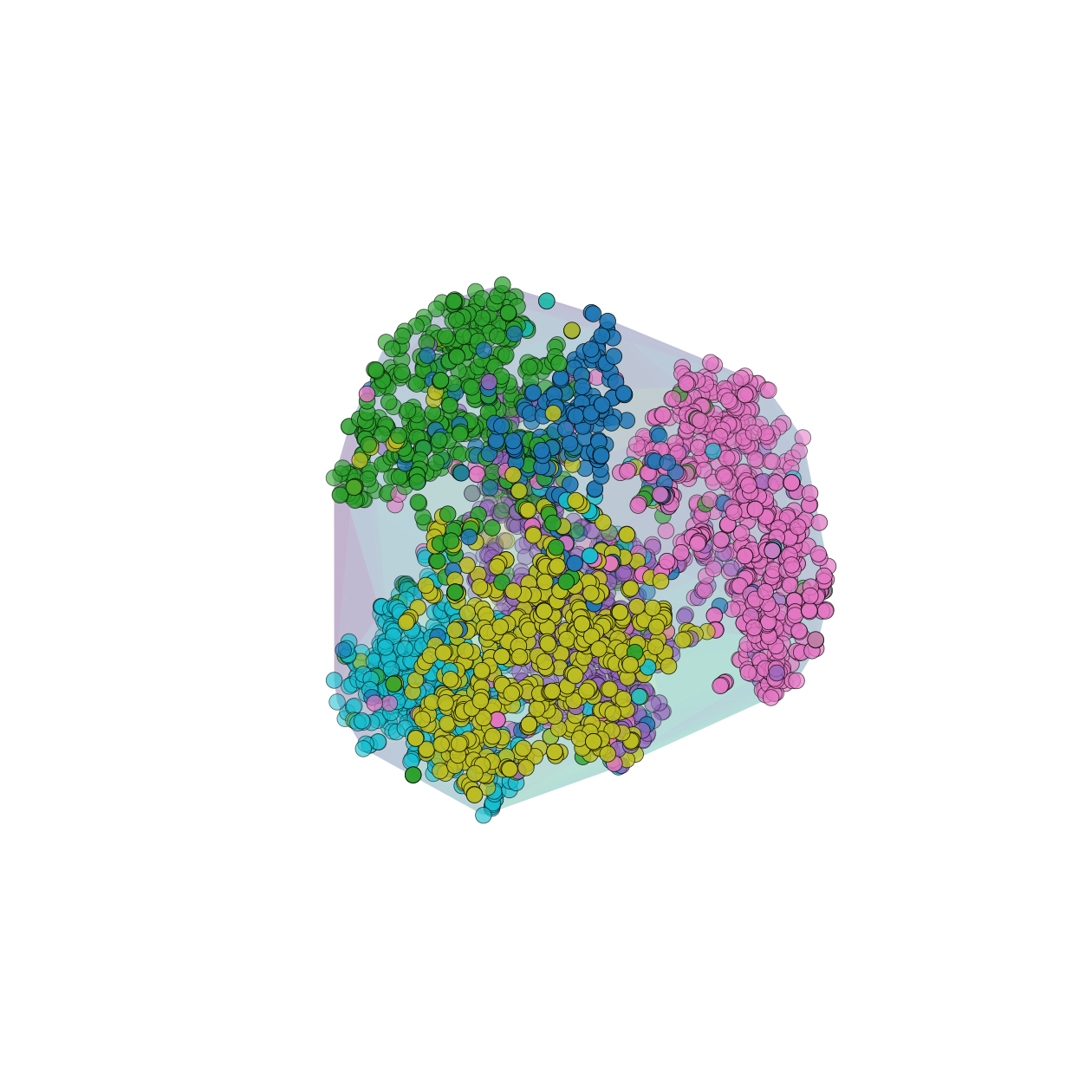}}
\end{subfigure}
\vspace{-4pt}
\caption{Geometry learned by \modelname\ on \textsc{CiteSeer}.\textbf{Left}: Original graph topology colored by class under the layout from the learned embedding projection to 2-D. \textbf{Middle}: Degree-preserving rewiring based on learned geodesic distances reveals clearer class separation. \textbf{Right}: 3-D t-SNE embedding with curvature visualization shows adaptive geometry, the translucent hull is coloured by the magnitude of the mean curvature 
(\textcolor{violet}{violet} $\!\to\!$ flat, \textcolor{yellow}{yellow} $\!\to\!$ strongly curved)}
\label{fig:citeseer_analysis}
\end{figure}

% === PubMed =================================================================
\begin{figure}[H]
\centering
\begin{subfigure}[b]{0.32\columnwidth}
  \centering
  \fbox{\includegraphics[width=\linewidth]{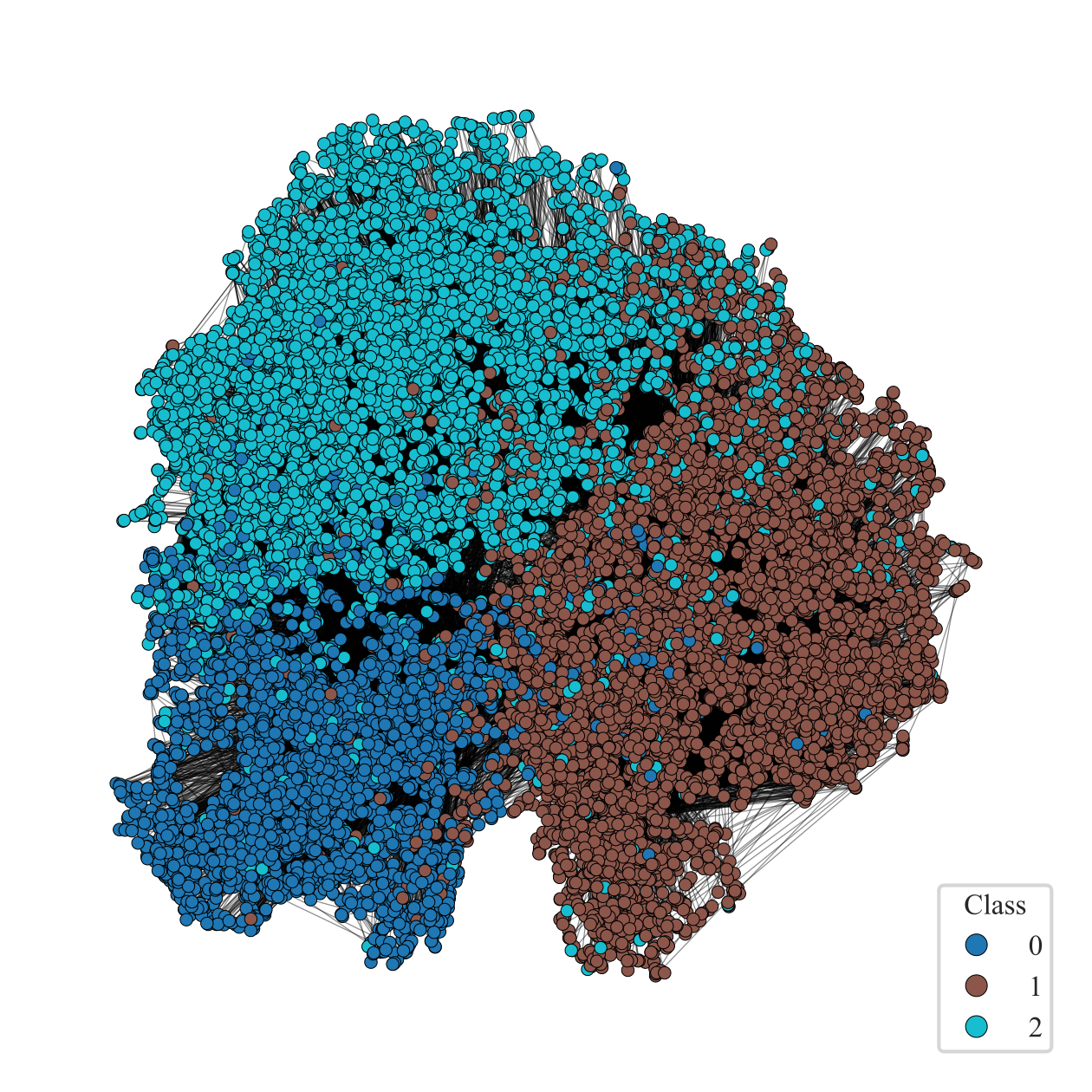}}
\end{subfigure}\hfill
\begin{subfigure}[b]{0.32\columnwidth}
  \centering
  \fbox{\includegraphics[width=\linewidth]{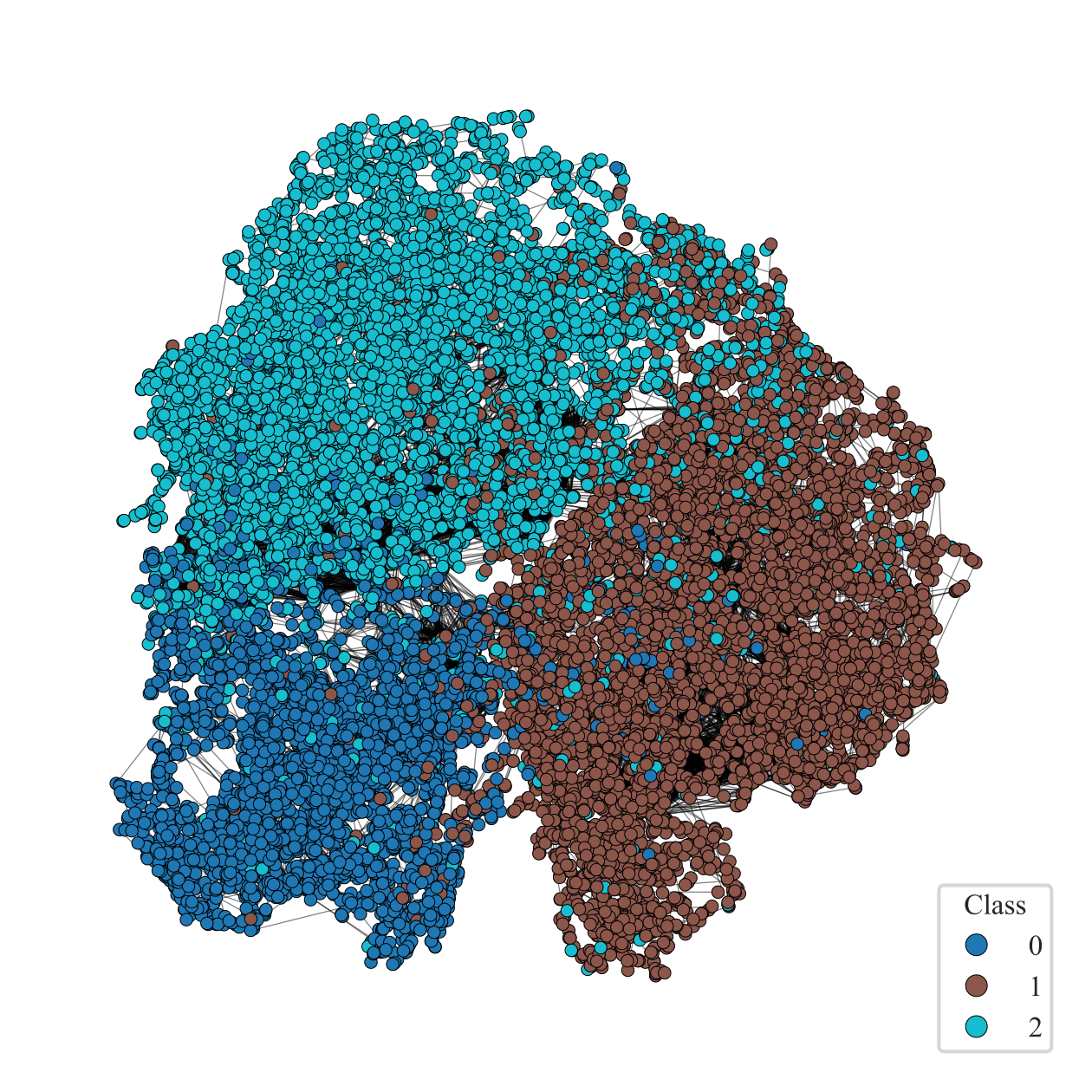}}
\end{subfigure}\hfill
\begin{subfigure}[b]{0.32\columnwidth}
  \centering
  \fbox{\includegraphics[trim=60 50 60 70,clip,width=\linewidth]{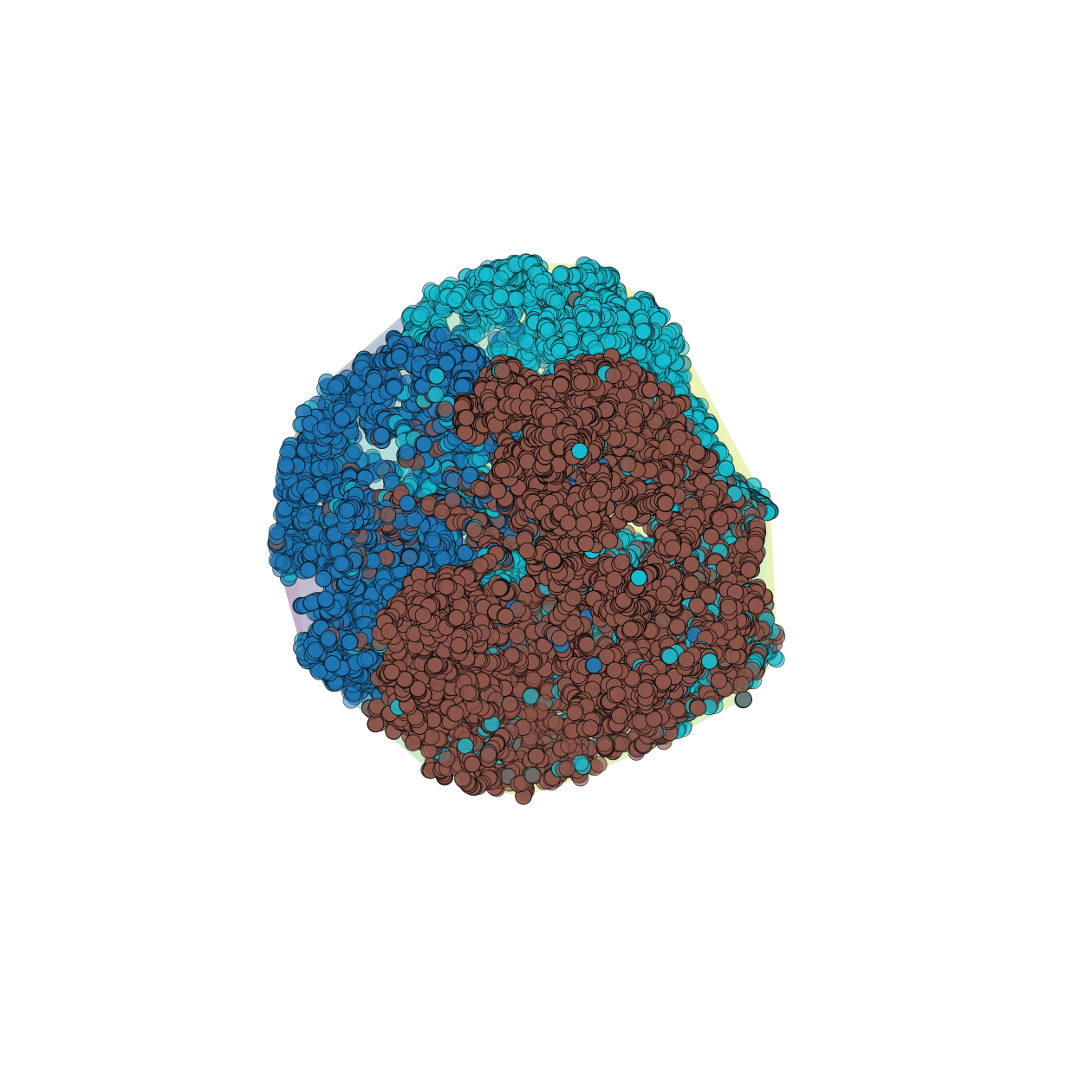}}
\end{subfigure}
\vspace{-4pt}
\caption{Geometry learned by \modelname\ on \textsc{PubMed}.\textbf{Left}: Original graph topology colored by class under the layout from the learned embedding projection to 2-D. \textbf{Middle}: Degree-preserving rewiring based on learned geodesic distances reveals clearer class separation. \textbf{Right}: 3-D t-SNE embedding with curvature visualization shows adaptive geometry, the translucent hull is coloured by the magnitude of the mean curvature 
(\textcolor{violet}{violet} $\!\to\!$ flat, \textcolor{yellow}{yellow} $\!\to\!$ strongly curved)}
\label{fig:pubmed_analysis}
\end{figure}

% === Actor ==================================================================
\begin{figure}[H]
\centering
\begin{subfigure}[b]{0.32\columnwidth}
  \centering
  \fbox{\includegraphics[width=\linewidth]{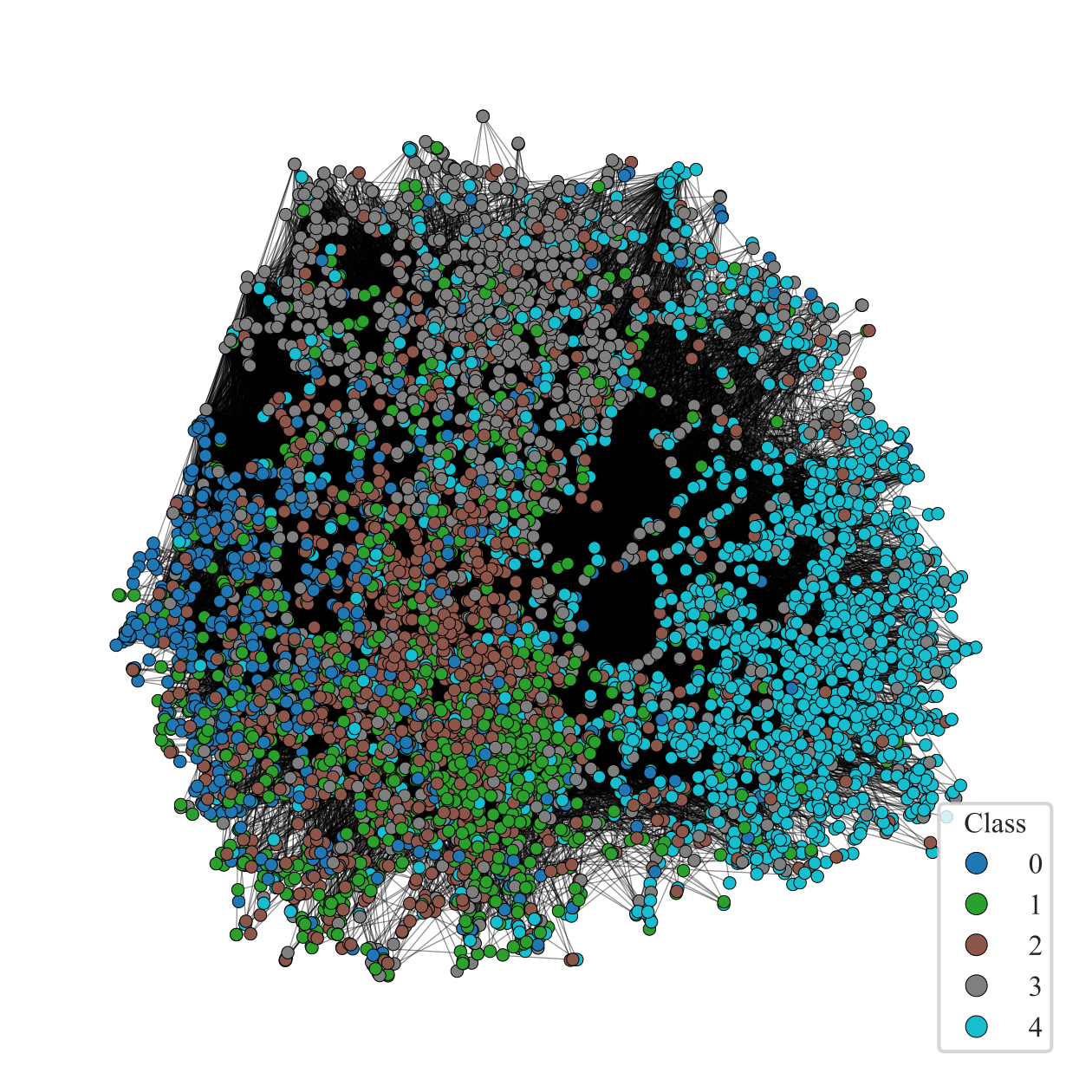}}
\end{subfigure}\hfill
\begin{subfigure}[b]{0.32\columnwidth}
  \centering
  \fbox{\includegraphics[width=\linewidth]{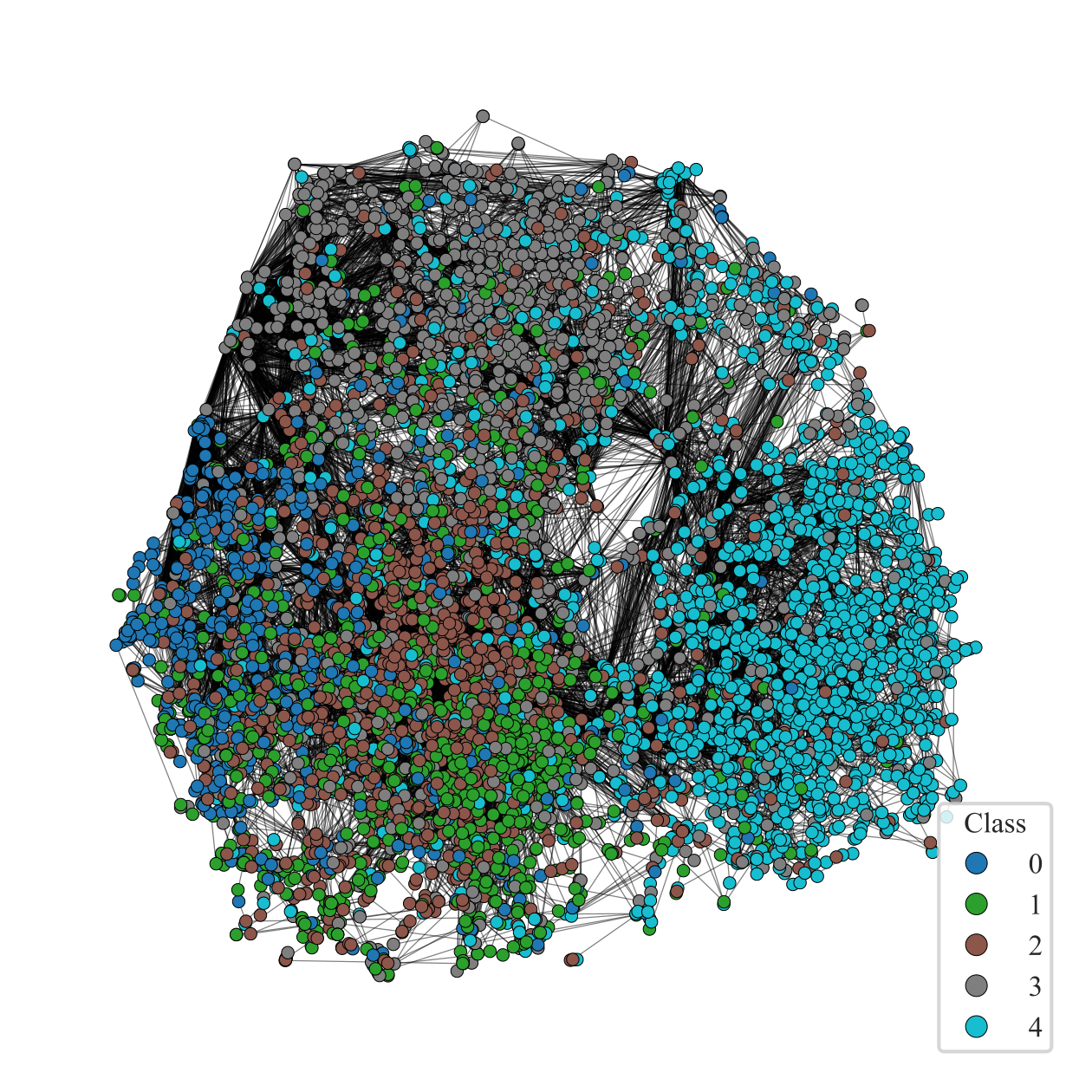}}
\end{subfigure}\hfill
\begin{subfigure}[b]{0.32\columnwidth}
  \centering
  \fbox{\includegraphics[trim=60 50 60 70,clip,width=\linewidth]{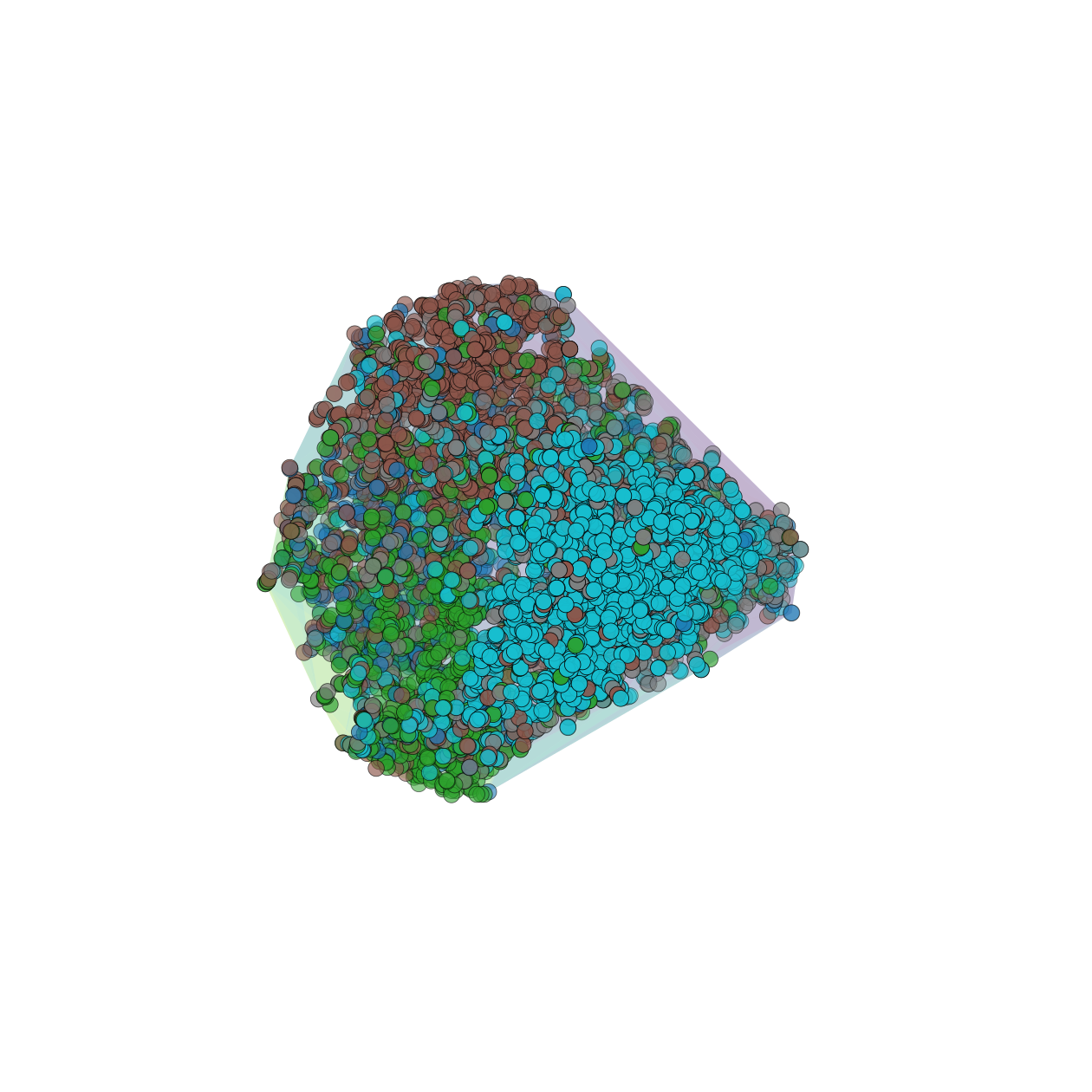}}
\end{subfigure}
\vspace{-4pt}
\caption{Geometry learned by \modelname\ on \textsc{Actor}.\textbf{Left}: Original graph topology colored by class under the layout from the learned embedding projection to 2-D. \textbf{Middle}: Degree-preserving rewiring based on learned geodesic distances reveals clearer class separation. \textbf{Right}: 3-D t-SNE embedding with curvature visualization shows adaptive geometry, the translucent hull is coloured by the magnitude of the mean curvature 
(\textcolor{violet}{violet} $\!\to\!$ flat, \textcolor{yellow}{yellow} $\!\to\!$ strongly curved)}
\label{fig:actor_analysis}
\end{figure}

% === Chameleon ==============================================================  
\begin{figure}[H]
\centering
\begin{subfigure}[b]{0.32\columnwidth}
  \centering
  \fbox{\includegraphics[width=\linewidth]{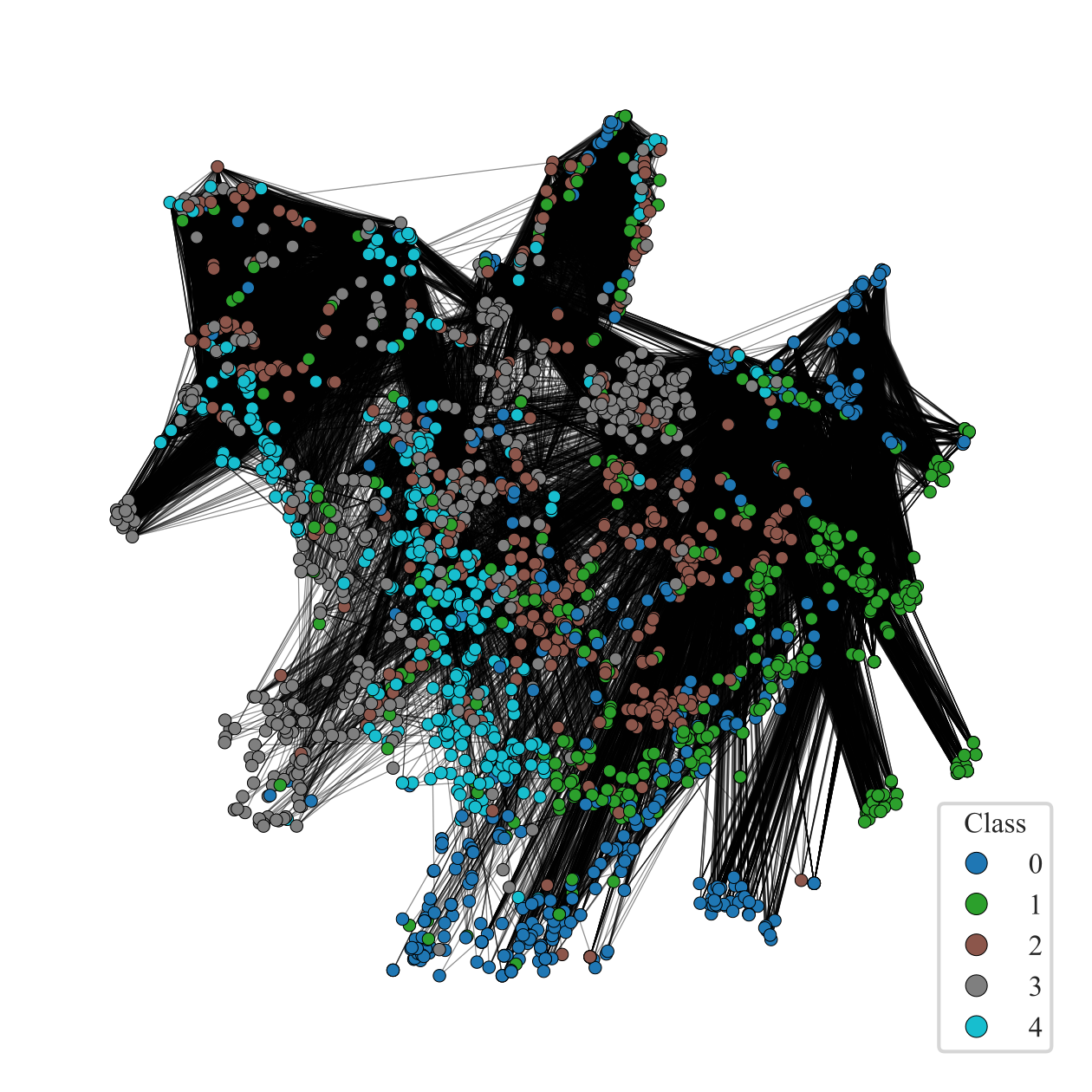}}
\end{subfigure}\hfill
\begin{subfigure}[b]{0.32\columnwidth}
  \centering
  \fbox{\includegraphics[width=\linewidth]{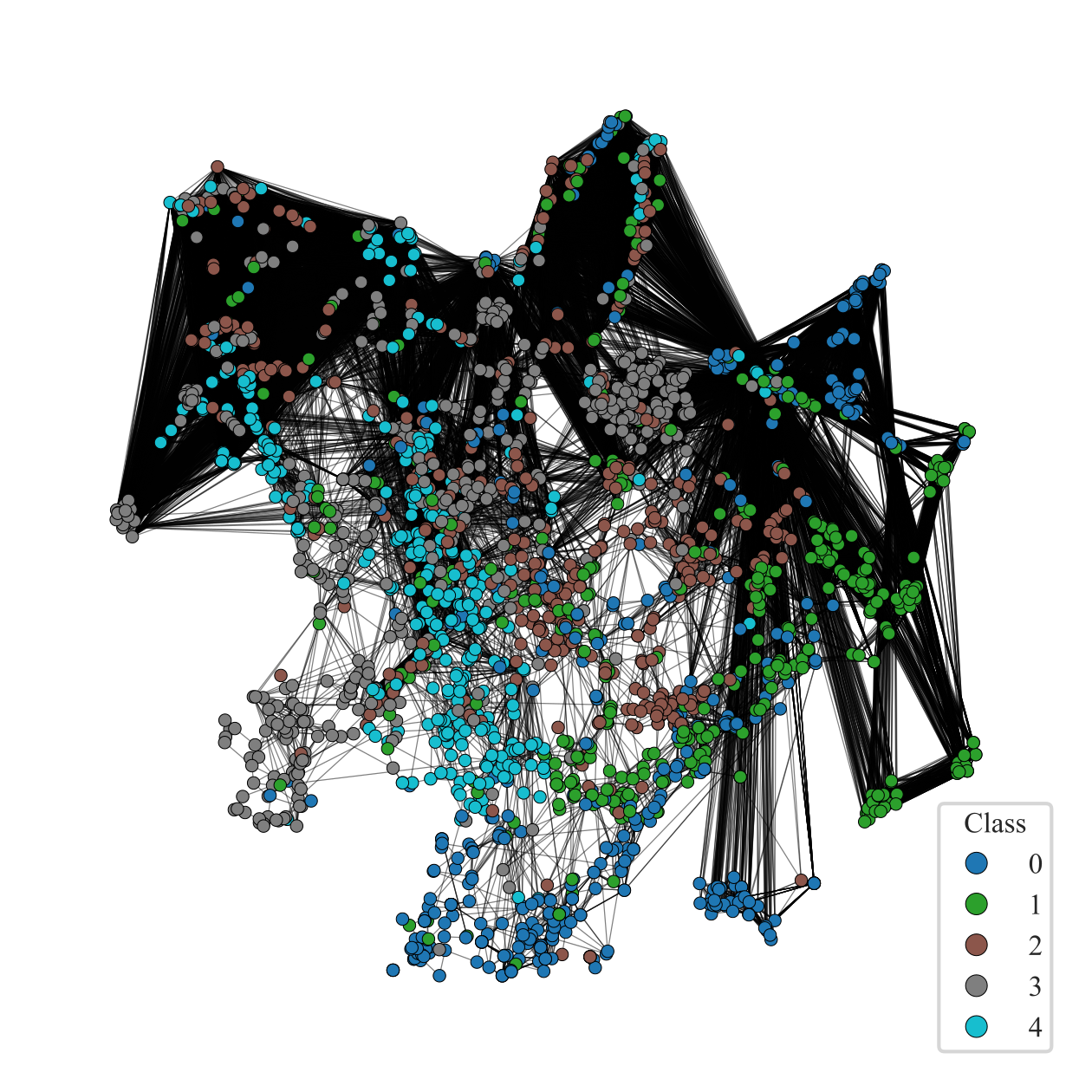}}
\end{subfigure}\hfill
\begin{subfigure}[b]{0.32\columnwidth}
  \centering
  \fbox{\includegraphics[trim=60 50 60 70,clip,width=\linewidth]{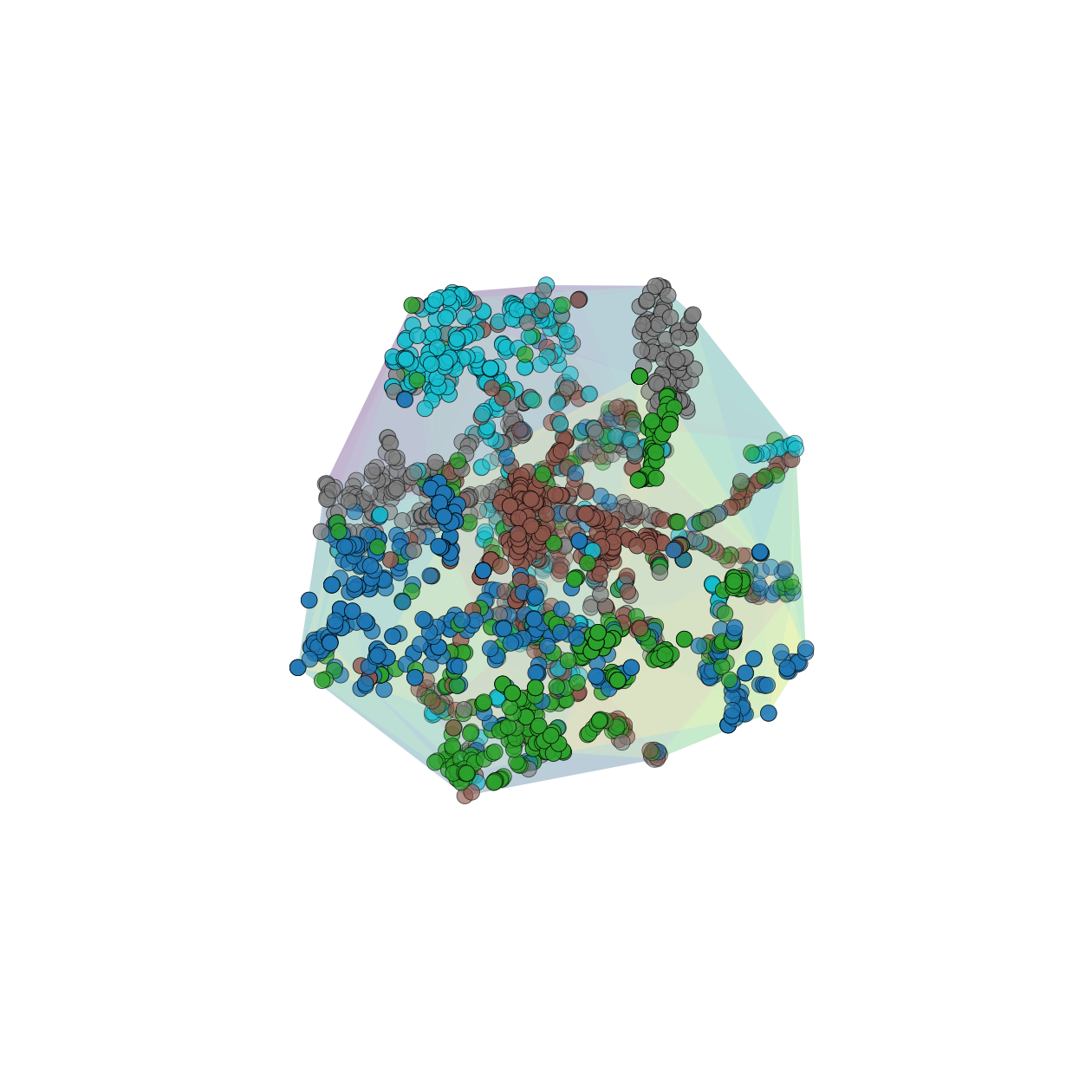}}
\end{subfigure}
\vspace{-4pt}
\caption{Geometry learned by \modelname\ on \textsc{Chameleon}.\textbf{Left}: Original graph topology colored by class under the layout from the learned embedding projection to 2-D. \textbf{Middle}: Degree-preserving rewiring based on learned geodesic distances reveals clearer class separation. \textbf{Right}: 3-D t-SNE embedding with curvature visualization shows adaptive geometry, the translucent hull is coloured by the magnitude of the mean curvature 
(\textcolor{violet}{violet} $\!\to\!$ flat, \textcolor{yellow}{yellow} $\!\to\!$ strongly curved)}
\label{fig:chameleon_analysis}
\end{figure}

% === Squirrel ===============================================================
\begin{figure}[H]
\centering
\begin{subfigure}[b]{0.32\columnwidth}
  \centering
  \fbox{\includegraphics[width=\linewidth]{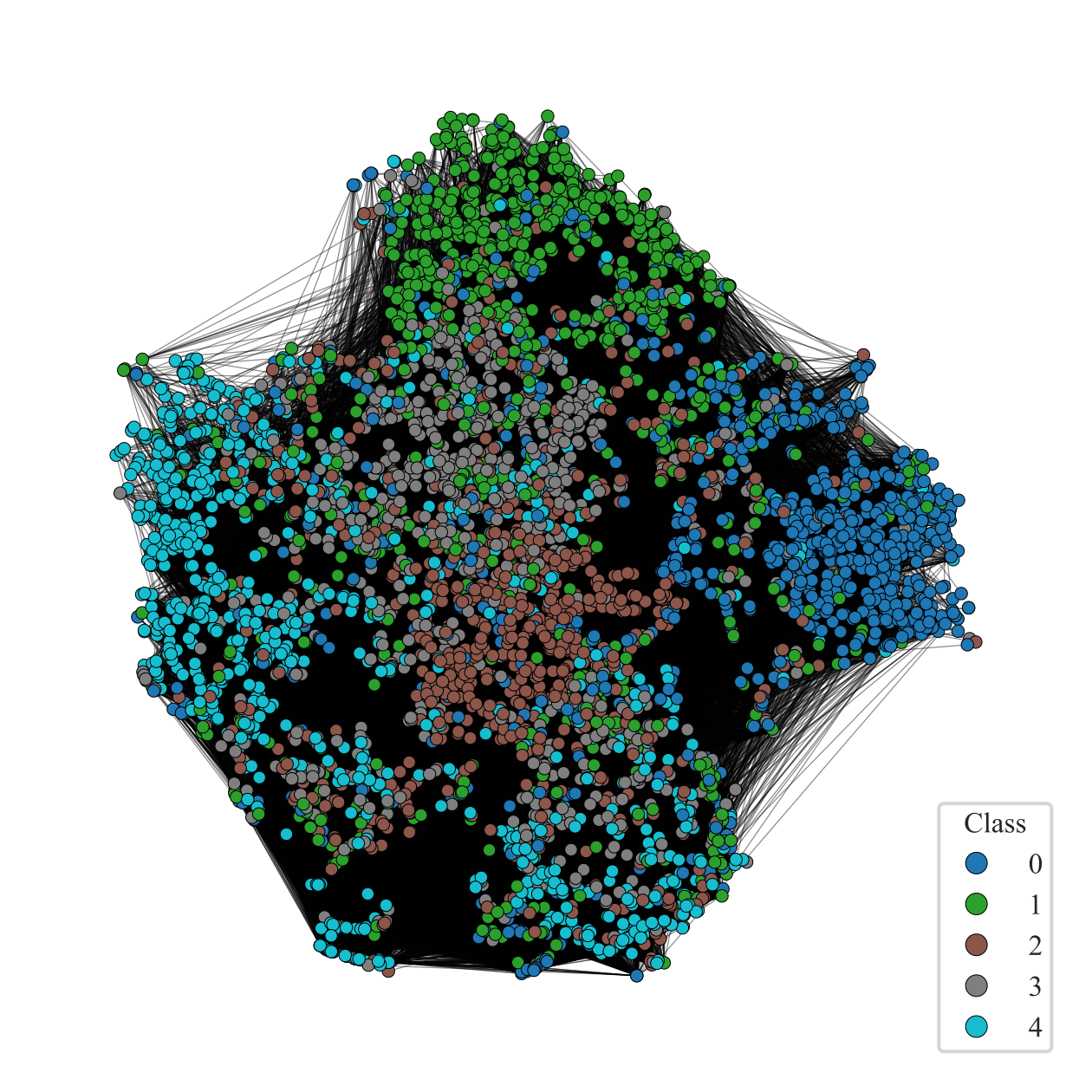}}
\end{subfigure}\hfill
\begin{subfigure}[b]{0.32\columnwidth}
  \centering
  \fbox{\includegraphics[width=\linewidth]{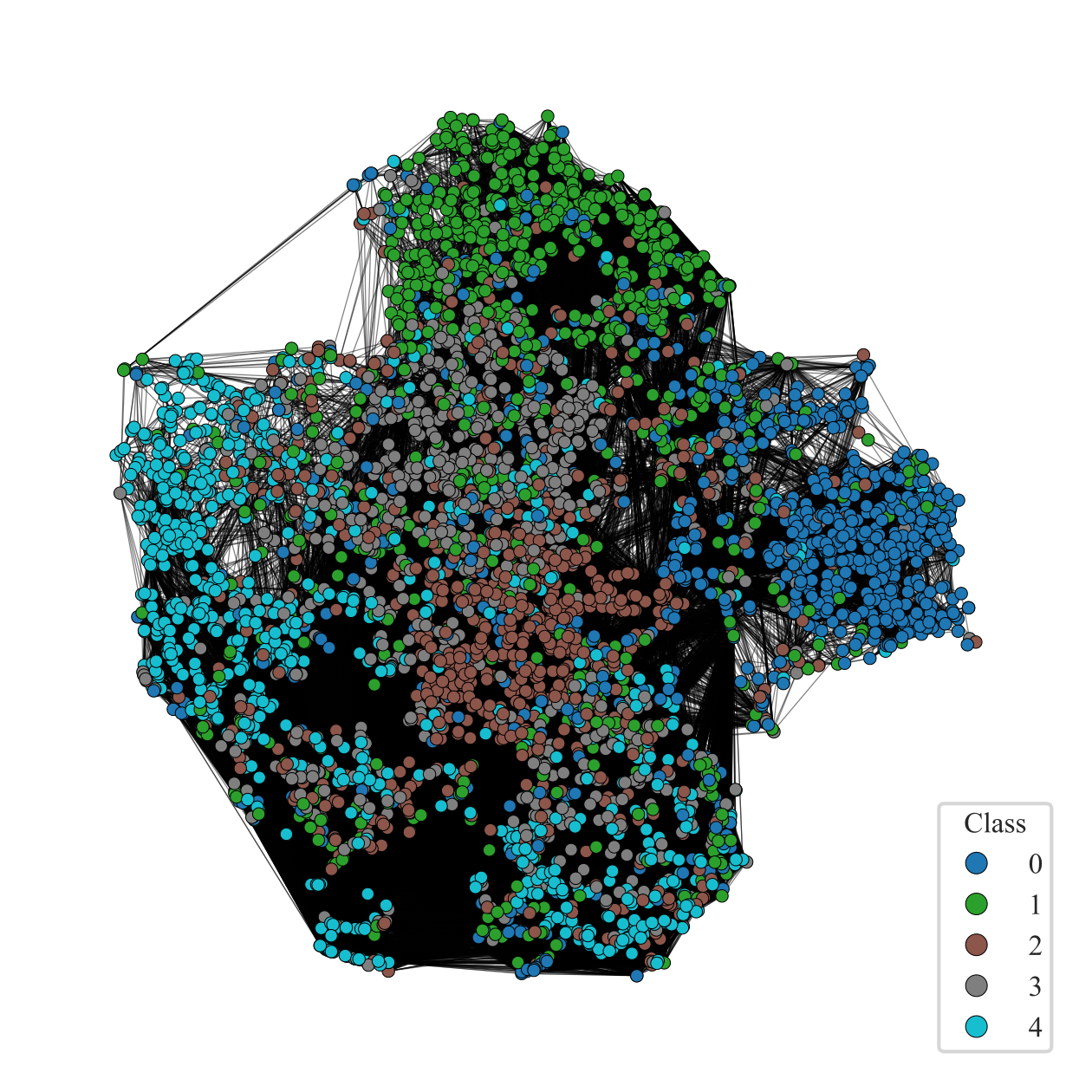}}
\end{subfigure}\hfill
\begin{subfigure}[b]{0.32\columnwidth}
  \centering
  \fbox{\includegraphics[trim=60 50 60 70,clip,width=\linewidth]{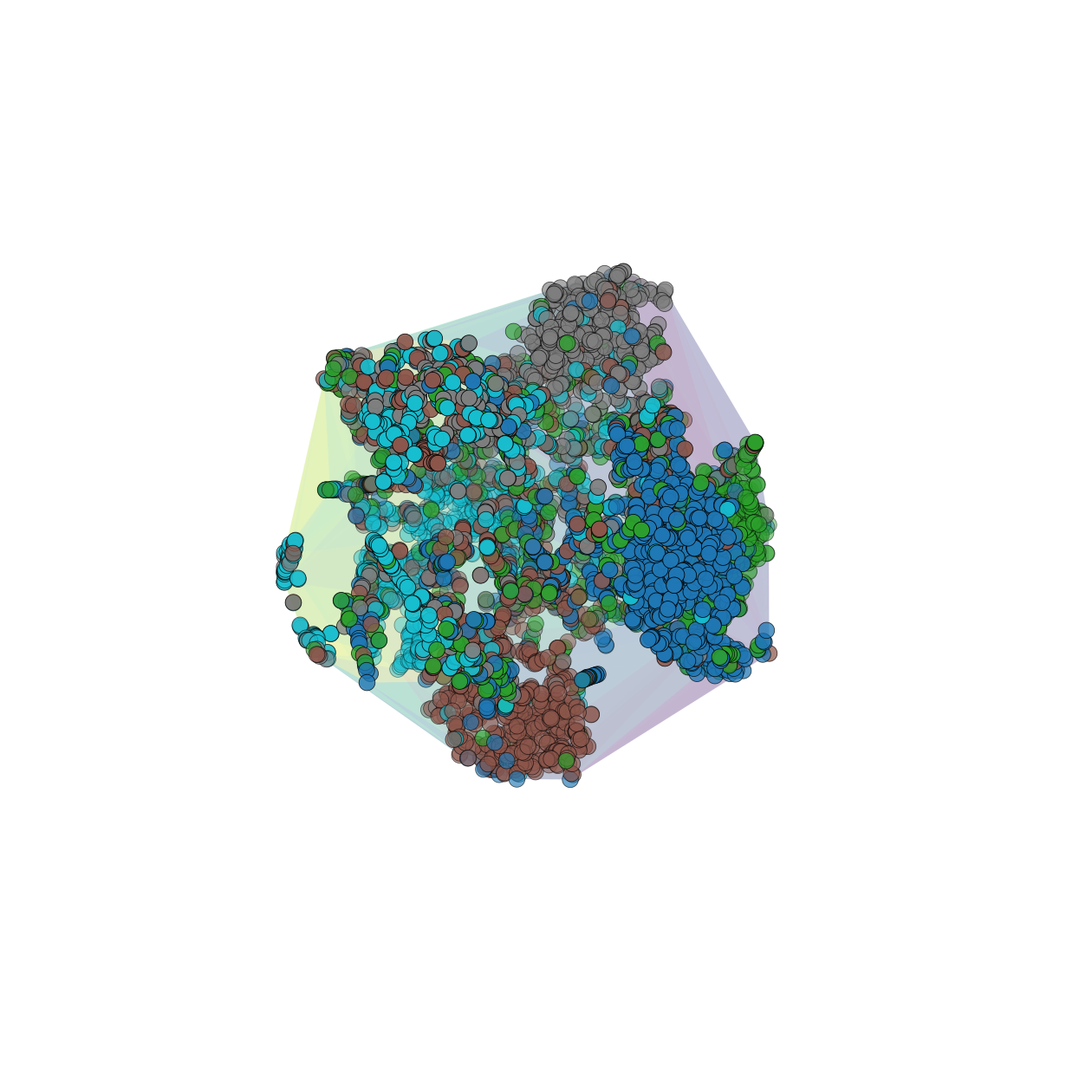}}
\end{subfigure}
\vspace{-4pt}
\caption{Geometry learned by \modelname\ on \textsc{Squirrel}.\textbf{Left}: Original graph topology colored by class under the layout from the learned embedding projection to 2-D. \textbf{Middle}: Degree-preserving rewiring based on learned geodesic distances reveals clearer class separation. \textbf{Right}: 3-D t-SNE embedding with curvature visualization shows adaptive geometry, the translucent hull is coloured by the magnitude of the mean curvature 
(\textcolor{violet}{violet} $\!\to\!$ flat, \textcolor{yellow}{yellow} $\!\to\!$ strongly curved)}
\label{fig:squirrel_analysis}
\end{figure}

% === Texas ==================================================================
\begin{figure}[H]
\centering
\begin{subfigure}[b]{0.32\columnwidth}
  \centering
  \fbox{\includegraphics[width=\linewidth]{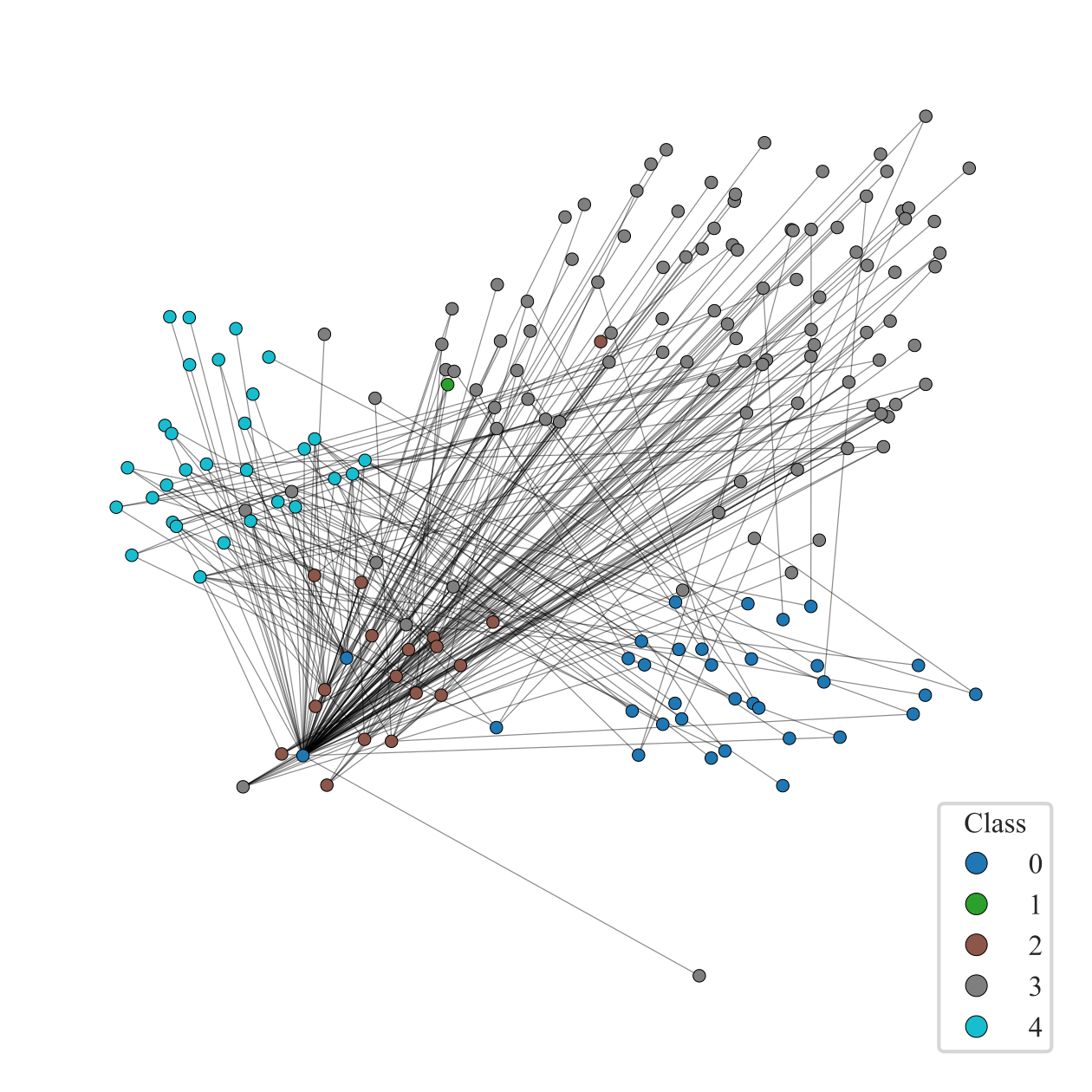}}
\end{subfigure}\hfill
\begin{subfigure}[b]{0.32\columnwidth}
  \centering
  \fbox{\includegraphics[width=\linewidth]{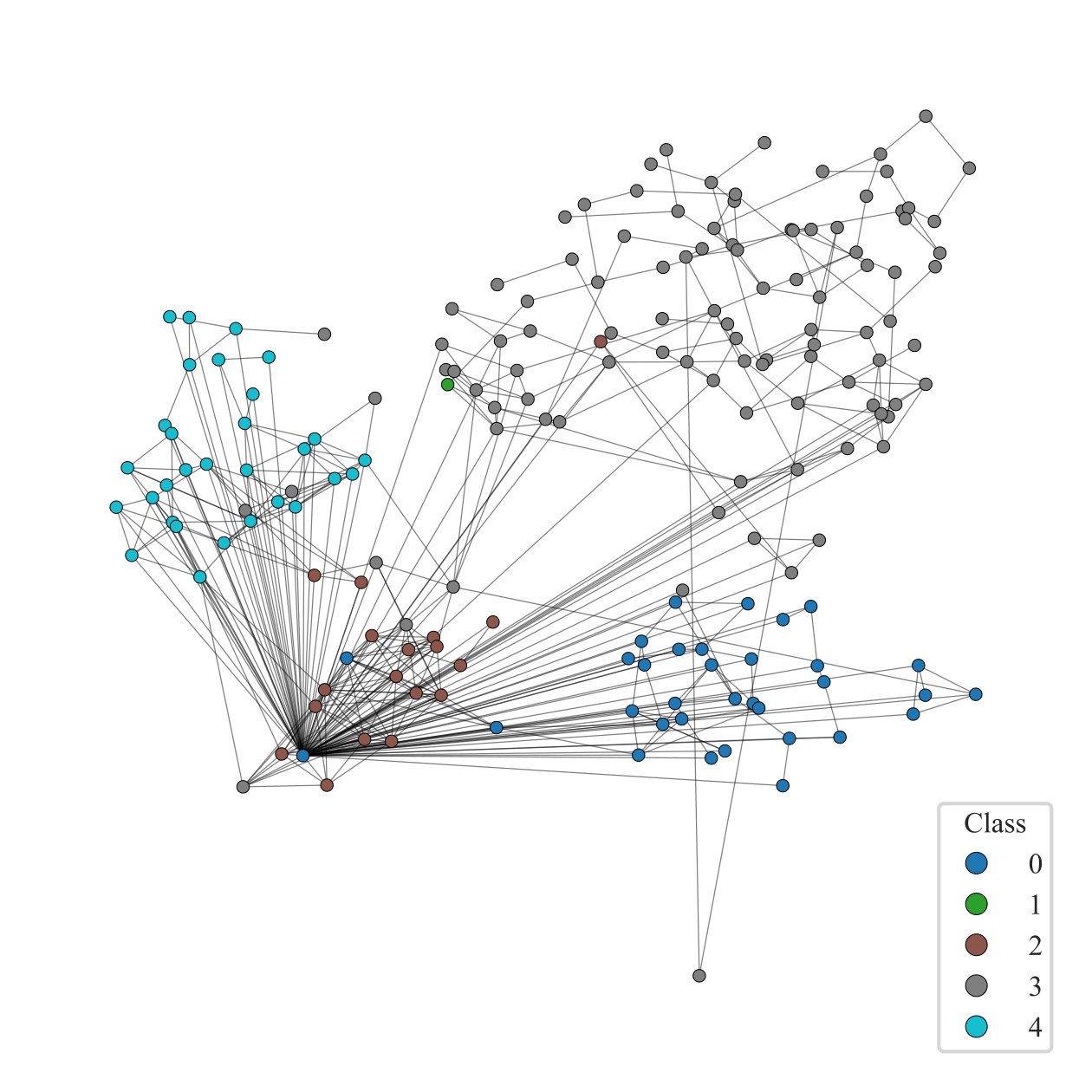}}
\end{subfigure}\hfill
\begin{subfigure}[b]{0.32\columnwidth}
  \centering
  \fbox{\includegraphics[trim=60 50 60 70,clip,width=\linewidth]{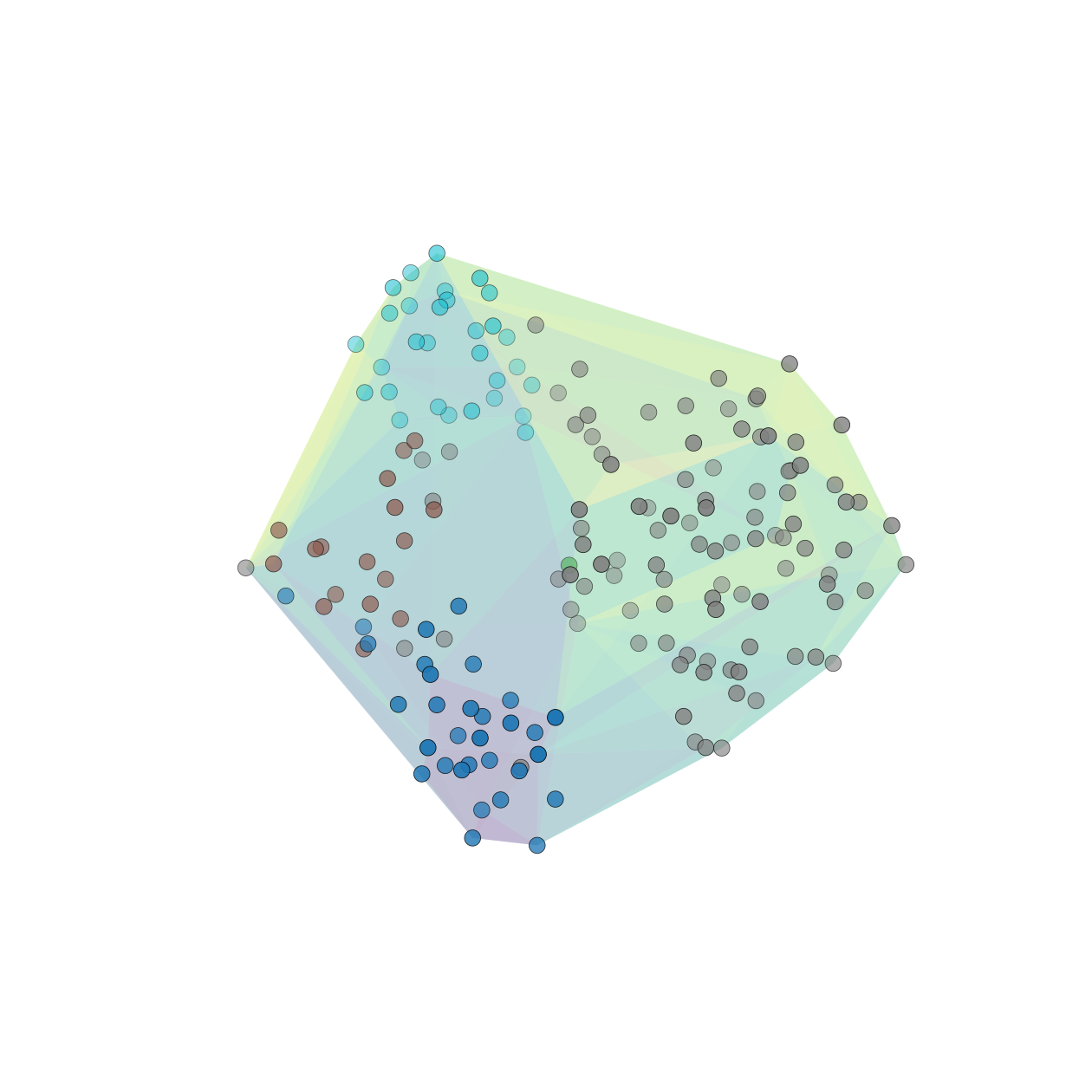}}
\end{subfigure}
\vspace{-4pt}
\caption{Geometry learned by \modelname\ on \textsc{Texas}.\textbf{Left}: Original graph topology colored by class under the layout from the learned embedding projection to 2-D. \textbf{Middle}: Degree-preserving rewiring based on learned geodesic distances reveals clearer class separation. \textbf{Right}: 3-D t-SNE embedding with curvature visualization shows adaptive geometry, the translucent hull is coloured by the magnitude of the mean curvature 
(\textcolor{violet}{violet} $\!\to\!$ flat, \textcolor{yellow}{yellow} $\!\to\!$ strongly curved)}
\label{fig:texas_analysis}
\end{figure}

% === Cornell ================================================================
\begin{figure}[H]
\centering
\begin{subfigure}[b]{0.32\columnwidth}
  \centering
  \fbox{\includegraphics[width=\linewidth]{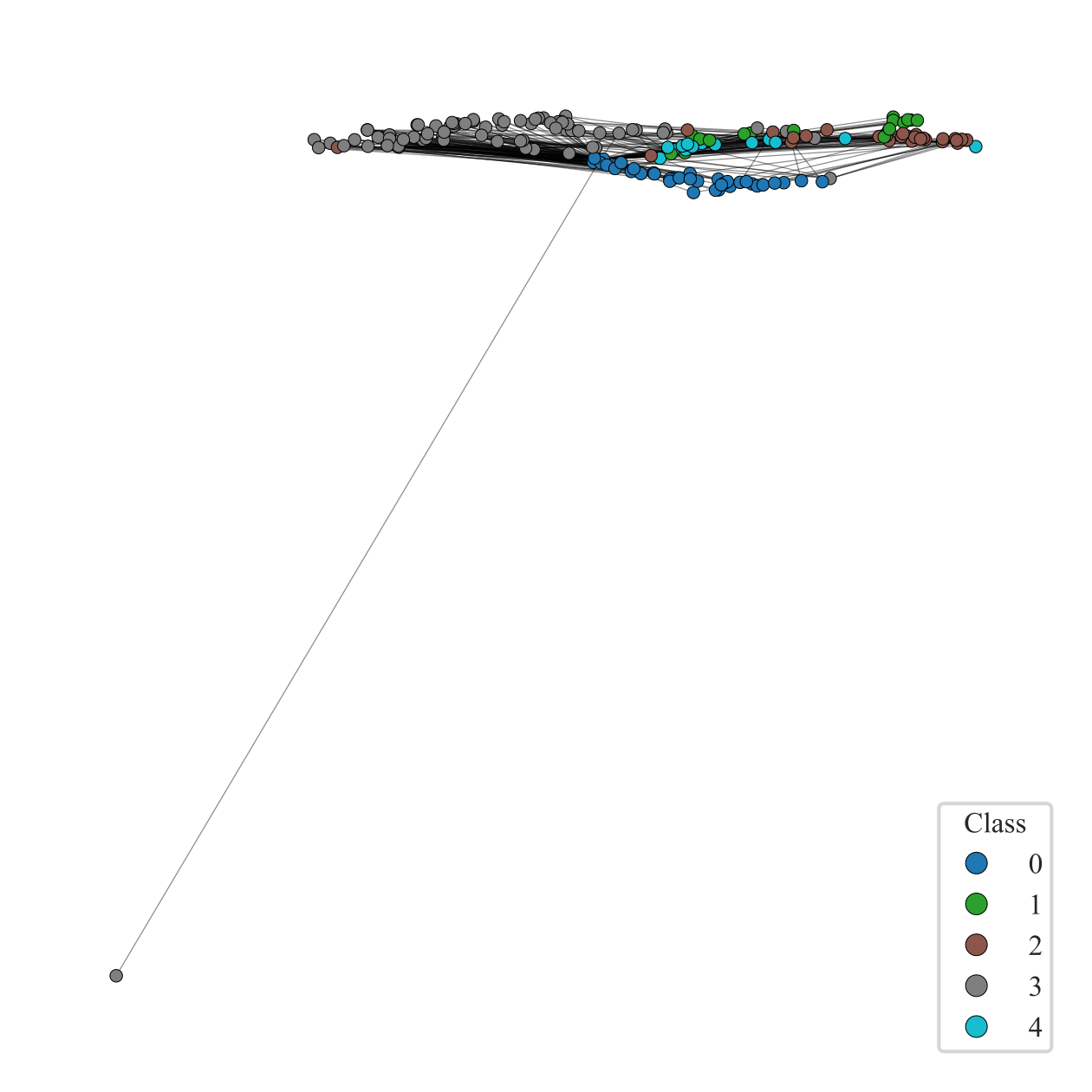}}
\end{subfigure}\hfill
\begin{subfigure}[b]{0.32\columnwidth}
  \centering
  \fbox{\includegraphics[width=\linewidth]{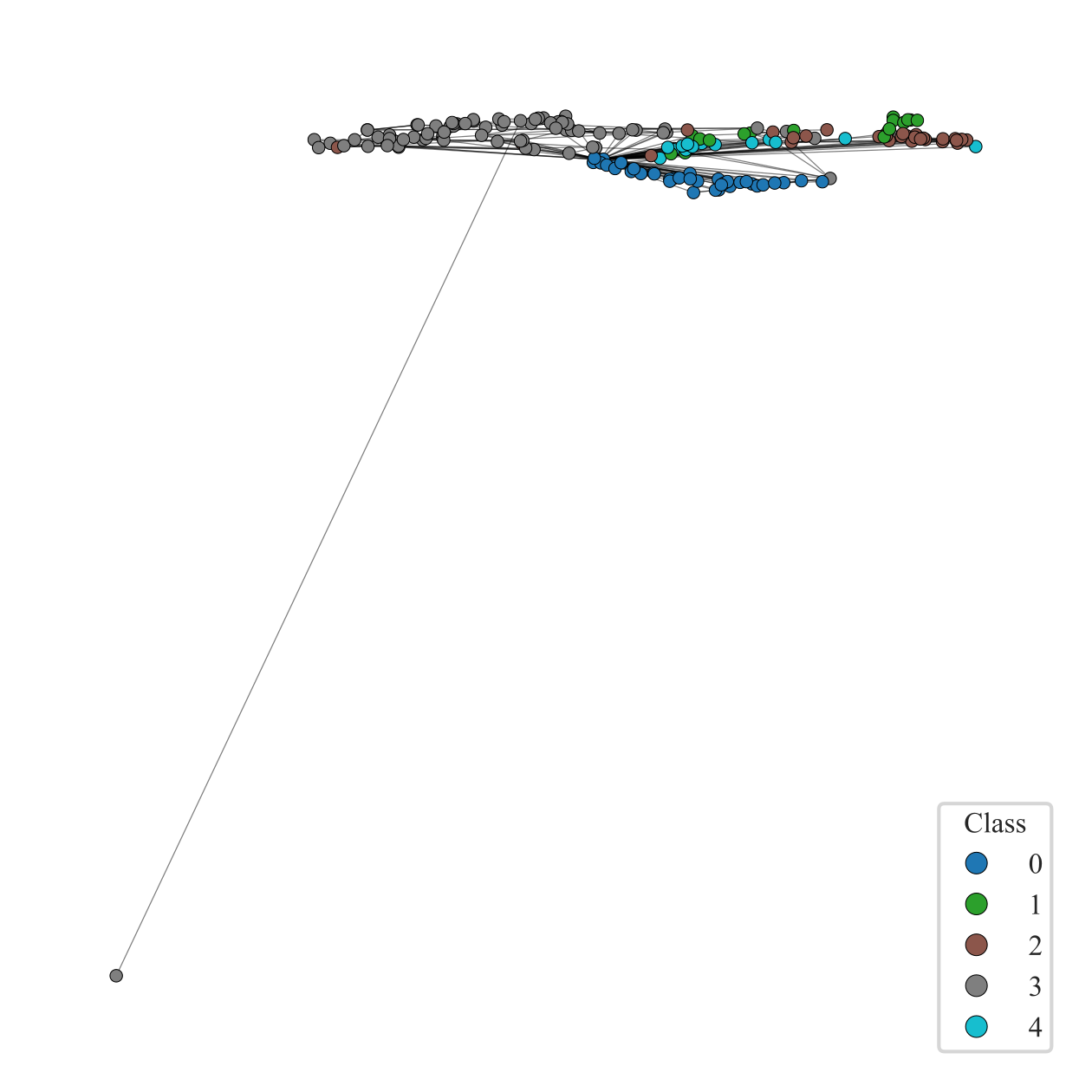}}
\end{subfigure}\hfill
\begin{subfigure}[b]{0.32\columnwidth}
  \centering
  \fbox{\includegraphics[trim=60 50 60 70,clip,width=\linewidth]{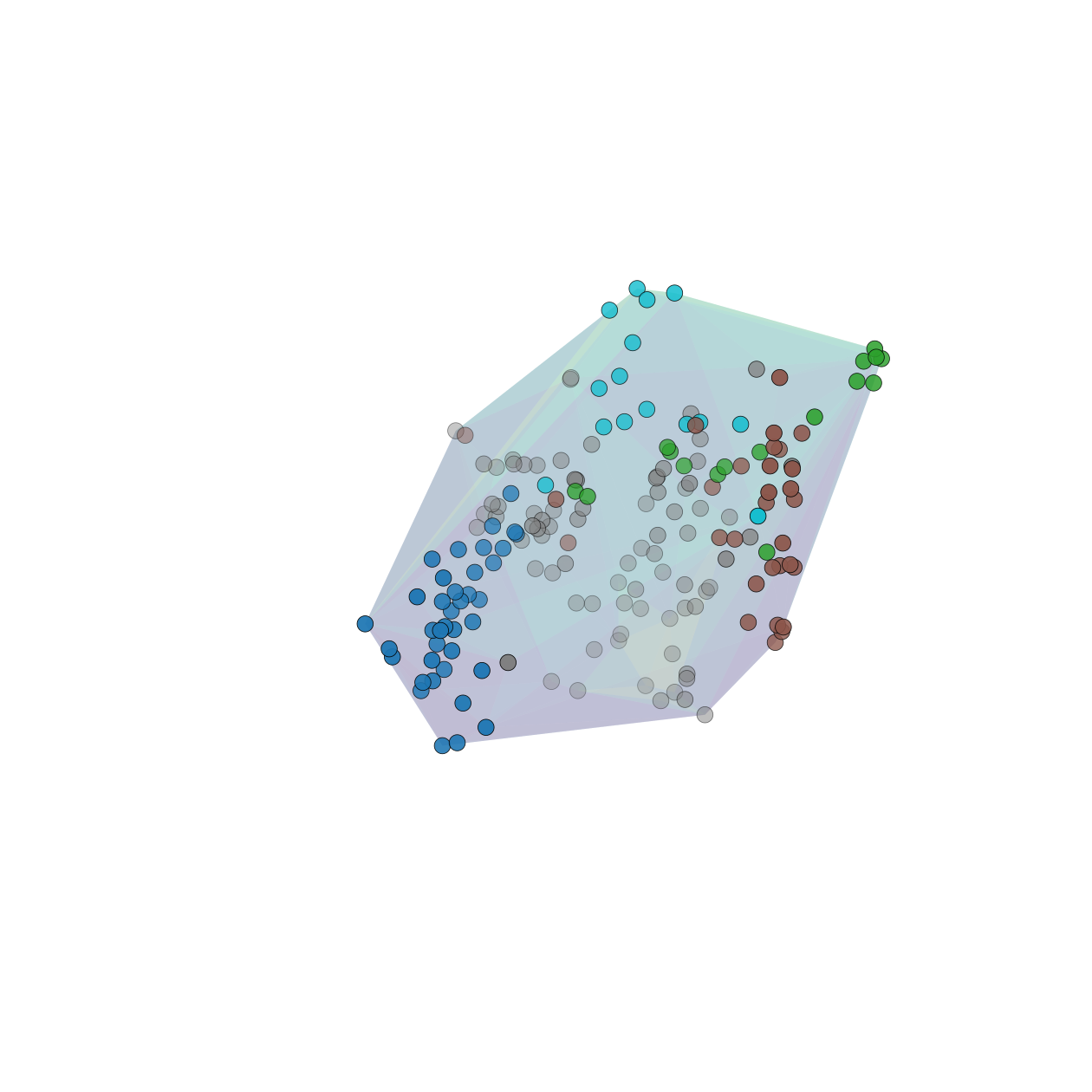}}
\end{subfigure}
\vspace{-4pt}
\caption{Geometry learned by \modelname\ on \textsc{Cornell}.\textbf{Left}: Original graph topology colored by class under the layout from the learned embedding projection to 2-D. \textbf{Middle}: Degree-preserving rewiring based on learned geodesic distances reveals clearer class separation. \textbf{Right}: 3-D t-SNE embedding with curvature visualization shows adaptive geometry, the translucent hull is coloured by the magnitude of the mean curvature 
(\textcolor{violet}{violet} $\!\to\!$ flat, \textcolor{yellow}{yellow} $\!\to\!$ strongly curved)}
\label{fig:cornell_analysis}
\end{figure}

\section{Discussion on Limitation and Borderline Effects}
\label{app:limitations}

\subsection{Diagonal Metric Tensor Simplification}

While \modelname demonstrates strong empirical performance through learnable diagonal metric tensors, this design choice represents a deliberate trade-off between expressiveness and computational efficiency. The diagonal parameterization $\mathbf{G}_i = \text{diag}(\mathbf{g}_i)$ constrains the metric to axis-aligned anisotropy, potentially limiting the model's ability to capture certain geometric structures:

\begin{itemize}
\item \textbf{Geometric Expressiveness}: Full metric tensors $\mathbf{G}_i \in \text{SPD}(d)$ could capture arbitrary local orientations and anisotropic scaling, enabling richer geometric representations. The additional $\frac{d(d-1)}{2}$ off-diagonal parameters per node would allow modeling of correlated feature dimensions and rotational geometry.

\item \textbf{Computational Complexity}: However, full tensors incur $O(d^3)$ complexity for metric operations and $O(|\mathcal{V}|d^2)$ memory overhead. For typical GNN dimensions ($d=128$), this represents a 128× increase in metric computation cost. Low-rank factorizations $\mathbf{G}_i = \mathbf{L}_i\mathbf{L}_i^T$ with $\mathbf{L}_i \in \mathbb{R}^{d \times r}$ offer a middle ground but still require $O(dr^2)$ operations.

\item \textbf{Optimization Challenges}: Non-diagonal metrics cannot leverage our efficient conformal geometry framework (Section~\ref{sec:method}), requiring expensive SPD manifold projections and Riemannian optimization. The increased parameter space also exacerbates overfitting on small graphs.
\end{itemize}

We view exploring full or low-rank metric tensors as valuable future work, particularly for applications requiring fine-grained geometric modeling. However, this extension lies beyond the current scope, as our diagonal approach already achieves state-of-the-art performance while maintaining computational efficiency comparable to standard GNNs.

\subsection{Statements of Broader Impact}

This work introduced \modelname, a novel framework that learns continuous, anisotropic metric tensor fields to capture the geometric diversity inherent in real-world graphs. This advancement possibly has broader impacts for scientific and societal benefits, such as enhancing community network analysis and advancing biomedical research by its node and link prediction power. Apart from offering possible scientific and societal benefits, to the best of our knowledge, our research works do not involve ethical constraints or religious and political restrictions.

\end{document}